\documentclass{article}

\usepackage{microtype}
\usepackage{graphicx}
\usepackage{subfigure}
\usepackage{enumitem}
\usepackage{float}
\usepackage{booktabs} 

\usepackage{hyperref}
\usepackage{url}

\usepackage[accepted]{icml2024}

\usepackage{amsmath}
\usepackage{amssymb}
\usepackage{mathtools}
\usepackage{amsthm}
\usepackage{natbib}

\usepackage[capitalize,noabbrev]{cleveref}
\usepackage{comment}

\newtheorem{theorem}{Theorem}

\newtheorem{lemma}{Lemma}

\theoremstyle{definition}

\newtheorem{assumption}{Assumption}
\newtheorem{remark}{Remark}
\newtheorem{example}{Example}

\newcommand*\diff{\mathop{}\!\mathrm{d}}
\newcommand{\z}{{\bf z}}
\newcommand{\x}{{\bf x}}
\newcommand{\y}{{\bf y}}

\newcommand{\X}{\mathcal{X}}
\newcommand{\Y}{\mathcal{Y}}

\newcommand{\Hess}{{\bf H}}
\newcommand{\norm}[1]{\left\lVert#1\right\rVert}
\DeclareMathOperator*{\argmin}{arg\,min}

\makeatletter

\usepackage[textsize=tiny]{todonotes}

\icmltitlerunning{Overcoming Saturation in Density Ratio Estimation by Iterated Regularization}

\begin{document}

\twocolumn[
\icmltitle{Overcoming Saturation in Density Ratio Estimation by Iterated Regularization}




\begin{icmlauthorlist}
\icmlauthor{Lukas Gruber}{ml}
\icmlauthor{Markus Holzleitner}{mal}
\icmlauthor{Johannes Lehner}{ml}
\icmlauthor{Sepp Hochreiter}{ml,nxai}
\icmlauthor{Werner Zellinger}{ric}
\end{icmlauthorlist}

\icmlaffiliation{ml}{ELLIS Unit Linz and LIT AI Lab, Institute for Machine Learning, Johannes Kepler University Linz, Austria}
\icmlaffiliation{mal}{MaLGa Center, Department of Mathematics, University of Genoa}
\icmlaffiliation{nxai}{NXAI GmbH, Linz, Austria}
\icmlaffiliation{ric}{Johann Radon Institute for Computational and Applied Mathematics, Austrian Academy of Sciences}

\icmlcorrespondingauthor{Werner Zellinger}{werner.zellinger@ricam.oeaw.ac.at}
\icmlcorrespondingauthor{Lukas Gruber}{gruber@ml.jku.at}

\icmlkeywords{Machine Learning, ICML}

\vskip 0.3in
]



\printAffiliationsAndNotice{}

\begin{abstract}
Estimating the ratio of two probability densities from finitely many samples, is a central task in machine learning and statistics.
In this work, we show that a large class of kernel methods for density ratio estimation suffers from error saturation, which prevents algorithms from achieving fast error convergence rates on highly regular learning problems.
To resolve saturation, we introduce iterated regularization in density ratio estimation to achieve fast error rates.
Our methods outperform its non-iteratively regularized versions on benchmarks for density ratio estimation as well as on large-scale evaluations for importance-weighted ensembling of deep unsupervised domain adaptation models.
\end{abstract}

\section{Introduction}
Given two i.i.d.~samples $\x=\{x_i\}_{i=1}^m$ and $\x'=\{x_i'\}_{i=1}^n$, drawn from two probability measures $P$ and $Q$, respectively, the problem of density ratio estimation is to approximate $\beta:=\frac{\diff P}{\diff Q}$.
This problem arises in anomaly detection~\citep{smola2009relative,hido2011statistical}, two-sample testing~\citep{keziou2005test,kanamori2011f}, divergence estimation~\citep{nguyen2007estimating,nguyen2010estimating}, unsupervised domain adaptation~\citep{shimodaira2000improving,dinu2022aggregation}, generative modeling~\citep{mohamed2016learning}, conditional density estimation~\citep{schuster2020kernel}, and PU learning~\citep{kato2019learning}.

A large class of methods for density ratio estimation relies on the estimation of minimizers of the form~\citet{sugiyama2012densitybregman,sugiyama2012density,zellinger2023adaptive}
\begin{align}
    \label{eq:regularized_Bregman_objective}
    f^\lambda:=\argmin_{f\in\mathcal{H}} B_F\!\left(\beta,g(f)\right) + \frac{\lambda}{2} \norm{f}^2,
\end{align}
where $g(f^\lambda)$ is a model for the unknown density ratio $\beta$ with $f\in\mathcal{H}$
in a reproducing kernel Hilbert space (RKHS) $\mathcal{H}$,
$B_F(\beta,\widetilde{\beta}):=F(\beta)-F(\widetilde{\beta})-\nabla F(\widetilde{\beta})[\beta-\widetilde{\beta}]$ is a Bregman divergence generated by some convex functional $F:L^1(Q)\to\mathbb{R}$,
$\norm{.}$ is a norm on $\mathcal{H}$, and $\lambda>0$ is a regularization parameter.
From the view point of regularization theory, Eq.~\eqref{eq:regularized_Bregman_objective} can be seen as \textit{Bregmanized} density-ratio estimation.
The idea of Bregmanization has been already successfully employed in the context of total variation regularization~\cite{osher2005iterative,resmerita2006error}.

In the following, we exemplify four methods in this class. We refer to~\citet[Section~7]{sugiyama2012density} and~\citet{menon2016linking,kato2021non} for further examples.%
~\\
\begin{example}%
\label{ex:introduction}
~\vspace{-1em}%
\begin{itemize}[leftmargin=1em]
\setlength\itemsep{0em}
    \item The kernel unconstrained least squares importance fitting procedure (KuLSIF)~\citep{kanamori2009least} can be obtained from Eq.~\eqref{eq:regularized_Bregman_objective} by $F(h)=\int_\X (h(x)-1)^2/2\diff Q(x)$ and $g(f)=f$.
    \item The logistic regression approach (LR) employed in~\citet[Section~7]{bickel2009discriminative} can be obtained using $F(h)=\int_\X h(x)\log(h(x))-(1+h(x))\log(1+h(x))\diff Q(x)$ and $g(f)=e^{f}$.
    \item The exponential approach (Exp) in~\citet{menon2016linking} is realized by $F(h)=\int_\X h(x)^{-3/2}\diff Q(x)$, $g(f)=e^{2 f}$.
    \item The square loss approach (SQ) in~\citet{menon2016linking} originates from $F(h)=\int_X 1/(2 h(x) + 2)\diff Q(x)$, $g(f)=\frac{-1+2 f}{2-2 f}$.
    \end{itemize}
\end{example}

However, state-of-the-art error guarantees as~\citet{kanamori2012statistical,que2013inverse,gizewski2022regularization,zellinger2023adaptive} for the methods above are sub-optimal for problems of high regularity, i.e., they don't achieve optimal convergence rates when increasing the number of samples, see Theorem~\ref{thm:error_rates_result} and Figure~\ref{fig:saturation}.
This issue is called \textit{saturation} in inverse problems~\citep{engl1996regularization} and learning theory~\citep{bauer2007regularization,li2022saturation}, and its negative effect in supervised learning has been recently highlighted~\citep{beugnot2021beyond}.

Inspired by the Bregman method~\citep{bregman1967relaxation} and Bregmanization~\citep{osher2005iterative,resmerita2006error}, we approach saturation of Eq.~\eqref{eq:regularized_Bregman_objective} by iterated regularization~\cite{thomas1979approximation} of the form
\begin{align}
    \label{eq:novel_objective}
    f^{\lambda,t+1}:=\argmin_{f\in\mathcal{H}} B_F\!\left(\beta,g(f)\right) + \frac{\lambda}{2} \norm{f-f^{\lambda,t}}^2
\end{align}
with $f^{\lambda,0}=0$.
We prove that iterative density ratio estimation methods, according to Eq.~\eqref{eq:novel_objective}, do not suffer from saturation.
For $t=1$, Eq.~\eqref{eq:novel_objective} includes, e.g., approaches in the example above.
As a consequence, we are able to improve the empirical performance of these approaches on benchmark experiments, using a higher iteration number $t>1$ ($t\leq 10$ in our experiments).

Density ratio estimation is an integral part of many parameter choice methods in domain adaptation, see e.g.~\citet{sugiyama2007covariate,you2019towards,dinu2022aggregation}.
Based on our iteratively regularized approaches, we are able to increase the performance of importance-weighted ensembling in domain adaptation on several large-scale domain adaptation tasks.

Our main contributions are summarized as follows:%
~\vspace{-.5em}%
\begin{itemize}[leftmargin=1em]
\setlength\itemsep{0em}
    \item We introduce iterated Tikhonov regularization for a large class of methods for density ratio estimation.
    \item We provide the (to the best of our knowledge first) non-saturating error bounds for density ratio estimation, which are optimal for the square loss approach~\cite{menon2016linking} in the sense of~\citet{caponnetto2007optimal}.
    \item Our new iteratively regularized methods outperform its non-iteratively regularized versions on average on all benchmark datasets.
    \item Consistent performance improvements are also observed in the application of importance-weighted ensembling of deep unsupervised domain adaptation models.
\end{itemize}

\begin{figure*}[t]
  \centering
  \begin{minipage}[b]{0.41\textwidth}
    \includegraphics[width=\textwidth]{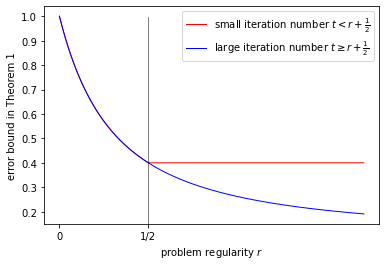}
  \end{minipage}
  \quad\quad
  \begin{minipage}[b]{0.45\textwidth}
    \includegraphics[width=\textwidth]{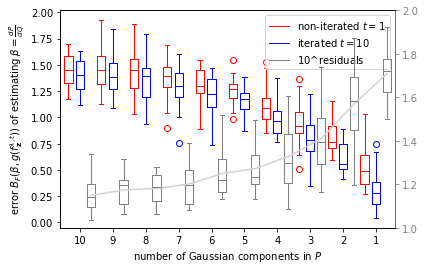}
    
  \end{minipage}
  \caption{Saturation issue of classical methods Eq.~\eqref{eq:regularized_Bregman_objective}
  versus our novel iteratively regularized approach Eq.~\eqref{eq:novel_objective}.
  Left: Error rates proven in Theorem~\ref{thm:error_rates_result} for classical methods (red) and ours (blue). Right: Error for classical KuLSIF~\cite{kanamori2009least} method with $F(h)=\int (h(x)-1)^2/2 q(x)\diff x$ in Eq.~\eqref{eq:regularized_Bregman_objective} (red), and our iteratively regularized approach (blue) applied to Gaussian mixture $P$ and Gaussian distribution $Q$; smaller number of components allows higher regularity index $r$. The residual differences (grey) between the methods increase with higher regularity.
  }
  \label{fig:saturation}
\end{figure*}

\section{Related Work}
\label{sec:related_work}

It was first observed in~\citet{sugiyama2012densitybregman} that many loss functions for density ratio estimation can be modeled by a Bregman divergence $B_F(\beta,g(f))$ as used in Eq.~\eqref{eq:regularized_Bregman_objective}, see also~\citet{menon2016linking,kato2021non} and references therein.

The study of saturation issues in regularization dates back to inverse problems literature~\citep{neubauer1997converse,engl1996regularization} and has been further investigated in learning theory~\citep{bauer2007regularization,caponnetto2007optimal,li2022saturation,lin2020optimal}; which includes recent studies for general (self-concordant) convex loss functions~\cite{beugnot2021beyond} and vector-valued output spaces~\cite{meunier2024optimal}.
Our work can be regarded as an extension of many of these investigations from supervised learning to density ratio estimation.

There are some theoretical results on the large class of algorithms following variations of Eq.~\eqref{eq:regularized_Bregman_objective}, see~\citet{kanamori2012statistical,kato2021non,gizewski2022regularization,zellinger2023adaptive}.
However, to the best of our knowledge, none of these works achieves the optimal error rate discovered in~\citet{caponnetto2007optimal} for supervised learning (for $r>\frac{1}{2}$ in Assumption~\ref{ass:source_condition}).
We are also not aware of any approach provably overcoming the saturation problem of Tikhonov regularization $\norm{f}^2$ in Eq.~\eqref{eq:regularized_Bregman_objective}, except~\citet{nguyen2023regularized} who study iterated Lavrentiev regularization only for the special case of KuLSIF. They include only a small numerical visualization and their pointwise error measure is different from our $B_F$.
Our approach overcomes saturation and provably achieves the optimal rate of square loss.

As noted above, Eq.~\eqref{eq:novel_objective} is inspired by the Bregman method~\cite{bregman1967relaxation} and Bregmanization~\cite{osher2005iterative,resmerita2006error} as studied in Inverse Problems.
It is demonstrated in~\citet{resmerita2006error}, that regularization by Bregman distance can be considered as an instance of \textit{enhancing techniques}, commonly used in image processing.
We refer to the book~\cite{scherzer2009variational} for in-depth convergence studies.

Eq.~\eqref{eq:empirical_loss_objective} identifies the proposed methods as iterated Tikhonov regularizations, which have been studied in Inverse Problems literature, e.g., for linear problems~\cite{hanke1998nonstationary}, for non-linear problems~\cite{scherzer1993convergence}, in the continuous case~\cite{scherzer2001inverse}, for non-quadratic penalties~\cite{tadmor2004multiscale,kindermann2023multiscale} and for non-linear operators combined with non-linear penalty~\cite{modin2019multiscale}.
All these studies can serve as a basis for analysing Eq.~\eqref{eq:novel_objective}.

Our analysis of the sampling behaviour extends fast error rates from kernel regression~\citep{beugnot2021beyond} 
(which relies on~\citet{Engl:96,bauer2007regularization})
to density ratio estimation, using the duality between divergence estimation~\citep{reid2011information} and binary classification~\citep{menon2016linking}, and the self-concordance property of associated loss functions~\citep{bach2010self,marteau2019beyond}.

It is important to note that iterated Tikhonov regularization is intrinsically related to the \textit{proximal point method}~\citep{rockafellar1976monotone}.

The focus of our work is to improve density ratio estimation for methods with theoretical guarantees such as kernel methods and infinitely wide neural networks, see~\citet{Jacot:18}. For general deep learning, our theory does not apply, as additional regularizations appear~\citep{Gidel:19}. However, some recent works identify scientific deep density ratio estimation problems. For example~\citet{rhodes2020telescoping} and~\citet{kato2021non,Kiryo:17} (PU-learning) and~\citet{Srivastava:23} discuss problems (e.g., density-chasm) appearing for very different densities; note that saturation also happens for very similar densities (see~\cref{fig:beugnot}). Notably, \citet{Choi:21} propose an approach by training normalizing flows to obtain closer and simpler densities.
Another important recent line of research is density ratio estimation for correcting deep generative diffusion models~\citep{Kim:23}.
Notably, in~\citet{Kim:24}, the Bregman divergence approach in Eq.~\eqref{eq:regularized_Bregman_objective} is extended to time-dependent correction and they identify further interesting Bregman divergences. 

In our empirical studies, we rely on the benchmark experiments for density ratio estimation of~\citet{kanamori2012statistical} and the large scale-domain adaptation framework of~\citet{ragab2023adatime} with importance weighting based parameter choice as introduced in~\citet{dinu2022aggregation}.

\section{Notation}
\label{sec:notation}

\paragraph*{Class Probability Estimation}

Let us denote by $\mathcal{X}$ a (compact) input space and by $\mathcal{Y}:=\{-1,1\}$ the binary label space and assume that an input $x$ is drawn with probability\footnote{See the supplementary material of~\cite{menon2016linking} for probabilities $\pi\neq \frac{1}{2}$.} $\frac{1}{2}$ from $P$ and $Q$.
That is, there is a probability measure $\rho$ on $\mathcal{X}\times\mathcal{Y}$, with conditional measures
$\rho(x|y=1):=P(x), \rho(x|y=-1):=Q(x)$ and marginal measure $\rho_\mathcal{Y}$ defined as Bernoulli measure assigning probability $\frac{1}{2}$ to both events $y=1, y=-1$, see~\citet{reid2010composite}.
That is, the sample $\z:=(x_i,1)_{i=1}^m\cup (x_i',-1)_{i=1}^n$ can be regarded as an i.i.d.~sample of an $\mathcal{X}\times\mathcal{Y}$-valued random variable $Z$ with measure $\rho$.

For a \textit{loss} function $\ell:\mathcal{Y}\times\mathbb{R}\to\mathbb{R}$,
we measure by $\ell(1,f(x))$ and $\ell(-1,f(x))$, respectively, the error of a classifier $f(x)$ which predicts whether $x$ is drawn from $P$ or $Q$.
A loss function $\ell$ is called \textit{strictly proper composite}~\citep{buja2005loss} if there exists an invertible link function $\Psi:[0,1]\to\mathbb{R}$ so that the Bayes optimal model $f^\ast:=\argmin_{f:\mathcal{X}\to\mathcal{Y}} \mathcal{R}(f)$ with expected risk $\mathcal{R}(f):=\mathbb{E}_{(x,y)\sim\rho}[\ell(x,f(x))]$ is achieved by $f^\ast=\Psi\circ \rho(y=1|x)$.
We denote by $G(u):= u \ell(1,\Psi(u))+(1-u) \ell(-1,\Psi(u))$ the conditional \textit{Bayes risk} of $\ell$.
We will also denote by
$\ell_z:\mathcal{H}\to\mathbb{R}$ the function defined by $\ell_{z}(f):=\ell(y,f(x))$ for $z=(x,y)$.

Throughout this work, we neglect pathological loss functions and only consider strictly proper composite losses $\ell$ with twice differentiable Bayes risk $G$.
\paragraph*{Learning in RKHS}

In this work, we focus on models $f\in\mathcal{H}$ in a reproducing kernel Hilbert space $\mathcal{H}$ with continuous bounded kernel $k:\mathcal{X}\times\mathcal{X}\to\mathbb{R}$.
Furthermore, for simplicity,  for any considered loss function $\ell$, we will assume that the expected risk minimizer $f_\mathcal{H}:=\argmin_{f\in\mathcal{H}} \mathcal{R}(f)$ exists and is in $\mathcal{H}$.

\section{Iterated Regularization}
\label{sec:iterated_tikhonov}

We rely on the following lemma to approximate the solution of Eq.~\eqref{eq:novel_objective} from the data $\z$.
\begin{lemma}[{\citet{menon2016linking}}]
\label{lemma:bregman_form_of_loss}
Any strictly proper composite loss $\ell$ with invertible link $\Psi:[0,1]\to\mathbb{R}$ and twice differentiable Bayes risk $G:[0,1]\to\mathbb{R}$ satisfies
\begin{align}
    \label{eq:bregman_identity_for_density_ratio}
    \frac{1}{2} B_F(\beta,g(f)) = \mathcal{R}(f)-\mathcal{R}(f^\ast)
\end{align}
with $F(h):=-\int_\X (1+h(x)) G\!\left(\frac{h(x)}{1+h(x)}\right)\diff Q(x)$ and $g(f):=\frac{\Psi^{-1}\circ f}{1-\Psi^{-1}\circ f}$.
\end{lemma}
Notably, the construction in Lemma~\ref{lemma:bregman_form_of_loss} includes all approaches in Example~\ref{ex:introduction}.
\\
\begin{example}%
\label{ex:losses}
~\vspace{-1em}%
\begin{itemize}[leftmargin=1em]
\setlength\itemsep{0em}
    \item\sloppy For KuLSIF, Eq.~\eqref{eq:bregman_identity_for_density_ratio} holds with $\ell(-1,v)=\frac{1}{2}v^2$, $\ell(1,v)=-v$ and $\Psi(v)=\frac{v}{1-v}$.
    \item LR realizes Eq.~\eqref{eq:bregman_identity_for_density_ratio} with $\ell(-1,v)=\log(1+e^v)$, $\ell(1,v)=\log(1+e^{-v})$ and the link function $\Psi(v)=-\log(\frac{1}{v}-1)$.
    \item Exp can be obtained using $\ell(-1,v)=e^v, \ell(1,v)=e^{-v}$ and $\Psi(v)=-\frac{1}{2}\log(\frac{1}{v}-1)$.
    \item SQ uses $\ell(-1,v)=(1+v)^2, \ell(1,v)=(1-v)^2$ and $\Psi(v)=\frac{v+1}{2}$. 
    \end{itemize}
\end{example}

Eq.~\eqref{eq:bregman_identity_for_density_ratio} allows us to estimate the density ratio $\beta$ by $g(f^{\lambda,t}_{\z})$ with $f^{\lambda,0}_{\z}=0$ and the empirical risk minimizer
\begin{align}
    \label{eq:empirical_loss_objective}
    f^{\lambda,t+1}_{\z}:=\argmin_{f\in\mathcal{H}} \frac{1}{|\z|}\sum_{i=1}^{m+n} \ell(y_i,f(x_i))+ \frac{\lambda}{2}\norm{f-f^{\lambda,t}_{\z}}^2.
\end{align}
In the following, we will discuss, how to compute $f^{\lambda,t+1}_{\z}$.

\subsection{Algorithmic Realization} \label{subsec:algo}

To solve Eq.~\eqref{eq:empirical_loss_objective}, we follow the 
representer theorem~\citep{wahba1990spline,scholkopf2001generalized}
, make the model ansatz $f(\cdot)=\sum_{i=1}^{m+n} \alpha_i k(x_i,\cdot)$ and apply the conjugate gradient (CG) method for minimization, except for KuLSIF in Example \ref{ex:losses}, for which an explicit solution is available, see Appendix \ref{sec:app_kulsif}.
Originally, the CG method was developed to solve linear equations with a positive semi-definite matrix. There, the goal is to compute a sequence of conjugate search directions, where two vectors are considered conjugate if their dot product with respect to the given matrix is zero. These search directions arise as gradients of an associated quadratic optimization problem, see e.g. \citet[Section 11.3]{golub2013matrix}. 

The CG-method can be generalized to deal with nonlinear problems like Eq.~\eqref{eq:empirical_loss_objective}, which leads to a minimization problem of a strongly convex smooth function in each iteration step. Therefore, several CG approaches can be applied, see, e.g., \citet{hager2006survey}. 
The method which is implemented in Python Scipy \citep{virtanen2020scipy} and which we are using, is based on an additional Polak-Ribiere line search scheme and it is guaranteed to converge globally in our case \citep{hager2006survey, cohen1972rate}. In the upcoming section we will also briefly discuss how this algorithm choice affects our theoretical guarantees.

\section{Learning Theoretic Analysis}
\label{sec:error_bound}


In learning theory, the regularity of learning problems is traditionally encoded by \textit{source conditions}, see, e.g.,~\citet{bauer2007regularization,caponnetto2007optimal,marteau2019beyond}.
\begin{assumption}[source condition]
    \label{ass:source_condition}
    There exist $r>0, v\in\mathcal{H}$ such that $f_{\mathcal{H}}=\Hess(f_{\mathcal{H}})^r v$ with the expected Hessian $\Hess(f):=\mathbb{E}[\nabla^2 \ell_Z(f)]$.
\end{assumption}
To interpret Assumption~\ref{ass:source_condition}, let us assume $\ell_z(f):=(y-f(x))^2$ is the square loss.
Then Assumption~\ref{ass:source_condition} recovers the polynomial source condition used in~\citet{caponnetto2007optimal}.
Starting from $r=0$, the regularity of $f_\mathcal{H}$ increases with increasing $r$, essentially reducing the number of eigenvectors of the covariance operator $T$ needed to well approximate $f_\mathcal{H}$.
However, it is known that optimal learning rates are only possible under a further condition on the \textit{capacity} of the space $\mathcal{H}$~\citep{marteau2019beyond}.
\begin{assumption}[capacity condition]
    \label{ass:capacity_condition}
    There exist $\alpha\geq 1$ and $S>0$ such that $\mathrm{df}_\lambda\leq S \lambda^{-\frac{1}{\alpha}}$ with the degrees of freedom
    \begin{align}
        \mathrm{df}_\lambda:=\mathbb{E}\left[\norm{\Hess_\lambda(f_\mathcal{H})^{-\frac{1}{2}}\nabla\ell_Z(f_\mathcal{H})}^2\right]
    \end{align}
    and $\Hess_\lambda(f):=\Hess(f)+\lambda I$.
\end{assumption}
The \textit{degrees of freedom} term $\mathrm{df}_\lambda$ is also called \textit{Fisher information}~\citep{van2000asymptotic}, appears in spline smoothing~\citep{wahba1990spline}, and reduces to the \textit{effective dimension}~\citep{caponnetto2005empirical} for the square loss.
In the latter case, a bigger $\alpha$ implies that fewer eigenvectors are needed to approximate elements from the image space of $T$ with a given capacity~\citep{blanchard2018optimal}.

Of course, it is not possible to achieve fast convergence with \textit{any} loss function.
We therefore focus on the large class of convex \textit{pseudo self-concordant} losses, see~\citet{bach2010self,ostrivskii2021finite}.
\begin{assumption}[pseudo self-concordance]
    \label{ass:self-concordance}
    For any $y\in\Y$, the function $\ell_y:\mathbb{R}\to\mathbb{R}$ defined by $\ell_y(\eta):=\ell(y,\eta)$ is convex, three times differentiable and satisfies
    \begin{align}
    \label{eq:self_concordant}
        \left|\ell_y'''(\eta)\right|\leq \ell_y''(\eta).%
    \end{align}%
\end{assumption}
It is important to note that all methods in Example~\ref{ex:losses} (cf.~\citet{bach2010self} for LR) satisfy Assumption~\ref{ass:self-concordance}.
This allows us to make the following statement.
\begin{theorem}
\label{thm:error_rates_result}
    Let the Assumptions~\ref{ass:source_condition}--\ref{ass:technical} be satisfied, let $t\in\mathbb{N}\setminus\{0\}$, $\delta\in (0,1]$ and denote by $s:=\min\{r+\frac{1}{2},t\}$.
    Then there exist quantities $c,C>0$ not depending on $m,n$ such that the minimizer $f_{\z}^{\lambda,t}$ of Eq.~\eqref{eq:empirical_loss_objective} with
    \begin{align}
    \label{eq:balancing_lambda}
        \lambda=c (m+n)^{-\frac{\alpha}{1+\alpha(2 r+1)}}
    \end{align}
    satisfies    
    \begin{align*}
        B_F(\beta,g(f_{\z}^{\lambda,t}))-B_F(\beta,g(f_\mathcal{H}))\leq C\cdot (m+n)^{-\frac{2s\alpha}{2s\alpha+1}}
    \end{align*}
    with probability at least $1-\delta$ and for large enough sample sizes $m, n$.
\end{theorem}

Theorem~\ref{thm:error_rates_result} provides a fast error rate whenever $t\geq r+\frac{1}{2}$.
For learning problems of high regularity $r> \frac{1}{2}$, an iteration number $t>1$ is required. Otherwise the error will saturate at a rate of $(m+n)^{-\frac{2}{3}}$, which is slower than the achievable rate $(m+n)^{-\frac{2r\alpha+\alpha}{2r\alpha+\alpha+1}}$.

\begin{remark}
    For $r\leq \frac{1}{2}, t=1$, Theorem~\ref{thm:error_rates_result} recovers the state-of-the-art error rate of~\citet{zellinger2023adaptive}.
    In addition, it extends the state of the art for current approaches for $r> \frac{1}{2}, t=1$.
    In case of square loss, it is optimal (cf.~\citet{caponnetto2007optimal,blanchard2018optimal}).
    For $r> \frac{1}{2}, t>1$, Theorem~\ref{thm:error_rates_result} provides the first non-saturating error rate for density ratio estimation.
\end{remark}
\begin{remark}
    The parameter choice in Eq.~\eqref{eq:balancing_lambda} balances a bias and a variance term in a tight error bound~\citep{marteau2019beyond} and is thus rate optimal, see, e.g.,~\citet[Section~4.3]{lu2013regularization}.
    Eq.~\eqref{eq:balancing_lambda} is \textit{a priori} and requires the knowledge of the regularity index $r$. This issue can be resolved using an \textit{a posteriori} parameter choice rule, e.g.,~\citet{de2010adaptive,zellinger2023adaptive}.
\end{remark}

\begin{remark} \label{eq:rem_numerical}
For nonlinear CG algorithms and strictly convex smooth objectives, $\log(\varepsilon^{-1})$ gradient and function evaluations are required to achieve an error tolerance $\varepsilon$, see e.g.~\citet{neumaier2022globally, chan2022nonlinear}. 
Combining these findings with~ arguments from~\citet[Section~3.2]{beugnot2021beyond}, we see that $\varepsilon$ influences Theorem~\ref{thm:error_rates_result} in an additive way, i.e. with probability at least $1-\delta$,
\begin{align}
\label{eq:main_bound_numerical}
B_F(\beta,g(\bar{f}_{\z}^{\lambda,t}))&-B_F(\beta,g(f_\mathcal{H}))\nonumber\\
&\le\tilde{C} \!\left( (m+n)^{-\frac{2s\alpha}{2s\alpha+1}}+\varepsilon\right)
\end{align}
for a numerical approximation $\bar{f}_{\z}^{\lambda,t}$ of $f^{\lambda,t}_{\z}$ computed as described in Subsection \ref{subsec:algo} and $\tilde{C}>0$ independent of $n,m$ and $\varepsilon$, see Appendix~\ref{appendix:optimization_error} for details.
\end{remark}
However, as it is typically the case in learning theory~\citep{vapnik2013nature,zhang2021understanding}, the constant $C$ in Theorem~\ref{thm:error_rates_result} can be quite large.
Therefore, we also test the performance of the proposed iterated regularization on benchmark datasets.

\section{Empirical Evaluations}
\label{sec:experiments}

We investigate the performance of iterated density ratio estimation regarding three aspects below.
We include additional experiments conducted during the rebuttal period that helped to improve our work. More precisely, we compare to a SOTA domain adaptation method~\citep{dinu2022aggregation}, combine our approach with telescoping from~\citet{rhodes2020telescoping}, and extend our approach to Deep Learning methods.
Please find further details in Appendix~\ref{appendix:sec:detaile_experiments}.

\paragraph*{Sample convergence in highly regular problems.}
To investigate the sample convergence of iterated approaches compared to their non-iterated regularized versions, i.e., Theorem~\ref{thm:error_rates_result}), we follow the study of~\citet{beugnot2021beyond}.
From their dataset, density ratios with known smoothness can be extracted, see dataset description below.

\paragraph*{Accuracy improvement for known density ratios.}
To study the accuracy of the iteratively regularized estimations of density ratios, we follow investigations of~\citet{kanamori2012statistical} and construct high dimensional data with exact known density ratios.
Unfortunately, for real-world problems, exact density ratios are not available and the question arises whether synthetic data is too special.
In a small ablation study, we therefore also indicate performance improvements of iteration on a synthetically constructed density ratio.

\paragraph*{Importance weighted ensembling in deep domain adaptation.} \label{sec:agg}
To test the effect of iterated regularization in a real real world setting we rely on large-scale (over 9000 trained neural networks) state-of-the-art experiments for re-solving parameter choice issues in unsupervised domain adaptation.
There, we are given two datasets, a source dataset ${\bf x}'\subset\X$ drawn from $Q$ with labels ${\bf y}'\subset \mathbb{R}$ and an unlabeled target dataset ${\bf x}\subset\X$ drawn from $P$ without labels.
The goal is to learn a model $f:\X\to\Y$ with low risk $\mathcal{R}(f):=\mathbb{E}_{(x,y)\sim P}[\ell(y,f(x))]$ on future target data from $P$.
In the state-of-the-art approach of~\citet{dinu2022aggregation} the key problem of choosing algorithm parameters in this setting is approached by ensemble learning.
More precisely, for a given domain adaptation algorithm $A((\x',\y'),\x)=f^{\alpha_1,\ldots,\alpha_b}$ with $b$ different parameters $\alpha_1,\ldots,\alpha_b$, they compute several model candidates $f_1,\ldots,f_l$ with $l$ different parameter settings, and then compute a simple ensemble model $A_{c_1,\ldots,c_l}(x)=\sum_{i=1}^l c_i^\ast f_i(x)$ out of $f_1,\ldots,f_l$ using the minimizers $c_1^\ast,\ldots,c_l^\ast$ of
\begin{align}
    \label{eq:aggregation}
    \min_{c_1,\ldots,c_l\in\mathbb{R}}\frac{1}{|\x'|}\sum_{(x',y')\in (\x,\y)} \widehat{\beta}(x')\cdot \ell\!\left(y',A_{c_1,\ldots,c_l}(x')\right)
\end{align}
with an estimate $\widehat{\beta}$ of the density ratio $\beta=\frac{\diff P}{\diff Q}$.
We refer to~\citet{sugiyama2007covariate} as a motivation for using $\beta$ as importance weight in domain adaptation.

\subsection{Datasets}
\label{subsec:datasets}

\paragraph*{Known regularity}
To illustrate Theorem~\ref{thm:error_rates_result}, we adapt an example of~\citet[Section 4]{beugnot2021beyond} to density ratio estimation.
In this example, the regularity index $r$ and the capacity parameter $\alpha$ are known by design.
Let therefore $\mathcal{X}=[0,1]$ and the kernel of $\mathcal{H}$ given by $k(x,y)=h_{\alpha}(x,y)=1+\sum_{l \in \mathbb{Z} \setminus 0} \frac{e^{2 i l \pi(x-y)}}{|l|^{\alpha}}$, for which an explicit formula is known for even $\alpha$~\citep[Page 22]{wahba1990spline}, which is essentially determined by Bernoulli polynomials \citep[Section~24.2]{Olver:10}.
If we fix $\rho_{\mathcal{X}}$ as the uniform distribution, $\ell$ the logistic loss and $f_{\mathcal{H}}(x)=h_{(r+\frac12)\alpha+\frac12}(0,x)$, then $f_{\mathcal{H}} \in \mathcal{H}$ and Assumptions~\ref{ass:source_condition} and~\ref{ass:capacity_condition} are satisfied~\citep{rudi2017generalization}.
Moreover, it is shown in~\citet[Lemma A1]{ciliberto2020general} that 
\begin{align} \label{eq:f_h_rep_toy}
f_{\mathcal{H}}(x)=\argmin_{z } \mathbb{E}_{y \sim \rho(y|x)}[ \ell(y,z)]. 
\end{align}
From Eq.~\eqref{eq:f_h_rep_toy} we obtain~\citep[Appendix E3]{beugnot2021beyond} $\rho(y|x)=\frac{1}{1+e^{-yf_{\mathcal{H}}(x)}}$.
To obtain a formula for $P$ and $Q$, we apply Bayes' Theorem: $\rho_\Y(y)=\int_{\mathcal{X}} \frac{1}{1+e^{-yf_{\mathcal{H}}(x)}} d\rho(x)$, $P(x)=\rho(x|y=1)= \frac{\rho(y=1|x)}{\rho(y=1)}$ and $Q(x)=\rho(x|y=-1) =  \frac{\rho(y=-1|x)}{\rho(y=-1)}$.
However, the inverse link in Example~\ref{ex:losses} changes from $\Psi^{-1}(v)= \frac{1}{1+e^{-v}}$ to the new link function~\citep[Lemma~5]{menon2016linking}
$$
\widetilde{\Psi}^{-1}(v)=\frac{\Psi^{-1}(v)\pi}{\Psi^{-1}(v)/\pi(1-2\pi)+1-\pi}, \pi=\rho_{\Y}(y=1),
$$
resulting in the density ratio $\beta(x)=\frac{1-\pi}{\pi}\frac{\rho(y=1|x)}{1-\rho(y=1|x)}$ and a different $B_F$, see~\citet[Section~4.2]{menon2016linking}.
We investigate the sample convergence in the distribution-independent error $\norm{\beta-g(f_\z^{\lambda,t})}_{L^1([0,1])}$.

In addition to the dataset adapted from~\citet{beugnot2021beyond}, we also construct density ratios between one-dimensional Gaussian mixtures $P$ and a Gaussian distribution $Q$, with the goal of approximation in a Gaussian RKHS, see Figure~\ref{fig:saturation} (right).
A greater number of components allows for rougher density ratios, according to Assumption~\ref{ass:source_condition}.
This may be observed, e.g., from the Fourier decomposition of $f_{\mathcal{H}}$ with respect to the eigenbasis associated to $\Hess(f_{\mathcal{H}})$.
A higher number of Gaussian components in $P$ intuitively yields heavier oscillations, and therefore emphasizes higher order Fourier modes.
This, in turn, gives a smaller decay rate of the associated Fourier coefficients, implying a reduced regularity according to Assumption~\ref{ass:source_condition}.

\paragraph*{Known density ratios}
Following~\citet{kanamori2012statistical}, we generate ten different datasets using Gaussian mixture models in the $50$-dimensional space with different numbers $\{1,2,3\}$ of components for source and target distribution. Each component has different a mean in $[ 0,0.5 ]^{50}$ and accordingly sampled covariance matrices.
In all datasets, we have full knowledge of the exact density ratio, as it is calculated as the fraction of the Gaussian mixtures, see Appendix~\ref{appendix:geometric_dataset}.
To debate whether our method is only suited for Gaussian mixtures, we follow~\citet[Section~5.3]{kanamori2012statistical} and utilize the breast cancer dataset~\citep{Street:93} consisting of $569$ samples in the $30$-dimensional space with binary labels.
There, a target density ratio model $\beta$ is constructed using a binary support vector machine~\citep{platt1999probabilistic} to estimate the class posterior, and, by subsequently re-labeling the miss-labeled data. Note, that the i.i.d.~assumption of statistical learning is violated in this case, making the above process more an ablation study than a statistical performance benchmark.

\paragraph*{Domain adaptation: text data}
The Amazon reviews dataset~\citep{blitzer2006domain} consists of bag-of-words representations of text reviews from four domains (books, DVDs, kitchen, and electronics) with binary labels indicating the class of review. Twelve domain adaptation tasks are constructed using every domain, once as source domain and once as target domain.

\begin{table*}[ht]
\scalebox{0.7}{
    \begin{tabular}{l c c c c | c c c c}
    \toprule
    \multicolumn{9}{c}{\textbf{Geometric Figures}}\\
    \cmidrule{2-9}
      & \multicolumn{4}{c|}{\textbf{No Iteration}} & \multicolumn{4}{c}{\textbf{Iteration}}\\
      \cmidrule{2-9}
    \textbf{Dataset} & KuLSIF & Exp & LR & SQ & KuLSIF & Exp & LR & SQ \\
    \midrule
    c3,d1.70 & $8.616 (\pm 0.011)$  & $8.322 (\pm 0.009)$ & $8.840 (\pm 0.021)$ & $9.170 (\pm 0.011)$ &  $\bf{8.605}(\pm 0.005)$ & $\bf{8.827(\pm 0.014)}$ & $8.840(\pm 0.021)$ & $\bf{8.840}(\pm 0.010)$\\
    c2,d1.72 & $13.031(\pm 0.005)$ & $12.994(\pm 0.015)$ & $13.255(\pm 0.013)$ & $13.537(\pm 0.027)$ &  $\bf{13.026}(\pm 0.009)$ & $\bf{12.855}(\pm 0.017)$ & $13.255(\pm 0.011)$ & $\bf{13.255}(\pm 0.032)$\\
    c2,d1.59 & $12.625(\pm 0.005)$ & $19.748(\pm 0.037)$ & $12.829(\pm 0.014)$ & $13.056(\pm 0.015)$ &  $\bf{12.622}(\pm 0.007)$ & $\bf{12.521}(\pm 0.051)$ & $12.829(\pm 0.014)$ & $\bf{12.829}(\pm 0.012)$\\
    c1,d1.55 & $11.813(\pm 0.007)$ & $14.477(\pm 0.103)$& $12.001(\pm 0.023)$ & $12.179(\pm 0.013)$ &  $\bf{11.805}(\pm 0.006)$ & $\bf{11.552}(\pm 0.109)$ & $12.001(\pm 0.023)$ & $\bf{12.001}(\pm 0.019)$\\
    c2,d1.78 & $9.632(\pm 0.003)$ & $18.008(\pm 0.069)$ & $9.802(\pm 0.035)$ & $9.990(\pm 0.006)$ &  $\bf{9.626(\pm 0.019)}$ & $19.895(\pm 0.021)$ & $9.802(\pm 0.037)$ & $\bf{9.802}(\pm 0.017)$\\
    c2,d1.55 & $10.371(\pm 0.007)$ & $9.774(\pm 0.019)$ & $10.555(\pm 0.059)$ & $10.757(\pm 0.023)$ &  $\bf{10.363}(\pm 0.003)$ & $\bf{9.646}(\pm 0.017)$ & $10.555(\pm 0.059)$ & $\bf{10.555}(\pm 0.025)$\\
    c3,d1.57 & $12.014(\pm 0.003)$ & $70.827(\pm 0.926)$ & $12.214(\pm 0.037)$ & $14.048(\pm 0.029)$ &  $\bf{12.006}(\pm 0.007)$ & $\bf{12.082}(\pm 0.014)$ & $12.214(\pm 0.019)$ & $\bf{12.214}(\pm 0.046)$\\
    c2,d1.61 & $11.614(\pm 0.004)$ & $11.282(\pm 0.034)$ & $11.800(\pm 0.008)$ & $12.242(\pm 0.007)$ &  $\bf{11.609}(\pm 0.004)$ & $\bf{15.275}(\pm 0.044)$ & $11.800(\pm 0.009)$ & $\bf{11.800}(\pm 0.008)$\\
    c3,d1.46 & $12.803(\pm 0.009)$ & $12.616(\pm 0.008)$ & $12.971(\pm 0.007)$ & $13.159(\pm 0.006)$ &  $\bf{12.795}(\pm 0.005)$ & $\bf{12.218}(\pm 0.013)$ & $12.971(\pm 0.004)$ & $\bf{12.971}(\pm 0.004)$\\
    c1,d1.63 & $9.527(\pm 0.006)$ & $9.704(\pm 0.009)$ & $9.732(\pm 0.014)$ & $9.965(\pm 0.015)$ &  $\bf{9.517}(\pm 0.010)$ & $14.026(\pm 0.016)$ & $9.732(\pm 0.010)$ & $\bf{9.732}(\pm 0.008)$\\
    \midrule
    Avg & $11.205(\pm 0.006)$ & $18.775(\pm 0.123)$ & $11.400(\pm 0.023)$ & $11.810(\pm 0.015)$ & $\bf{11.198}(\pm 0.008)$ & $\bf{12.890}(\pm 0.032)$ & $11.400(\pm 0.021)$ & $\bf{11.400}(\pm 0.018)$\\
    \bottomrule
    \end{tabular}
    }
    \caption{Mean and standard deviation (after $\pm$) of twice the Bregman divergence error on the geometrically constructed datasets following~\citet{kanamori2012statistical} over ten different sample draws from $P$ and $Q$.}
\end{table*}

\begin{table*}[!ht]
\label{tab:mdn}
\scalebox{0.7125}{
\vspace{15em}
\begin{tabular}{ l c c c c | c c c c}
\toprule
    \multicolumn{9}{c}{\textbf{Domain Adaptation: MiniDomainNet}}\\
    \cmidrule{2-9}
 & \multicolumn{4}{c|}{\textbf{No Iteration}} & \multicolumn{4}{c}{\textbf{Iteration}}\\
 \cmidrule{2-9}
 \textbf{DA-Method} & KuLSIF & Exp & LR & SQ & KuLSIF & Exp & LR & SQ \\
 \midrule
 MMDA    & $0.527 (\pm 0.006)$    & $0.528 (\pm 0.009)$    & $0.528 (\pm 0.012)$    & $0.062 (\pm 0.010)$    & $0.527 (\pm 0.007)$     & $0.528 (\pm 0.009)$    & $0.528 (\pm 0.009)$    & $0.062 (\pm 0.012)$ \\
 CoDATS    & $0.536 (\pm 0.012)$     & $0.532 (\pm 0.017)$    & $0.530 (\pm 0.023)$    & $0.061 (\pm 0.025)$    & $\bf{0.539 (\pm 0.012)}$    & $0.532 (\pm 0.017)$    & $\bf{0.532 (\pm 0.017)}$    & $\bf{0.069 (\pm 0.025)}$ \\
 DANN    & $0.531 (\pm 0.010)$   & $0.522 (\pm 0.021)$    & $0.520 (\pm 0.021)$    & $\bf{0.060} (\pm 0.026)$    & $\bf{0.539 (\pm 0.006)}$    & $0.522 (\pm 0.021)$    & $\bf{0.522 (\pm 0.021})$    & $0.059 (\pm 0.025)$ \\
 CDAN    & $0.531 (\pm 0.012)$    & $0.531 (\pm 0.015)$    & $0.524 (\pm 0.020)$    & $0.062 (\pm 0.020)$    & $\bf{0.533 (\pm 0.011)}$    & $0.531 (\pm 0.015)$    & $\bf{0.531 (\pm 0.015)}$    & $0.062 (\pm 0.027)$ \\
 DSAN    & $0.539 (\pm 0.009)$    & $0.532 (\pm 0.012)$    & $0.527 (\pm 0.019)$    & $0.068 (\pm 0.015)$    & $\bf{0.546 (\pm 0.005)}$     & $0.532 (\pm 0.012)$    & $\bf{0.532 (\pm 0.012)}$    & $\bf{0.069 (\pm 0.015)}$ \\
 DIRT    & $0.517 (\pm 0.010)$    & $0.386 (\pm 0.013)$    & $\bf{0.520} (\pm 0.012)$    & $\bf{0.082} (\pm 0.015)$    & $\bf{0.515 (\pm 0.009)}$    & $0.386 (\pm 0.013)$    & $0.519 (\pm 0.013)$    & $0.081 (\pm 0.014)$ \\
 AdvSKM    & $0.516 (\pm 0.026)$     & $0.515 (\pm 0.177)$    & $0.512 (\pm 0.023)$    & $0.072 (\pm 0.025)$    & $0.516 (\pm 0.028)$    & $0.515 (\pm 0.177)$    & $\bf{0.515 (\pm 0.024)}$    & $\bf{0.073 (\pm 0.024)}$ \\
 HoMM    & $0.531 (\pm 0.011)$     & $0.529 (\pm 0.015)$    & $0.521 (\pm 0.015)$    & $0.078 (\pm 0.025)$    & $0.531 (\pm 0.010)$    & $0.529 (\pm 0.015)$    & $\bf{0.529 (\pm 0.015)}$    & $0.078 (\pm 0.025)$ \\
 DDC    & $\bf{0.517} (\pm 0.012)$   & $0.517 (\pm 0.013)$    & $0.514 (\pm 0.012)$    & $0.074 (\pm 0.017)$    & $0.516 (\pm 0.010)$    & $0.517 (\pm 0.013)$    & $\bf{0.517 (\pm 0.013)}$    & $0.074 (\pm 0.018)$ \\
 DeepCoral    & $0.535 (\pm 0.008)$    & $0.528 (\pm 0.012)$    & $0.530 (\pm 0.015)$    & $0.070 (\pm 0.012)$    & $\bf{0.539 (\pm 0.006)}$    & $0.528 (\pm 0.012)$    & $0.528 (\pm 0.012)$    & $0.070 (\pm 0.011)$ \\
 CMD    & $0.529 (\pm 0.009)$   & $0.524 (\pm 0.011)$    & $0.521 (\pm 0.012)$    & $0.091 (\pm 0.013)$    & $\bf{0.533 (\pm 0.007)}$    & $0.524 (\pm 0.011)$    & $\bf{0.524 (\pm 0.011)}$    & $0.091 (\pm 0.013)$ \\
\midrule
 Avg.      & $0.528 (\pm 0.011)$    & $0.513 (\pm 0.029)$    & $0.523 (\pm 0.017)$    & $0.071 (\pm 0.018)$    & $\bf{0.530 (\pm 0.010)}$    & $0.513 (\pm 0.029)$    & $\bf{0.525 (\pm 0.015)}$    & $\bf{0.072 (\pm 0.019)}$ \\
 \bottomrule
\end{tabular}
}
~\\~\\~\\
\scalebox{0.7125}{
\begin{tabular}{ l c c c c | c c c c}
\toprule
    \multicolumn{9}{c}{\textbf{Domain Adaptation: Amazon Reviews}}\\
    \cmidrule{2-9}
  & \multicolumn{4}{c|}{\textbf{No Iteration}} & \multicolumn{4}{c}{\textbf{Iteration}}\\
  \cmidrule{2-9}
 \textbf{DA-Method} & KuLSIF & Exp & LR & SQ & KuLSIF & Exp & LR & SQ \\
 \midrule
 MMDA    & $0.786 (\pm 0.011)$    & $\bf{0.584} (\pm 0.062)$    & $0.784 (\pm 0.010)$    & $0.216 (\pm 0.012)$    & $\bf{0.789 (\pm 0.010)}$    & $0.581 (\pm 0.062)$    & $\bf{0.788 (\pm 0.010)}$    & $0.216 (\pm 0.012)$ \\
 CoDATS    & $0.795 (\pm 0.012)$    & $0.552 (\pm 0.091)$    & $0.794 (\pm 0.010)$    & $\bf{0.209} (\pm 0.041)$    & $\bf{0.798 (\pm 0.011)}$    & $\bf{0.562 (\pm 0.091)}$    & $\bf{0.796 (\pm 0.010)}$    & $0.208 (\pm 0.011)$ \\
 DANN    & $0.794 (\pm 0.013)$     & $\bf{0.555} (\pm 0.104)$    & $0.793 (\pm 0.012)$    & $0.207 (\pm 0.015)$    & $\bf{0.800 (\pm 0.009)}$    & $0.553 (\pm 0.101)$    & $\bf{0.798 (\pm 0.013)}$    & $\bf{0.232 (\pm 0.030)}$ \\
 CDAN    & $0.787 (\pm 0.012)$     & $0.568 (\pm 0.066)$    & $0.787 (\pm 0.012)$    & $\bf{0.247} (\pm 0.013)$    & $\bf{0.790 (\pm 0.010)}$    & $0.568 (\pm 0.066)$    & $\bf{0.789 (\pm 0.010)}$    & $0.213 (\pm 0.015)$ \\
 DSAN    & $0.794 (\pm 0.011)$    & $0.561 (\pm 0.050)$    & $0.793 (\pm 0.012)$    & $\bf{0.207} (\pm 0.012)$    & $\bf{0.800 (\pm 0.009)}$     & $\bf{0.566 (\pm 0.048)}$    & $\bf{0.796 (\pm 0.010)}$    & $0.204 (\pm 0.060)$ \\
 DIRT    & $0.787 (\pm 0.011)$     & $\bf{0.561} (\pm 0.093)$    & $0.788 (\pm 0.012)$    & $0.212 (\pm 0.015)$    & $\bf{0.794 (\pm 0.011)}$    & $0.559 (\pm 0.099)$    & $\bf{0.791 (\pm 0.012)}$    & $\bf{0.246 (\pm 0.011)}$ \\
 AdvSKM    & $0.779 (\pm 0.011)$     & $\bf{0.541} (\pm 0.041)$    & $0.777 (\pm 0.012)$    & $0.223 (\pm 0.010)$    & $\bf{0.780 (\pm 0.008)}$    & $0.540 (\pm 0.039)$    & $\bf{0.779 (\pm 0.009)}$    & $0.223 (\pm 0.067)$ \\
 HOMM   & $0.777 (\pm 0.011)$   & $0.558 (\pm 0.098)$    & $0.774 (\pm 0.013)$    & $\bf{0.227} (\pm 0.015)$    & $0.777 (\pm 0.010)$    & $\bf{0.562 (\pm 0.104)}$    & $\bf{0.776 (\pm 0.011)}$    & $0.224 (\pm 0.010)$ \\
 DDC    & $0.780 (\pm 0.010)$    & $0.568 (\pm 0.094)$    & $0.778 (\pm 0.011)$    & $\bf{0.226} (\pm 0.016)$    & $0.780 (\pm 0.010)$    & $\bf{0.572 (\pm 0.100)}$    & $\bf{0.779 (\pm 0.011)}$    & $0.222 (\pm 0.010)$ \\
 DeepCoral    & $0.784 (\pm 0.009)$   & $0.569 (\pm 0.094)$    & $0.781 (\pm 0.011)$    & $\bf{0.223} (\pm 0.013)$    & $0.784 (\pm 0.009)$    & $\bf{0.573 (\pm 0.098)}$    & $\bf{0.783 (\pm 0.010)}$    & $0.217 (\pm 0.009)$ \\
 CMD    & $0.790 (\pm 0.010)$     & $\bf{0.577} (\pm 0.103)$    & $0.786 (\pm 0.010)$    & $0.215 (\pm 0.013)$    & $\bf{0.794 (\pm 0.008)}$    & $0.574 (\pm 0.100)$    & $\bf{0.789 (\pm 0.008)}$    & $\bf{0.224 (\pm 0.011)}$ \\
\midrule
 Avg.      & $0.787 (\pm 0.011)$    & $0.563 (\pm 0.082)$    & $0.785 (\pm 0.011)$    & $0.219 (\pm 0.016)$    & $\bf{0.790 (\pm 0.010)}$    & $\bf{0.565 (\pm 0.083)}$    & $\bf{0.788 (\pm 0.011)}$    & $\bf{0.221 (\pm 0.023)}$ \\
\bottomrule
\end{tabular}
}
~\\~\\~\\
\scalebox{0.7125}{
\begin{tabular}{ l c c c c | c c c c}
\toprule
    \multicolumn{9}{c}{\textbf{Domain Adaptation: HHAR}}\\
    \cmidrule{2-9}
  & \multicolumn{4}{c|}{\textbf{No Iteration}} & \multicolumn{4}{c}{\textbf{Iteration}}\\
  \cmidrule{2-9}
 \textbf{DA-Method} & KuLSIF & Exp & LR & SQ & KuLSIF & Exp & LR & SQ \\
 \midrule
 MMDA    &  $0.780 (\pm 0.007)$    & $0.670 (\pm 0.013)$    & $0.773 (\pm 0.018)$    & $0.024 (\pm 0.006)$    & $0.780 (\pm 0.007)$     & $\bf{0.679 (\pm 0.018)}$    & $0.773 (\pm 0.018)$    & $0.024 (\pm 0.006)$ \\
 CoDATS    & $0.723 (\pm 0.029)$     & $0.776 (\pm 0.166)$    & $0.779 (\pm 0.029)$    & $0.002 (\pm 0.007)$    & $\bf{0.789 (\pm 0.030)}$    & $\bf{0.779 (\pm 0.113)}$    & $0.779 (\pm 0.029)$    & $0.002 (\pm 0.007)$ \\
 DANN    & $0.697 (\pm 0.170)$    & $0.785 (\pm 0.015)$    & $0.795 (\pm 0.016)$    & $0.001 (\pm 0.004)$    & $\bf{0.796 (\pm 0.008)}$   & $\bf{0.795 (\pm 0.016)}$    & $0.795 (\pm 0.016)$    & $0.001 (\pm 0.004)$ \\
 CDAN    & $0.792 (\pm 0.130)$    & $0.706 (\pm 0.021)$    & $0.788 (\pm 0.021)$    & $0.012 (\pm 0.003)$    & $\bf{0.791 (\pm 0.015)}$     & $\bf{0.739 (\pm 0.022)}$    & $0.788 (\pm 0.021)$    & $0.012 (\pm 0.003)$ \\
 DSAN    & $0.754 (\pm 0.179)$    & $0.528 (\pm 0.024)$    & $0.792 (\pm 0.021)$    & $0.002 (\pm 0.002)$    & $\bf{0.818 (\pm 0.010)}$    & $0.528 (\pm 0.021)$    & $0.792 (\pm 0.021)$    & $0.002 (\pm 0.002)$ \\
 DIRT    & $0.731 (\pm 0.017)$   & $0.790 (\pm 0.215)$    & $0.790 (\pm 0.019)$    & $0.004 (\pm 0.012)$    & $\bf{0.806 (\pm 0.017)}$    & $\bf{0.793 (\pm 0.216)}$    & $0.790 (\pm 0.019)$    & $0.004 (\pm 0.012)$ \\
 AdvSKM    & $0.752 (\pm 0.084)$    & $0.746 (\pm 0.040)$    & $0.746 (\pm 0.023)$    & $0.004 (\pm 0.006)$    &$\bf{0.752 (\pm 0.024)}$    & $0.746 (\pm 0.027)$    & $0.746 (\pm 0.023)$    & $0.004 (\pm 0.006)$ \\
 HoMM    & $0.759 (\pm 0.128)$  & $0.662 (\pm 0.228)$    & $0.754 (\pm 0.015)$    & $0.007 (\pm 0.003)$    & $\bf{0.759 (\pm 0.012)}$    & $\bf{0.665 (\pm 0.228)}$    & $0.754 (\pm 0.015)$    & $0.007 (\pm 0.003)$ \\
 DDC    & $0.748 (\pm 0.115)$    & $0.454 (\pm 0.021)$    & $0.744 (\pm 0.015)$    & $0.016 (\pm 0.005)$    & $\bf{0.748 (\pm 0.011)}$    & $0.454 (\pm 0.016)$    & $0.744 (\pm 0.015)$    & $0.016 (\pm 0.005)$ \\
 DeepCoral    & $0.701 (\pm 0.009)$    & $0.749 (\pm 0.096)$    & $0.758 (\pm 0.021)$    & $0.006 (\pm 0.008)$    & $\bf{0.764 (\pm 0.009)}$    & $\bf{0.758 (\pm 0.095)}$    & $0.758 (\pm 0.021)$    & $0.006 (\pm 0.008)$ \\
 CMD    & $0.671 (\pm 0.008)$    & $0.770 (\pm 0.098)$    & $0.770 (\pm 0.013)$    & $0.015 (\pm 0.008)$    & $\bf{0.766 (\pm 0.008)}$  & $0.770 (\pm 0.092)$    & $0.770 (\pm 0.013)$    & $0.015 (\pm 0.008)$ \\
\midrule
 Avg.      & $0.737 (\pm 0.080)$    & $0.694 (\pm 0.085)$    & $0.772 (\pm 0.019)$    & $0.008 (\pm 0.006)$    & $\bf{0.779 (\pm 0.014)}$     & $\bf{0.700 (\pm 0.079)}$    & $0.772 (\pm 0.019)$    & $0.008 (\pm 0.006)$ \\
 \bottomrule
\end{tabular}
}
\caption{Mean and standard deviation (after $\pm$) of target classification accuracy on MiniDomainNet, Amazon Reviews and HHAR datasets over three different random initialization of model weights and several domain adaptation tasks.}
\end{table*}

\paragraph*{Domain adaptation: image data}
The DomainNet-2019 dataset~\citep{peng2019moment} consists of six different image domains (Quickdraw: Q, Real: R, Clipart: C, Sketch: S, Infograph: I, and Painting: P).
We follow~\citet{zellinger2021balancing} and rely on the reduced version of DomainNet-2019, referred to as MiniDomainNet, which reduces the number of classes to the top five largest representatives in the training set across all six domains.
To further improve computation time, we rely on a ImageNet~\citep{krizhevsky2012imagenet} pre-trained ResNet-18~\citep{he2016deep} backbone. 
Therefore, we assume that the backbone has learned lower-level filters suitable for the "real" image category, and we only need to adapt to the remaining five domains (Clipart, Sketch,...).
This results in five domain adaptation tasks.

\paragraph*{Domain adaptation: sensory data}
The Heterogeneity Human Activity Recognition~\citep{Stisen:15} dataset investigate \mbox{sensor-,} device- and workload-specific heterogeneities using $36$ smartphones and smartwatches, consisting of $13$ different device models from four manufacturers.
In total, we used all fiv domain adaptation scenarios from~\citet[Tables~25 and~26]{dinu2022aggregation}.

\subsection{Methods}
\label{subsec:methods}

We test our suggested iterated regularization strategy for the four different baseline approaches in Example~\ref{ex:introduction}: KuLSIF~\citep{kanamori2009least}, LR~\citep{bickel2009discriminative}, Exp~\citep{menon2016linking} and SQ~\citep{menon2016linking}. 
Except for KuLSIF, which has a closed form solution, we rely on the CG method implemented in Python Scipy~\citep{virtanen2020scipy} with Polak-Ribiere line search scheme~\citep{hager2006survey}.

In domain adaptation experiments we follow the evaluation protocol of~\citet{dinu2022aggregation}.
That is, we compute an ensemble of different deep neural networks using the importance weighted functional regression approach.
As importance weight, we use the estimated density ratio of the iterated and non-iterated versions of KuLSIF, LR, Exp, and SQ.

To get the full picture of the impact on different deep domain adaptation methods, we compute ensemble model candidates for $11$ domain adaptation algorithms: Adversarial Spectral Kernel Matching (AdvSKM) \citep{Liu2021kernel}, Deep Domain Confusion (DDC) \citep{tzeng2014deep}, Correlation Alignment via Deep Neural Networks (Deep-Coral) \citep{Sun2017correlation}, Central Moment Discrepancy (CMD) \citep{zellinger2017central}, Higher-order Moment Matching (HoMM) \citep{Chen2020moment}, Minimum Discrepancy Estimation for Deep Domain Adaptation (MMDA) \citep{Rahman2020}, Deep Subdomain Adaptation (DSAN) \citep{Zhu2021subdomain}, Domain-Adversarial Neural Networks (DANN) \citep{ganin2016domain}, Conditional Adversarial Domain Adaptation (CDAN) \citep{Long2018conditional}, A DIRT-T Approach to Unsupervised Domain Adaptation (DIRT) \citep{Shu2018dirt} and Convolutional deep Domain Adaptation model for Time-Series data (CoDATS) \citep{Wilson2020adaptation}.

In total, in this benchmark, we trained more than $9000$ deep neural network models and we evaluated the whole aggregation benchmark of~\citet{dinu2022aggregation} for Amazon Review, MiniDomainNet, and HHAR for each of the eight (Example~\ref{ex:introduction} with and without iteration) density ratio estimation methods.

\subsection{Results}
\label{subsec:results}

\paragraph*{Summary}
In general, we observe a consistent improvement of iterated density ratio estimation compared to its related non-iterative methods, in all, synthetic, numerically controlled and benchmark domain adaptation experiments.
That is, iterative regularization is always suggested when enough computational resources are available.
One highlight is, that iteration almost never leads to worse results, i.e., overfitting, compared to the non-iterated version, when the same cross-validation procedure is used to select the regularization parameter as for selecting the number of iterations.

\begin{figure}[!h]
  \centering
  \begin{minipage}[b]{0.45\textwidth}
    \includegraphics[width=\textwidth]{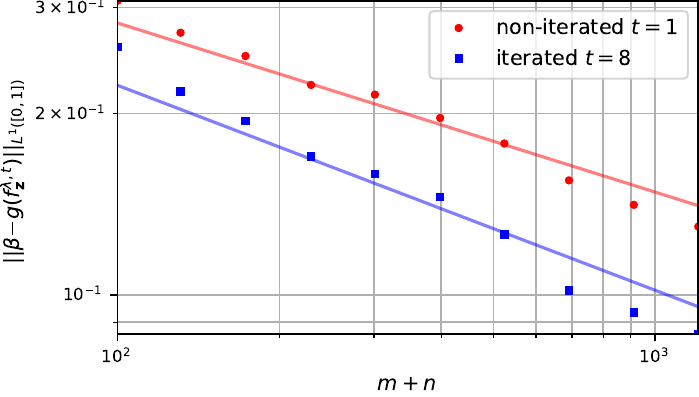}
  \end{minipage}
  \caption{Error $\|\beta-g(f_{\bf z}^{\lambda,t})\|_{L^1([0,1])}$ of estimating $\beta$ by $g(f_{\bf z}^{\lambda,t})$ as a function of sample size $m+n$ on dataset of~\citet{beugnot2021beyond}.
  Slope for iterated estimate ($t=8$, blue) is steeper, as suggested by Theorem~\ref{thm:error_rates_result}.}
  \label{fig:beugnot}
\end{figure}

\paragraph*{Results for highly regular problems}
For the highly regular problems adapted from~\citet{beugnot2021beyond}, we observe, for increasing sample size, a trend as predicted by Theorem~\ref{thm:error_rates_result}, see Figure~\ref{fig:beugnot}.
This result suggests that our approach successfully extends saturation prevention from supervised regression~\citep{beugnot2021beyond,li2022saturation} to density ratio estimation.
In addition, in the case of a Gaussian mixture $P$ and a Gaussian distribution $Q$ in Figure~\ref{fig:saturation} (right), we observe a clear improvement of iteratively regularized KuLSIF compared to KuLSIF.
The significance of the improvement (non-intersecting confidence intervals) becomes more apparent for increasing problem regularity (lower number of components).

\paragraph*{Results for known density ratio}
On all geometrically constructed learning problems, three (KuLSIF, Exp, SQ) out of four approaches are improved by iterated regularization.
In particular, KuLSIF and SQ are always improved and Exp is improved in eight out of ten cases.
The LR approach is not improved by iterated regularization.
However, its performance remained the same.
On average over all four approaches, iterated regularization clearly outperforms non-iterated regularization.

For our ablation study on the breast cancer data, we observe an error of $1.21$ for KuLSIF, $3.47$ for Exp, $8.83$ for LR and $15000$ for SQ.
In the case of SQ, we could not obtain a convergence to lower error, suggesting a closer look with regard to its regularization parameter ranges.
However, iterated KuLSIF achieves an error of $1.19$, iterated Exp $3.17$, iterated LR $3.94$ and iterated SQ $100$.
That is, our iterated regularization approach always improves the results despite the fact that the data is not drawn from Gaussian mixtures.

\paragraph*{Results for domain adaptation}
The detailed results are reported in Tables~\ref{tab:first_table_appendix}--\ref{tab:last_table_appendix}.
In $9$ out of $12$ cases (averaged over all tasks and domain adaptation methods), iterated regularization improves its related non-iteratively regularized vesions. In $3$ cases it has the same performance.
In the same average scenario, it never reduces the performance, i.e., causes overfitting.
In more detail, in the case of text data (AmazonReviews), all four density ratio methods (KuLSIF, LR, Exp, SQ) are improved on average across all domain adaptation methods and datasets, see Tables~\ref{tab:amazon_reviews_first}--\ref{tab:amazon_reviews_last} for detailed results.
In the case of image data (MiniDomainNet), three (KuLSIF, LR, SQ) out of four methods are improved in average over all datsets and tasks, while for Exp, average performance remains the same when the regularization is done iteratively.
For time-series data (HHAR) KuLSIF and Exp, averaged across all domain adaptation methods and datasets, are improved while LR and SQ have, on average, the same performance.

\paragraph*{Further observations}
Although LR outperforms KuLSIF in $50\%$ of the cases on time series data (see Tables~\ref{tab:HHAR_first}--\ref{tab:last_table_appendix}), the iteratively regularized KuLSIF method outperforms LR in $11$ out of $12$ cases.
This additionally underpins the importance of considering iterated regularization strategies in density ratio estimation.
Moreover, in all Tables~\ref{tab:first_table_appendix}--\ref{tab:last_table_appendix} in Appendix~\ref{appendix:experiments}, the maximum average performance is achieved by an iterated method, either iterated KuLSIF, iterated LR, iterated Exp or iterated SQ.

\section{Conclusion and Future Work}
In this work, we showed that common applications of density ratio estimation methods suffer from saturation, an issue preventing learning methods from achieving fast convergence rates in theoretical error bounds.

Our study contained four main parts:
Part (a) extends state-of-the-art error bounds from kernel regression to density ratio estimation, leading to optimal, i.e., improvable, error rates in the case of square loss.
The error rates suffer from saturation.
Part (b) introduces iterated regularization for a large class of density ratio estimation methods, proving that saturation can be prevented.
Part (c) empirically tests the iterated approaches for granting improvements in numerical datasets with known density ratio.
Part (d) evaluates the performance of the iterated regularization in practical problems of ensemble learning in unsupervised domain adaptation.
In all four parts of our study, clear advantages of the proposed approach could be identified.

One important future aspect is parameter choice, especially under the requirement of reliable methods with theoretical error guarantees, since our approach introduces the number of iterations as a new parameter.
We also plan to improve our (time-invariant) time-series domain adaptation using the time-dependent Bregman divergence approach of~\citet{Kim:23}.

A reference implementation of the methods presented
in this paper is available at: \href{https://github.com/lugruber/dre_iter_reg}{https://github.com/lugruber/dre\_iter\_reg}.

\section*{Acknowledgements}
We thank Otmar Scherzer, Sergei Pereverzyev, Stefan Kindermann, Marius-Constantin Dinu and four anonymous reviewers for helpful comments.
The ELLIS Unit Linz, the LIT AI Lab, the Institute for Machine Learning, are supported by the Federal State Upper Austria. We thank the projects Medical Cognitive Computing Center (MC3), INCONTROL-RL (FFG-881064), PRIMAL (FFG-873979), S3AI (FFG-872172), DL for GranularFlow (FFG-871302), EPILEPSIA (FFG-892171), AIRI FG 9-N (FWF-36284, FWF-36235), AI4GreenHeatingGrids (FFG- 899943), INTEGRATE (FFG-892418), ELISE (H2020-ICT-2019-3 ID: 951847), Stars4Waters (HORIZON-CL6-2021-CLIMATE-01-01). We thank Audi.JKU Deep Learning Center, TGW LOGISTICS GROUP GMBH, Silicon Austria Labs (SAL), FILL Gesellschaft mbH, Anyline GmbH, Google, ZF Friedrichshafen AG, Robert Bosch GmbH, UCB Biopharma SRL, Merck Healthcare KGaA, Verbund AG, GLS (Univ. Waterloo), Software Competence Center Hagenberg GmbH, Borealis AG, T\"{U}V Austria, Frauscher Sensonic, TRUMPF and the NVIDIA Corporation.

\section*{Impact Statement}

This paper presents work whose goal is to advance the field of Machine Learning. There are many potential societal consequences of our work, none which we feel must be specifically highlighted here.

\bibliography{arxiv_v2}

\begin{thebibliography}{102}
\providecommand{\natexlab}[1]{#1}
\providecommand{\url}[1]{\texttt{#1}}
\expandafter\ifx\csname urlstyle\endcsname\relax
  \providecommand{\doi}[1]{doi: #1}\else
  \providecommand{\doi}{doi: \begingroup \urlstyle{rm}\Url}\fi

\bibitem[Bach(2010)]{bach2010self}
Bach, F.
\newblock {Self-concordant analysis for logistic regression}.
\newblock \emph{Electronic Journal of Statistics}, 4\penalty0 (none):\penalty0
  384 -- 414, 2010.

\bibitem[Bauer et~al.(2007)Bauer, Pereverzev, and
  Rosasco]{bauer2007regularization}
Bauer, F., Pereverzev, S., and Rosasco, L.
\newblock On regularization algorithms in learning theory.
\newblock \emph{Journal of Complexity}, 23\penalty0 (1):\penalty0 52--72, 2007.

\bibitem[Beugnot et~al.(2021)Beugnot, Mairal, and Rudi]{beugnot2021beyond}
Beugnot, G., Mairal, J., and Rudi, A.
\newblock Beyond tikhonov: faster learning with self-concordant losses, via
  iterative regularization.
\newblock \emph{Advances in Neural Information Processing Systems},
  34:\penalty0 28196--28207, 2021.

\bibitem[Bickel et~al.(2009)Bickel, Br{\"u}ckner, and
  Scheffer]{bickel2009discriminative}
Bickel, S., Br{\"u}ckner, M., and Scheffer, T.
\newblock Discriminative learning under covariate shift.
\newblock \emph{Journal of Machine Learning Research}, 10\penalty0 (9), 2009.

\bibitem[Blanchard \& M{\"u}cke(2018)Blanchard and
  M{\"u}cke]{blanchard2018optimal}
Blanchard, G. and M{\"u}cke, N.
\newblock Optimal rates for regularization of statistical inverse learning
  problems.
\newblock \emph{Foundations of Computational Mathematics}, 18:\penalty0
  971--1013, 2018.

\bibitem[Blitzer et~al.(2006)Blitzer, McDonald, and Pereira]{blitzer2006domain}
Blitzer, J., McDonald, R., and Pereira, F.
\newblock Domain adaptation with structural correspondence learning.
\newblock In \emph{Proceedings of the 2006 conference on empirical methods in
  natural language processing}, pp.\  120--128, 2006.

\bibitem[Bregman(1967)]{bregman1967relaxation}
Bregman, L.
\newblock The relaxation method of finding the common point of convex sets and
  its application to the solution of problems in convex programming.
\newblock \emph{USSR Computational Mathematics and Mathematical Physics},
  7\penalty0 (3):\penalty0 200--217, 1967.

\bibitem[Breiman(1996)]{Breiman:96}
Breiman, L.
\newblock Bias, variance and arcing classifiers.
\newblock Technical report, Statistics Department, University of California,
  1996.

\bibitem[Buja et~al.(2005)Buja, Stuetzle, and Shen]{buja2005loss}
Buja, A., Stuetzle, W., and Shen, Y.
\newblock Loss functions for binary class probability estimation and
  classification: Structure and applications.
\newblock \emph{Working draft, November}, 3:\penalty0 13, 2005.

\bibitem[Caponnetto \& De~Vito(2007)Caponnetto and
  De~Vito]{caponnetto2007optimal}
Caponnetto, A. and De~Vito, E.
\newblock Optimal rates for the regularized least-squares algorithm.
\newblock \emph{Foundations of Computational Mathematics}, 7\penalty0
  (3):\penalty0 331--368, 2007.

\bibitem[Caponnetto et~al.(2005)Caponnetto, Rosasco, De~Vito, and
  Verri]{caponnetto2005empirical}
Caponnetto, A., Rosasco, L., De~Vito, E., and Verri, A.
\newblock Empirical effective dimension and optimal rates for regularized least
  squares algorithm.
\newblock Technical report, Computer Science and Artificial Intelligence
  Laboratory (CSAIL), MIT, 2005.

\bibitem[Chan-Renous-Legoubin \& Royer(2022)Chan-Renous-Legoubin and
  Royer]{chan2022nonlinear}
Chan-Renous-Legoubin, R. and Royer, C.~W.
\newblock A nonlinear conjugate gradient method with complexity guarantees and
  its application to nonconvex regression.
\newblock \emph{EURO Journal on Computational Optimization}, 10:\penalty0
  100044, 2022.

\bibitem[Chen et~al.(2020)Chen, Fu, Chen, Jin, Cheng, Jin, and
  Hua]{Chen2020moment}
Chen, C., Fu, Z., Chen, Z., Jin, S., Cheng, Z., Jin, X., and Hua, X.-S.
\newblock Homm: Higher-order moment matching for unsupervised domain
  adaptation.
\newblock \emph{Association for the Advancement of Artificial Intelligence
  (AAAI)}, 2020.

\bibitem[Choi et~al.(2021)Choi, Liao, and Ermon]{Choi:21}
Choi, K., Liao, M., and Ermon, S.
\newblock Featurized density ratio estimation.
\newblock In \emph{Proceedings of the Thirty-Seventh Conference on Uncertainty
  in Artificial Intelligence}, pp.\  172--182, 2021.

\bibitem[Ciliberto et~al.(2020)Ciliberto, Rosasco, and
  Rudi]{ciliberto2020general}
Ciliberto, C., Rosasco, L., and Rudi, A.
\newblock A general framework for consistent structured prediction with
  implicit loss embeddings.
\newblock \emph{The Journal of Machine Learning Research}, 21\penalty0
  (1):\penalty0 3852--3918, 2020.

\bibitem[Cohen(1972)]{cohen1972rate}
Cohen, A.~I.
\newblock Rate of convergence of several conjugate gradient algorithms.
\newblock \emph{SIAM Journal on Numerical Analysis}, 9\penalty0 (2):\penalty0
  248--259, 1972.

\bibitem[De~Vito et~al.(2010)De~Vito, Pereverzyev, and Rosasco]{de2010adaptive}
De~Vito, E., Pereverzyev, S.~V., and Rosasco, L.
\newblock Adaptive kernel methods using the balancing principle.
\newblock \emph{Foundations of Computational Mathematics}, 10\penalty0
  (4):\penalty0 455--479, 2010.

\bibitem[Dinu et~al.(2023)Dinu, Holzleitner, Beck, Nguyen, Huber, Eghbal-zadeh,
  Moser, Pereverzyev, Hochreiter, and Zellinger]{dinu2022aggregation}
Dinu, M.-C., Holzleitner, M., Beck, M., Nguyen, D.~H., Huber, A., Eghbal-zadeh,
  H., Moser, B.~A., Pereverzyev, S.~V., Hochreiter, S., and Zellinger, W.
\newblock Addressing parameter choice issues in unsupervised domain adaptation
  by aggregation.
\newblock \emph{International Conference on Learning Representations}, 2023.

\bibitem[Engl et~al.(1996{\natexlab{a}})Engl, Hanke, and Neubauer]{Engl:96}
Engl, H., Hanke, M., and Neubauer, A.
\newblock \emph{Regularization of Inverse Problems}, volume 375.
\newblock Springer Science \& Business Media, 1996{\natexlab{a}}.

\bibitem[Engl et~al.(1996{\natexlab{b}})Engl, Hanke, and
  Neubauer]{engl1996regularization}
Engl, H.~W., Hanke, M., and Neubauer, A.
\newblock \emph{Regularization of inverse problems}, volume 375.
\newblock Springer Science \& Business Media, 1996{\natexlab{b}}.

\bibitem[Ganin et~al.(2016)Ganin, Ustinova, Ajakan, Germain, Larochelle,
  Laviolette, Marchand, and Lempitsky]{ganin2016domain}
Ganin, Y., Ustinova, E., Ajakan, H., Germain, P., Larochelle, H., Laviolette,
  F., Marchand, M., and Lempitsky, V.
\newblock Domain-adversarial training of neural networks.
\newblock \emph{The journal of machine learning research}, 17\penalty0
  (1):\penalty0 2096--2030, 2016.

\bibitem[Gidel et~al.(2019)Gidel, Bach, and Lacoste-Julien]{Gidel:19}
Gidel, G., Bach, F., and Lacoste-Julien, S.
\newblock Implicit regularization of discrete gradient dynamics in linear
  neural networks.
\newblock In Wallach, H., Larochelle, H., Beygelzimer, A., d\textquotesingle
  Alch\'{e}-Buc, F., Fox, E., and Garnett, R. (eds.), \emph{Advances in Neural
  Information Processing Systems}, volume~32. Curran Associates, Inc., 2019.
\newblock URL
  \url{https://proceedings.neurips.cc/paper_files/paper/2019/file/f39ae9ff3a81f499230c4126e01f421b-Paper.pdf}.

\bibitem[Gizewski et~al.(2022)Gizewski, Mayer, Moser, Nguyen, Pereverzyev~Jr,
  Pereverzyev, Shepeleva, and Zellinger]{gizewski2022regularization}
Gizewski, E.~R., Mayer, L., Moser, B.~A., Nguyen, D.~H., Pereverzyev~Jr, S.,
  Pereverzyev, S.~V., Shepeleva, N., and Zellinger, W.
\newblock On a regularization of unsupervised domain adaptation in {RKHS}.
\newblock \emph{Applied and Computational Harmonic Analysis}, 57:\penalty0
  201--227, 2022.

\bibitem[Golub \& Van~Loan(2013)Golub and Van~Loan]{golub2013matrix}
Golub, G.~H. and Van~Loan, C.~F.
\newblock \emph{Matrix computations}.
\newblock JHU press, 2013.

\bibitem[Hager \& Zhang(2006)Hager and Zhang]{hager2006survey}
Hager, W.~W. and Zhang, H.
\newblock A survey of nonlinear conjugate gradient methods.
\newblock \emph{Pacific journal of Optimization}, 2\penalty0 (1):\penalty0
  35--58, 2006.

\bibitem[Hanke \& Groetsch(1998)Hanke and Groetsch]{hanke1998nonstationary}
Hanke, M. and Groetsch, C.~W.
\newblock Nonstationary iterated tikhonov regularization.
\newblock \emph{Journal of Optimization Theory and Applications}, 98:\penalty0
  37--53, 1998.

\bibitem[He et~al.(2016)He, Zhang, Ren, and Sun]{he2016deep}
He, K., Zhang, X., Ren, S., and Sun, J.
\newblock Deep residual learning for image recognition.
\newblock In \emph{Proceedings of the IEEE conference on computer vision and
  pattern recognition}, pp.\  770--778, 2016.

\bibitem[Hido et~al.(2011)Hido, Tsuboi, Kashima, Sugiyama, and
  Kanamori]{hido2011statistical}
Hido, S., Tsuboi, Y., Kashima, H., Sugiyama, M., and Kanamori, T.
\newblock Statistical outlier detection using direct density ratio estimation.
\newblock \emph{Knowledge and Information Systems}, 26:\penalty0 309--336,
  2011.

\bibitem[Jacot et~al.(2018)Jacot, Gabriel, and Hongler]{Jacot:18}
Jacot, A., Gabriel, F., and Hongler, C.
\newblock Neural tangent kernel: Convergence and generalization in neural
  networks.
\newblock In Bengio, S., Wallach, H., Larochelle, H., Grauman, K.,
  Cesa-Bianchi, N., and Garnett, R. (eds.), \emph{Advances in Neural
  Information Processing Systems}, volume~31. Curran Associates, Inc., 2018.
\newblock URL
  \url{https://proceedings.neurips.cc/paper_files/paper/2018/file/5a4be1fa34e62bb8a6ec6b91d2462f5a-Paper.pdf}.

\bibitem[Kanamori et~al.(2009)Kanamori, Hido, and Sugiyama]{kanamori2009least}
Kanamori, T., Hido, S., and Sugiyama, M.
\newblock A least-squares approach to direct importance estimation.
\newblock \emph{The Journal of Machine Learning Research}, 10:\penalty0
  1391--1445, 2009.

\bibitem[Kanamori et~al.(2011)Kanamori, Suzuki, and Sugiyama]{kanamori2011f}
Kanamori, T., Suzuki, T., and Sugiyama, M.
\newblock $ f $-divergence estimation and two-sample homogeneity test under
  semiparametric density-ratio models.
\newblock \emph{IEEE Transactions on Information Theory}, 58\penalty0
  (2):\penalty0 708--720, 2011.

\bibitem[Kanamori et~al.(2012{\natexlab{a}})Kanamori, Suzuki, and
  Sugiyama]{Kanamori:12}
Kanamori, T., Suzuki, T., and Sugiyama, M.
\newblock Statistical analysis of kernel-based least-squares density-ratio
  estimation.
\newblock \emph{Machine Learning}, 86:\penalty0 335--367, 2012{\natexlab{a}}.

\bibitem[Kanamori et~al.(2012{\natexlab{b}})Kanamori, Suzuki, and
  Sugiyama]{kanamori2012statistical}
Kanamori, T., Suzuki, T., and Sugiyama, M.
\newblock Statistical analysis of kernel-based least-squares density-ratio
  estimation.
\newblock \emph{Machine Learning}, 86\penalty0 (3):\penalty0 335--367,
  2012{\natexlab{b}}.

\bibitem[Kato \& Teshima(2021)Kato and Teshima]{kato2021non}
Kato, M. and Teshima, T.
\newblock Non-negative bregman divergence minimization for deep direct density
  ratio estimation.
\newblock In \emph{International Conference on Machine Learning}, pp.\
  5320--5333. PMLR, 2021.

\bibitem[Kato et~al.(2019)Kato, Teshima, and Honda]{kato2019learning}
Kato, M., Teshima, T., and Honda, J.
\newblock Learning from positive and unlabeled data with a selection bias.
\newblock In \emph{International Conference on Learning Representations}, 2019.

\bibitem[Keziou \& Leoni-Aubin(2005)Keziou and Leoni-Aubin]{keziou2005test}
Keziou, A. and Leoni-Aubin, S.
\newblock Test of homogeneity in semiparametric two-sample density ratio
  models.
\newblock \emph{Comptes Rendus Math{\'e}matique}, 340\penalty0 (12):\penalty0
  905--910, 2005.

\bibitem[Kim et~al.(2023)Kim, Kim, Kwon, Kang, and Moon]{Kim:23}
Kim, D., Kim, Y., Kwon, S.~J., Kang, W., and Moon, I.-C.
\newblock Refining generative process with discriminator guidance in
  score-based diffusion models.
\newblock In Krause, A., Brunskill, E., Cho, K., Engelhardt, B., Sabato, S.,
  and Scarlett, J. (eds.), \emph{Proceedings of the 40th International
  Conference on Machine Learning}, volume 202 of \emph{Proceedings of Machine
  Learning Research}, pp.\  16567--16598. PMLR, 23--29 Jul 2023.
\newblock URL \url{https://proceedings.mlr.press/v202/kim23i.html}.

\bibitem[Kim et~al.(2024)Kim, Na, Park, Jang, Kim, Kang, and chul Moon]{Kim:24}
Kim, Y., Na, B., Park, M., Jang, J., Kim, D., Kang, W., and chul Moon, I.
\newblock Training unbiased diffusion models from biased dataset.
\newblock In \emph{The Twelfth International Conference on Learning
  Representations}, 2024.
\newblock URL \url{https://openreview.net/forum?id=39cPKijBed}.

\bibitem[Kindermann et~al.(2023)Kindermann, Resmerita, and
  Wolf]{kindermann2023multiscale}
Kindermann, S., Resmerita, E., and Wolf, T.
\newblock Multiscale hierarchical decomposition methods for ill-posed problems.
\newblock \emph{Inverse Problems}, 39\penalty0 (12):\penalty0 125013, 2023.

\bibitem[King \& Chillingworth(1979)King and
  Chillingworth]{thomas1979approximation}
King, T.~J. and Chillingworth, D.
\newblock Approximation of generalized inverses by iterated regularization.
\newblock \emph{Numerical Functional Analysis and Optimization}, 1\penalty0
  (5):\penalty0 499--513, 1979.

\bibitem[Kiryo et~al.(2017)Kiryo, Niu, du~Plessis, and Sugiyama]{Kiryo:17}
Kiryo, R., Niu, G., du~Plessis, M.~C., and Sugiyama, M.
\newblock Positive-unlabeled learning with non-negative risk estimator.
\newblock In Guyon, I., Luxburg, U.~V., Bengio, S., Wallach, H., Fergus, R.,
  Vishwanathan, S., and Garnett, R. (eds.), \emph{Advances in Neural
  Information Processing Systems}, volume~30. Curran Associates, Inc., 2017.
\newblock URL
  \url{https://proceedings.neurips.cc/paper_files/paper/2017/file/7cce53cf90577442771720a370c3c723-Paper.pdf}.

\bibitem[Krizhevsky et~al.(2012)Krizhevsky, Sutskever, and
  Hinton]{krizhevsky2012imagenet}
Krizhevsky, A., Sutskever, I., and Hinton, G.~E.
\newblock Imagenet classification with deep convolutional neural networks.
\newblock In \emph{Advances in Neural Information Processing Systems}, pp.\
  1097--1105, 2012.

\bibitem[Li et~al.(2022)Li, Zhang, and Lin]{li2022saturation}
Li, Y., Zhang, H., and Lin, Q.
\newblock On the saturation effect of kernel ridge regression.
\newblock In \emph{The Eleventh International Conference on Learning
  Representations}, 2022.

\bibitem[Lin et~al.(2020)Lin, Rudi, Rosasco, and Cevher]{lin2020optimal}
Lin, J., Rudi, A., Rosasco, L., and Cevher, V.
\newblock Optimal rates for spectral algorithms with least-squares regression
  over hilbert spaces.
\newblock \emph{Applied and Computational Harmonic Analysis}, 48\penalty0
  (3):\penalty0 868--890, 2020.

\bibitem[Liu \& Xue(2021)Liu and Xue]{Liu2021kernel}
Liu, Q. and Xue, H.
\newblock Adversarial spectral kernel matching for unsupervised time series
  domain adaptation.
\newblock \emph{Proceedings of the International Joint Conference on Artificial
  Intelligence (IJCAI)}, 30, 2021.

\bibitem[Long et~al.(2018)Long, Cao, Wang, and Jordan]{Long2018conditional}
Long, M., Cao, Z., Wang, J., and Jordan, M.~I.
\newblock Conditional adversarial domain adaptation.
\newblock \emph{Advances in Neural Information Processing Systems (NeurIPS)},
  31, 2018.

\bibitem[Lu \& Pereverzev(2013)Lu and Pereverzev]{lu2013regularization}
Lu, S. and Pereverzev, S.~V.
\newblock Regularization theory for ill-posed problems.
\newblock In \emph{Regularization Theory for Ill-posed Problems}. de Gruyter,
  2013.

\bibitem[Marteau-Ferey et~al.(2019{\natexlab{a}})Marteau-Ferey, Bach, and
  Rudi]{marteau2019globally}
Marteau-Ferey, U., Bach, F., and Rudi, A.
\newblock Globally convergent newton methods for ill-conditioned generalized
  self-concordant losses.
\newblock \emph{Advances in Neural Information Processing Systems}, 32,
  2019{\natexlab{a}}.

\bibitem[Marteau-Ferey et~al.(2019{\natexlab{b}})Marteau-Ferey, Ostrovskii,
  Bach, and Rudi]{marteau2019beyond}
Marteau-Ferey, U., Ostrovskii, D., Bach, F., and Rudi, A.
\newblock Beyond least-squares: Fast rates for regularized empirical risk
  minimization through self-concordance.
\newblock In \emph{Conference on Learning Theory}, pp.\  2294--2340. PMLR,
  2019{\natexlab{b}}.

\bibitem[Menon \& Ong(2016)Menon and Ong]{menon2016linking}
Menon, A. and Ong, C.~S.
\newblock Linking losses for density ratio and class-probability estimation.
\newblock In \emph{International Conference on Machine Learning}, pp.\
  304--313. PMLR, 2016.

\bibitem[Meunier et~al.(2024)Meunier, Shen, Mollenhauer, Gretton, and
  Li]{meunier2024optimal}
Meunier, D., Shen, Z., Mollenhauer, M., Gretton, A., and Li, Z.
\newblock Optimal rates for vector-valued spectral regularization learning
  algorithms.
\newblock \emph{arXiv preprint arXiv:2405.14778}, 2024.

\bibitem[Modin et~al.(2019)Modin, Nachman, Rondi, et~al.]{modin2019multiscale}
Modin, K., Nachman, A., Rondi, L., et~al.
\newblock A multiscale theory for image registration and nonlinear inverse
  problems.
\newblock \emph{Advances in Mathematics}, 346:\penalty0 1009--1066, 2019.

\bibitem[Mohamed \& Lakshminarayanan(2016)Mohamed and
  Lakshminarayanan]{mohamed2016learning}
Mohamed, S. and Lakshminarayanan, B.
\newblock Learning in implicit generative models.
\newblock \emph{arXiv preprint arXiv:1610.03483}, 2016.

\bibitem[Neubauer(1997)]{neubauer1997converse}
Neubauer, A.
\newblock On converse and saturation results for tikhonov regularization of
  linear ill-posed problems.
\newblock \emph{SIAM journal on numerical analysis}, 34\penalty0 (2):\penalty0
  517--527, 1997.

\bibitem[Neumaier et~al.(2022)Neumaier, Kimiaei, and
  Azmi]{neumaier2022globally}
Neumaier, A., Kimiaei, M., and Azmi, B.
\newblock Globally linearly convergent nonlinear conjugate gradients without
  wolfe line search.
\newblock \emph{preprint}, 2022.

\bibitem[Nguyen et~al.(2023)Nguyen, Zellinger, and
  Pereverzyev]{nguyen2023regularized}
Nguyen, D., Zellinger, W., and Pereverzyev, S.
\newblock On regularized {R}adon-{N}ikodym differentiation.
\newblock \emph{RICAM-Report 2023-13}, 2023.

\bibitem[Nguyen et~al.(2007)Nguyen, Wainwright, and
  Jordan]{nguyen2007estimating}
Nguyen, X., Wainwright, M.~J., and Jordan, M.
\newblock Estimating divergence functionals and the likelihood ratio by
  penalized convex risk minimization.
\newblock \emph{Advances in Neural Information Processing Systems}, 20, 2007.

\bibitem[Nguyen et~al.(2010{\natexlab{a}})Nguyen, Wainwright, and
  Joradan]{Nguyen:10}
Nguyen, X., Wainwright, M.~J., and Joradan, M.~I.
\newblock Estimating divergence functionals and the likelihood ratio by convex
  risk minimization.
\newblock \emph{IEEE Transactions on Information Theory}, 56\penalty0
  (11):\penalty0 5847--5861, 2010{\natexlab{a}}.
\newblock URL \url{http://jmlr.org/papers/v10/bickel09a.html}.

\bibitem[Nguyen et~al.(2010{\natexlab{b}})Nguyen, Wainwright, and
  Jordan]{nguyen2010estimating}
Nguyen, X., Wainwright, M.~J., and Jordan, M.~I.
\newblock Estimating divergence functionals and the likelihood ratio by convex
  risk minimization.
\newblock \emph{IEEE Transactions on Information Theory}, 56\penalty0
  (11):\penalty0 5847--5861, 2010{\natexlab{b}}.

\bibitem[Olver et~al.(2010)Olver, Lozier, Boisvert, and Clark]{Olver:10}
Olver, F. W.~J., Lozier, D.~W., Boisvert, R.~F., and Clark, C.~W.
\newblock \emph{{NIST} handbook of mathematical functions}.
\newblock Cambridge University Press, 1 pap/cdr edition, 2010.
\newblock ISBN 9780521192255.

\bibitem[Osher et~al.(2005)Osher, Burger, Goldfarb, Xu, and
  Yin]{osher2005iterative}
Osher, S., Burger, M., Goldfarb, D., Xu, J., and Yin, W.
\newblock An iterative regularization method for total variation-based image
  restoration.
\newblock \emph{Multiscale Modeling \& Simulation}, 4\penalty0 (2):\penalty0
  460--489, 2005.

\bibitem[Ostrovskii \& Bach(2021)Ostrovskii and Bach]{ostrivskii2021finite}
Ostrovskii, D.~M. and Bach, F.
\newblock {Finite-sample analysis of $M$-estimators using self-concordance}.
\newblock \emph{Electronic Journal of Statistics}, 15\penalty0 (1):\penalty0
  326 -- 391, 2021.

\bibitem[Peng et~al.(2019)Peng, Bai, Xia, Huang, Saenko, and
  Wang]{peng2019moment}
Peng, X., Bai, Q., Xia, X., Huang, Z., Saenko, K., and Wang, B.
\newblock Moment matching for multi-source domain adaptation.
\newblock In \emph{Proceedings of the IEEE International Conference on Computer
  Vision}, pp.\  1406--1415, 2019.

\bibitem[Platt et~al.(1999)]{platt1999probabilistic}
Platt, J. et~al.
\newblock Probabilistic outputs for support vector machines and comparisons to
  regularized likelihood methods.
\newblock \emph{Advances in large margin classifiers}, 10\penalty0
  (3):\penalty0 61--74, 1999.

\bibitem[Que \& Belkin(2013)Que and Belkin]{que2013inverse}
Que, Q. and Belkin, M.
\newblock Inverse density as an inverse problem: The fredholm equation
  approach.
\newblock \emph{Advances in neural information processing systems}, 26, 2013.

\bibitem[Ragab et~al.(2023)Ragab, Eldele, Tan, Foo, Chen, Wu, Kwoh, and
  Li]{ragab2023adatime}
Ragab, M., Eldele, E., Tan, W.~L., Foo, C.-S., Chen, Z., Wu, M., Kwoh, C.-K.,
  and Li, X.
\newblock Adatime: A benchmarking suite for domain adaptation on time series
  data.
\newblock \emph{ACM Transactions on Knowledge Discovery from Data}, 17\penalty0
  (8):\penalty0 1--18, 2023.

\bibitem[Rahman et~al.(2020)Rahman, Fookes, Baktashmotlagh, and
  Sridharan]{Rahman2020}
Rahman, M.~M., Fookes, C., Baktashmotlagh, M., and Sridharan, S.
\newblock On minimum discrepancy estimation for deep domain adaptation.
\newblock \emph{Domain Adaptation for Visual Understanding}, 2020.

\bibitem[Reid \& Williamson(2011)Reid and Williamson]{reid2011information}
Reid, M. and Williamson, R.
\newblock Information, divergence and risk for binary experiments.
\newblock \emph{Journal of Machine Learning Research}, 2011.

\bibitem[Reid \& Williamson(2010)Reid and Williamson]{reid2010composite}
Reid, M.~D. and Williamson, R.~C.
\newblock Composite binary losses.
\newblock \emph{The Journal of Machine Learning Research}, 11:\penalty0
  2387--2422, 2010.

\bibitem[Resmerita \& Scherzer(2006)Resmerita and Scherzer]{resmerita2006error}
Resmerita, E. and Scherzer, O.
\newblock Error estimates for non-quadratic regularization and the relation to
  enhancement.
\newblock \emph{Inverse Problems}, 22\penalty0 (3):\penalty0 801, 2006.

\bibitem[Rhodes et~al.(2020)Rhodes, Xu, and Gutmann]{rhodes2020telescoping}
Rhodes, B., Xu, K., and Gutmann, M.~U.
\newblock Telescoping density-ratio estimation.
\newblock \emph{Advances in neural information processing systems},
  33:\penalty0 4905--4916, 2020.

\bibitem[Rockafellar(1976)]{rockafellar1976monotone}
Rockafellar, R.~T.
\newblock Monotone operators and the proximal point algorithm.
\newblock \emph{SIAM journal on control and optimization}, 14\penalty0
  (5):\penalty0 877--898, 1976.

\bibitem[Rudi \& Rosasco(2017)Rudi and Rosasco]{rudi2017generalization}
Rudi, A. and Rosasco, L.
\newblock Generalization properties of learning with random features.
\newblock \emph{Advances in Neural Information Processing Systems}, 30, 2017.

\bibitem[Scherzer(1993)]{scherzer1993convergence}
Scherzer, O.
\newblock Convergence rates of iterated tikhonov regularized solutions of
  nonlinear iii—posed problems.
\newblock \emph{Numerische Mathematik}, 66\penalty0 (1):\penalty0 259--279,
  1993.

\bibitem[Scherzer \& Groetsch(2001)Scherzer and Groetsch]{scherzer2001inverse}
Scherzer, O. and Groetsch, C.
\newblock Inverse scale space theory for inverse problems.
\newblock In \emph{International Conference on Scale-Space Theories in Computer
  Vision}, pp.\  317--325. Springer, 2001.

\bibitem[Scherzer et~al.(2009)Scherzer, Grasmair, Grossauer, Haltmeier, and
  Lenzen]{scherzer2009variational}
Scherzer, O., Grasmair, M., Grossauer, H., Haltmeier, M., and Lenzen, F.
\newblock \emph{Variational methods in imaging}, volume 167.
\newblock Springer, 2009.

\bibitem[Sch{\"o}lkopf \& Smola(2002)Sch{\"o}lkopf and Smola]{Schoelkopf:02}
Sch{\"o}lkopf, B. and Smola, A.~J.
\newblock Learning with kernels.
\newblock In \emph{Learning with kernels}. MIT Press, 2002.

\bibitem[Sch{\"o}lkopf et~al.(2001)Sch{\"o}lkopf, Herbrich, and
  Smola]{scholkopf2001generalized}
Sch{\"o}lkopf, B., Herbrich, R., and Smola, A.~J.
\newblock A generalized representer theorem.
\newblock In \emph{International Conference on Computational Learning Theory},
  pp.\  416--426. Springer, 2001.

\bibitem[Schuster et~al.(2020)Schuster, Mollenhauer, Klus, and
  Muandet]{schuster2020kernel}
Schuster, I., Mollenhauer, M., Klus, S., and Muandet, K.
\newblock Kernel conditional density operators.
\newblock In \emph{International Conference on Artificial Intelligence and
  Statistics}, pp.\  993--1004. PMLR, 2020.

\bibitem[Shimodaira(2000)]{shimodaira2000improving}
Shimodaira, H.
\newblock Improving predictive inference under covariate shift by weighting the
  log-likelihood function.
\newblock \emph{Journal of Statistical Planning and Inference}, 90\penalty0
  (2):\penalty0 227--244, 2000.

\bibitem[Shu et~al.(2018)Shu, Bui, Narui, and Ermon]{Shu2018dirt}
Shu, R., Bui, H., Narui, H., and Ermon, S.
\newblock A dirt-t approach to unsupervised domain adaptation.
\newblock \emph{International Conference on Learning Representations (ICLR)},
  2018.

\bibitem[Smola et~al.(2009)Smola, Song, and Teo]{smola2009relative}
Smola, A., Song, L., and Teo, C.~H.
\newblock Relative novelty detection.
\newblock In \emph{Artificial Intelligence and Statistics}, pp.\  536--543.
  PMLR, 2009.

\bibitem[Srivastava et~al.(2023)Srivastava, Han, Xu, Rhodes, and
  Gutmann]{Srivastava:23}
Srivastava, A., Han, S., Xu, K., Rhodes, B., and Gutmann, M.~U.
\newblock Estimating the density ratio between distributions with high
  discrepancy using multinomial logistic regression.
\newblock \emph{Transactions on Machine Learning Research}, 2023.
\newblock ISSN 2835-8856.
\newblock URL \url{https://openreview.net/forum?id=jM8nzUzBWr}.

\bibitem[Stisen et~al.(2015)Stisen, Blunck, Bhattacharya, Prentow,
  Kj\ae{}rgaard, Dey, Sonne, and Jensen]{Stisen:15}
Stisen, A., Blunck, H., Bhattacharya, S., Prentow, T.~S., Kj\ae{}rgaard, M.~B.,
  Dey, A., Sonne, T., and Jensen, M.~M.
\newblock Smart devices are different: Assessing and mitigatingmobile sensing
  heterogeneities for activity recognition.
\newblock In \emph{Proceedings of the 13th ACM Conference on Embedded Networked
  Sensor Systems}, pp.\  127–140, 2015.

\bibitem[Street et~al.(1993)Street, Wolberg, and Mangasarian]{Street:93}
Street, W.~N., Wolberg, W.~H., and Mangasarian, O.~L.
\newblock Nuclear feature extraction for breast tumor diagnosis.
\newblock In \emph{Electronic imaging}, 1993.
\newblock URL \url{https://api.semanticscholar.org/CorpusID:14922543}.

\bibitem[Sugiyama et~al.(2007)Sugiyama, Krauledat, and
  M{\"u}ller]{sugiyama2007covariate}
Sugiyama, M., Krauledat, M., and M{\"u}ller, K.-R.
\newblock Covariate shift adaptation by importance weighted cross validation.
\newblock \emph{Journal of Machine Learning Research}, 8\penalty0 (5), 2007.

\bibitem[Sugiyama et~al.(2012{\natexlab{a}})Sugiyama, Suzuki, and
  Kanamori]{sugiyama2012density}
Sugiyama, M., Suzuki, T., and Kanamori, T.
\newblock \emph{Density ratio estimation in machine learning}.
\newblock Cambridge University Press, 2012{\natexlab{a}}.

\bibitem[Sugiyama et~al.(2012{\natexlab{b}})Sugiyama, Suzuki, and
  Kanamori]{sugiyama2012densitybregman}
Sugiyama, M., Suzuki, T., and Kanamori, T.
\newblock Density-ratio matching under the bregman divergence: a unified
  framework of density-ratio estimation.
\newblock \emph{Annals of the Institute of Statistical Mathematics},
  64\penalty0 (5):\penalty0 1009--1044, 2012{\natexlab{b}}.

\bibitem[Sun et~al.(2017)Sun, Feng, and Saenko]{Sun2017correlation}
Sun, B., Feng, J., and Saenko, K.
\newblock Correlation alignment for unsupervised domain adaptation.
\newblock \emph{Domain Adaptation in Computer Vision Applications}, pp.\
  153--171, 2017.

\bibitem[Tadmor et~al.(2004)Tadmor, Nezzar, and Vese]{tadmor2004multiscale}
Tadmor, E., Nezzar, S., and Vese, L.
\newblock A multiscale image representation using hierarchical (bv, l 2)
  decompositions.
\newblock \emph{Multiscale Modeling \& Simulation}, 2\penalty0 (4):\penalty0
  554--579, 2004.

\bibitem[Tzeng et~al.(2014)Tzeng, Hoffman, Zhang, Saenko, and
  Darrell]{tzeng2014deep}
Tzeng, E., Hoffman, J., Zhang, N., Saenko, K., and Darrell, T.
\newblock Deep domain confusion: Maximizing for domain invariance.
\newblock \emph{arXiv preprint arXiv:1412.3474}, 2014.

\bibitem[Van~der Vaart(2000)]{van2000asymptotic}
Van~der Vaart, A.~W.
\newblock \emph{Asymptotic statistics}, volume~3.
\newblock Cambridge university press, 2000.

\bibitem[Vapnik(2013)]{vapnik2013nature}
Vapnik, V.~N.
\newblock \emph{The nature of statistical learning theory}.
\newblock Springer science \& business media, 2013.

\bibitem[Virtanen et~al.(2020)Virtanen, Gommers, Oliphant, Haberland, Reddy,
  Cournapeau, Burovski, Peterson, Weckesser, Bright, et~al.]{virtanen2020scipy}
Virtanen, P., Gommers, R., Oliphant, T.~E., Haberland, M., Reddy, T.,
  Cournapeau, D., Burovski, E., Peterson, P., Weckesser, W., Bright, J., et~al.
\newblock Scipy 1.0: fundamental algorithms for scientific computing in python.
\newblock \emph{Nature methods}, 17\penalty0 (3):\penalty0 261--272, 2020.

\bibitem[Wahba(1990)]{wahba1990spline}
Wahba, G.
\newblock \emph{Spline models for observational data}.
\newblock SIAM, 1990.

\bibitem[Wilson et~al.(2020)Wilson, Doppa, and Cook]{Wilson2020adaptation}
Wilson, G., Doppa, J.~R., and Cook, D.~J.
\newblock Multi-source deep domain adaptation with weak supervision for
  time-series sensor data.
\newblock \emph{Special Interest Group on Knowledge Discovery and Data Mining
  (SIGKDD)}, 2020.

\bibitem[You et~al.(2019)You, Wang, Long, and Jordan]{you2019towards}
You, K., Wang, X., Long, M., and Jordan, M.
\newblock Towards accurate model selection in deep unsupervised domain
  adaptation.
\newblock In \emph{Proceedings of the International Conference on Machine
  Learning}, pp.\  7124--7133, 2019.

\bibitem[Zellinger et~al.(2017)Zellinger, Grubinger, Lughofer, Natschl{\"a}ger,
  and Saminger-Platz]{zellinger2017central}
Zellinger, W., Grubinger, T., Lughofer, E., Natschl{\"a}ger, T., and
  Saminger-Platz, S.
\newblock Central moment discrepancy (cmd) for domain-invariant representation
  learning.
\newblock \emph{International Conference on Learning Representations}, 2017.

\bibitem[Zellinger et~al.(2021)Zellinger, Shepeleva, Dinu, Eghbal-zadeh,
  Nguyen, Nessler, Pereverzyev, and Moser]{zellinger2021balancing}
Zellinger, W., Shepeleva, N., Dinu, M.-C., Eghbal-zadeh, H., Nguyen, H.~D.,
  Nessler, B., Pereverzyev, S., and Moser, B.~A.
\newblock The balancing principle for parameter choice in distance-regularized
  domain adaptation.
\newblock \emph{Advances in Neural Information Processing Systems}, 34, 2021.

\bibitem[Zellinger et~al.(2023)Zellinger, Kindermann, and
  Pereverzyev]{zellinger2023adaptive}
Zellinger, W., Kindermann, S., and Pereverzyev, S.~V.
\newblock Adaptive learning of density ratios in {RKHS}.
\newblock \emph{Journal of Machine Learning Research}, 24:\penalty0 1--28,
  2023.

\bibitem[Zhang et~al.(2021)Zhang, Bengio, Hardt, Recht, and
  Vinyals]{zhang2021understanding}
Zhang, C., Bengio, S., Hardt, M., Recht, B., and Vinyals, O.
\newblock Understanding deep learning (still) requires rethinking
  generalization.
\newblock \emph{Communications of the ACM}, 64\penalty0 (3):\penalty0 107--115,
  2021.

\bibitem[Zhu et~al.(2021)Zhu, Zhuang, Wang, Ke, Chen, Bian, Xiong, and
  He]{Zhu2021subdomain}
Zhu, Y., Zhuang, F., Wang, J., Ke, G., Chen, J., Bian, J., Xiong, H., and He,
  Q.
\newblock Deep subdomain adaptation network for image classification.
\newblock \emph{IEEE Transactions on Neural Networks and Learning Systems},
  32(4):\penalty0 1713--1722, 2021.

\end{thebibliography}
\bibliographystyle{icml2024}

\newpage
\appendix
\onecolumn
\section{Proof of Theorem~\ref{thm:error_rates_result}}
\label{appendix:proof_of_main_theorem}


We start with the following common technical assumption, essentially requiring the reproducing kernel Hilbert space $\mathcal{H}$ to be rich enough and regular in a weak sense (cf.~\citet{caponnetto2007optimal,marteau2019beyond,gizewski2022regularization,zellinger2023adaptive}):
\\
\vspace{-1em}
\begin{assumption}[technical assumption]
    \label{ass:technical}
    ~
    \vspace{-1em}
    \begin{itemize}[leftmargin=1em]
    \setlength\itemsep{0em}
    \item The kernel $k$ is continuous and bounded $\sup_{x\in\mathcal{X}}\norm{k(x,\cdot)}\leq R$.
    \item An expected risk minimizer $f_\mathcal{H}:=\argmin_{f\in\mathcal{H}} \mathcal{R}(f)$ exists.
    \item The quantities $|\ell_z(0)|, \norm{\nabla \ell_z(0)}, \mathrm{Tr}(\nabla^2 \ell_z(0))$ are almost surely bounded for any realization $z$ of the random variable $Z$.
\end{itemize}
\end{assumption}

To obtain fast error rates, we rely on a \textit{general self-concordance} property of loss functions~\citep{marteau2019beyond}, which follows from Assumption~\ref{ass:self-concordance}.
To this end, we denote by $A[k,h]$ the evaluation of an operator $A:\mathcal{H}\times\mathcal{H}\to\mathbb{R}$ at the arguments (functions) $k,h\in\mathcal{H}$.
\begin{assumption}[generalized self-concordance~\citep{marteau2019beyond}]%
    \label{ass:generalized_self_concordance}%
    \sloppy
    For any $z=(x,y)\in \X\times\Y$, $\ell_z(f)$ is convex in $f$, three times Fr\'echet differentiable, and it exists a set $\varphi(z)\subset\mathcal{H}$ such that
    \begin{align}
    \label{eq:generalized_self_concordant}
        \forall f\in\mathcal{H},\forall h,p\in\mathcal{H}: |\nabla^3 \ell_z(f)[p,h,h]|
        \leq \nabla^2 \ell_z(f)[h,h]\cdot \sup_{q\in\varphi(z)} \left|\langle p, q\rangle\right|.%
    \end{align}%
\end{assumption}%
Under Assumption~\ref{ass:generalized_self_concordance}, an upper bound on the excess risk of iteratively regularized kernel regression can be provided~\citep{beugnot2021beyond}, which reads in our notation as follows.
\begin{lemma}
[{\citet[Theorem~1, first part]{beugnot2021beyond}}]
\label{lemma:beugnot_thm1}
    Let Assumptions~\ref{ass:source_condition},~\ref{ass:technical} and~\ref{ass:generalized_self_concordance} be satisfied for some loss function $\ell$ with corresponding set $\varphi(z)$ satisfying $\sup_{g\in \varphi(z)} \norm{g}\leq R$ for some $R>0$.
    Let further $\delta\in(0,1]$ and fix some $\lambda\in (0,L_0)$, $m+n\geq N$.
    Then, the following bound on the excess risk of a minimizer of Eq.~\eqref{eq:empirical_loss_objective} holds 
    with probability greater than $1-\delta$:
    \begin{align}
        \label{eq:beugnot_main_theorem}
        \mathcal{R}(f_\z^{\lambda,t})-\mathcal{R}(f_\mathcal{H})\leq C_\mathrm{bias} \lambda^{2 s}+C_\mathrm{var} \frac{\mathrm{df}_\lambda}{m+n}, \text{~with~~} s=\min\{r+\frac{1}{2},t\}.
    \end{align}
    
\end{lemma}
We refer to the supplementary material of~\citet{beugnot2021beyond} for the exact values of $L_0,N,C_\mathrm{var},C_\mathrm{bias}$, which depend on $r,\alpha,S,R,t,\delta$ and the distribution $\rho$ but not on the sample size $m+n$.

We will need the following statement.
\begin{lemma}
    \label{lemma:technical}
    Let $\ell:\mathcal{Y}\times\mathbb{R}\to\mathbb{R}$ satisfy Assumption~\ref{ass:self-concordance}.
    Then, for any $z=(x,y)\in \X\times\Y$, the functional $\ell_z:\mathcal{H}\to\mathbb{R}$ defined by $\ell_z(f):=\ell(y,f(x)), f\in\mathcal{H}$ satisfies Assumption~\ref{ass:generalized_self_concordance}.
\end{lemma}
\begin{proof}
    Let $\ell(y,\eta)$ be convex in $\eta$, then, for all $z=(x,y)\in\mathcal{X}\times\mathcal{Y}, f_1,f_2\in\mathcal{H}, t\in [0,1]$, it holds that
    \begin{align*}
        \ell_z(t f_1+(1-t) f_2)&=\ell(y,t f_1(x)+(1-t)f_2(x))\\
        &\leq t \ell(y,f_1(x))+(1-t)\ell(y,f_2(x))\\
        &=t \ell_z(f_1)+(1-t)\ell_z(f_2)
    \end{align*}
    and $\ell_z(f)$ is convex in $f$.
    
    To show Fr\'echet differentiability of $\ell_z(f)$, let us introduce the notation $\ell_y(\eta):=\ell(y,\eta)$ which is Fr\'echet differentiable with $\nabla \ell_y(\eta)=\ell_y'(\eta)$ by definition.
    We also denote by $B_x(f)=f(x)$ the (linear) evaluation functional, which is bounded for $f\in\mathcal{H}$ and it's Fr\'echet derivative $\nabla B_x$ therefore exists and is again the evaluation functional $B_x$.
    Then, by definition, all derivatives in the composition $\nabla \ell_y(B_x(f))\circ \nabla B_x(f)$ exist and are bounded, and therefore it  equals the desired quantity $\nabla \ell_z(f)=\nabla (\ell_y\circ B_x)(f)$ by the chain rule.
    The second and third Fr\'echet derivative can be obtained analogously.
    
    Now note that, for all $z=(x,y)\in\mathcal{X}\times\mathcal{Y}$,
    \begin{align*}
        \left|\nabla^3 \ell_z(f)[p,h,h]\right| &= \left|p(x)h(x)h(x)\ell_y'''(f(x))\right|\\
        &=\left|p(x)h(x)h(x)\left|\ell_y'''(f(x))\right|\right|\\
        &\leq\left|p(x)h(x)h(x)\ell_y''(f(x))\right|\\
        &=\left|p(x)\nabla^2\ell_z(f)[h,h]\right|\\
        &=\left|p(x)\right|\nabla^2\ell_z(f)[h,h],
    \end{align*}
    where the last line follows from the positive semi-definiteness of the Hessian.
    Note that 
    \begin{align*}
        |p(x)|
        =|y|\cdot\left| \langle p,k(x,\cdot)\rangle\right|
        =\left| \langle p,y k(x,\cdot)\rangle\right|\leq \sup_{q\in\varphi(z)} \left| \langle p,q\rangle\right|
    \end{align*}
    follows from the reproducing property and $\varphi(z):=\{y k(x,\cdot)\}$.
\end{proof}

We are now able to combine above reasoning to prove Theorem~\ref{thm:error_rates_result}.

As we note in Section~\ref{sec:notation}, we neglect pathological loss functions and consider only strictly proper composite $\ell$ with twice differentiable Bayes risk $G$.
As a consequence, from Lemma~\ref{lemma:bregman_form_of_loss}, we get
\begin{align}
\label{eq:proof_first_eq}
    B_F(\beta,g(f))-B_F(\beta,g(f_\mathcal{H})) = 2 \left(\mathcal{R}(f)-\mathcal{R}(f^\ast)-\mathcal{R}(f_\mathcal{H})+\mathcal{R}(f^\ast)\right)
    =2 \left(\mathcal{R}(f)-\mathcal{R}(f_\mathcal{H})\right).
\end{align}
Our goal is now to apply Lemma~\ref{lemma:beugnot_thm1}, which requires that Assumption~\ref{ass:generalized_self_concordance} is satisfied with some $\varphi(z)$ satisfying $\sup_{g\in \varphi(z)} \norm{g}\leq R$ and $R>0$.
This follows from Lemma~\ref{lemma:technical} and the fact that $\sup_{(x,y)\in\mathcal{X}\times\mathcal{Y}} \norm{y k(x,\cdot)}\leq R$ by Assumption~\ref{ass:technical}.

Let us choose
\begin{align*}
    \lambda_\ast:=L_0 (m+n)^{-\frac{\alpha}{2 s\alpha +1}}
\end{align*}
which balances the bias and the variance term in Lemma~\ref{lemma:beugnot_thm1} up to a constant factor (that ensures $\lambda_\ast\leq L_0$).
Then, Lemma~\ref{lemma:beugnot_thm1} can be applied and it holds that
\begin{align*}
    B_F(\beta,g(f))-B_F(\beta,g(f_\mathcal{H}))
    \leq 2\left( C_\mathrm{bias} \lambda_\ast^{2 s}+C_\mathrm{var} \frac{S \lambda_\ast^{-\frac{1}{\alpha}}}{m+n}\right)
    =2( C_\mathrm{bias}+C_\mathrm{var} S) (m+n)^{-\frac{2s\alpha}{2s\alpha+1}}
\end{align*}
with probability at least $1-2\delta$ and $m+n\geq N$, for an $N$ independent of $m+n$.
Theorem~\ref{thm:error_rates_result} follows by fixing $$C:=2( C_\mathrm{bias}+C_\mathrm{var} S).$$

\section{Optimization Error}
\label{appendix:optimization_error}

This Section is aimed to provide a more detailed sketch on how one can obtain the error bound Eq.~\eqref{eq:main_bound_numerical}. The first quantity we need to consider in this context is the so called Newton decrement, as developed in \citet{marteau2019globally}. Starting from $f^{\lambda,0}_{\z}=\bar{f}^{\lambda,0}_{\z}=0$, we define the following expressions for $k>0$ :
 \begin{align}
  f^{\lambda,t+1}_{\z}&= \argmin_{f\in\mathcal{H}} \frac{1}{|\z|}\sum_{i=1}^{m+n} \ell(y_i,f(x_i))+ \frac{\lambda}{2}\norm{f-f^{\lambda,t}_{\z}}^2:=\argmin_{f\in\mathcal{H}} \widehat{\mathcal{R}}_\lambda^{t}(f), \quad \nu_\lambda^{t+1}(f) :=\left\|\Hess_\lambda(f)^{-\frac{1}{2}} \nabla \widehat{\mathcal{R}}_\lambda^{t}(f)\right\|,  \label{eq:newton_true}\\
 \bar{f}^{\lambda,t+1}_{\z}&\approx \argmin_{f\in\mathcal{H}} \frac{1}{|\z|}\sum_{i=1}^{m+n} \ell(y_i,f(x_i))+ \frac{\lambda}{2}\norm{f-\bar{f}^{\lambda,t}_{\z}}^2:=\argmin_{f\in\mathcal{H}} \bar{\mathcal{R}}_\lambda^{t}(f), \quad \bar{\nu}_\lambda^{t+1}(f) :=\left\|\Hess_\lambda(f)^{-\frac{1}{2}} \nabla \bar{\mathcal{R}}_\lambda^{t}(f)\right\| , \label{eq:newton_approx}
 \end{align}
 
for a numerical approximation $\bar{f}_{\z}^{\lambda,t}$ of $f^{\lambda,t}_{\z}$ computed as described in Subsection \ref{subsec:algo}.
 Next we need to enforce a bound on the true Newton decrement in Eq.~\eqref{eq:newton_true} when we only have access to $\bar{\mathcal{R}}_\lambda^{t}$. This has been done in \citet[Proposition 1]{beugnot2021beyond}, and we repeat their result in our notation:

\begin{lemma}\citep[Proposition 1]{beugnot2021beyond} \label{lem:newton_decr}
Let $\varepsilon>0$ the target precision. Assume that we can solve each sub-problem with precision $\bar{\epsilon}_k$ :
\begin{align*}
\forall k \in\{1, \ldots, t\}, \quad \bar{\nu}_\lambda^{k-1}\left(\bar{f}^{\lambda,k}_{\z}\right) \leq \bar{\varepsilon}_t=\varepsilon \frac{1.4^{k-t}}{t},
\end{align*}
and that $\varepsilon \leq \frac{\sqrt{\lambda}}{2 R}$. This suffice to achieve an error $\varepsilon$ on the target function:
\begin{align*}
\nu_\lambda^{t-1}\left(\bar{f}^{\lambda,t}_{\z}\right) \leq \varepsilon
\end{align*}
\end{lemma}
 
 To compute this estimator in practice, one would need to solve $t$ optimization problems with decreasing precision. As the CG-algorithm that we are applying (cf. Section~\ref{subsec:algo}) has error complexity $\log( \varepsilon)$ (see Remark~\eqref{eq:rem_numerical}), the complexity of applying this algorithm iteratively would only be $t$ times bigger than computing a single iteration, so achieving the required small magnitudes of the weighted gradient norms Eq.~\eqref{eq:newton_approx}--\eqref{eq:newton_true} is computationally feasible.
 Analyzing the computational burden in terms of the number of samples required to compute the numerical solution of Eq.~\eqref{eq:empirical_loss_objective} would be an interesting avenue for future work.
 To conclude our investigations on how the theoretical guarantees change by the use of inexact solvers, we can adapt \citet[Proposition 2]{beugnot2021beyond} with the same arguments as used in our proof of Theorem ~\ref{thm:error_rates_result}:
 \begin{lemma}
 Let $\delta \in(0,1)$ and assume that the statistical assumptions of Theorem~\ref{thm:error_rates_result} hold as well as the optimization assumptions of Lemma \ref{lem:newton_decr}. Then, the following bound on the excess risk holds with probability greater than $1- \delta$ :
 \begin{align*}
 B_F(\beta,g(\bar{f}_{\z}^{\lambda,t}))&-B_F(\beta,g(f_\mathcal{H}))
\le\tilde{C} \!\left( (m+n)^{-\frac{2s\alpha}{2s\alpha+1}}+\varepsilon\right)
\end{align*}
 with $\tilde{C}>0$ independent of $n,m$ and $\varepsilon$.
 \end{lemma}
 This is all we need to obtain the bound Eq.~\eqref{eq:main_bound_numerical}.

\section{Derivation of Iterated KuLSIF}
\label{sec:app_kulsif}

Starting from objective~\eqref{eq:novel_objective}
\begin{align*}
    f^{\lambda,t+1}:=\argmin_{f\in\mathcal{H}} B_F\!\left(\beta,g(f)\right) + \frac{\lambda}{2} \norm{f-f^{\lambda,t}}^2
\end{align*}
where $f^{\lambda,0}=0$ we get for KuLSIF with $F(h)=\int_\X (h(x)-1)^2/2\diff Q(x)$ and $g(f)=f$
\begin{align*}
    f^{\lambda,t+1} &= \argmin_{f\in\mathcal{H}}  \frac12 \int (\beta - f)^2 dQ + \frac{\lambda}{2} \norm{f-f^{\lambda,t}}^2 \\
     &= \argmin_{f\in\mathcal{H}}  \frac12 \int ((\beta -f^{\lambda,t})-f)^2 dQ + \frac{\lambda}{2} \norm{f}^2.
\end{align*}
Setting $r_t=\beta-f^{\lambda,t}$ gives
\begin{align*}
    \argmin_{f\in\mathcal{H}} \int \frac12 (r_t-f)^2 dQ + \frac{\lambda}{2} \norm{f}^2
\end{align*}
Expanding the squares, using $\beta=\frac{dP}{dQ}$ and ignoring the summands only containing $\beta$ gives:
\begin{align*}
\argmin_{f\in\mathcal{H}} \frac12 \int f^2 dQ +\int f^{\lambda,t} f dQ - \int f dP +\frac{\lambda}{2} \norm{f}^2.
\end{align*}
 The empirical computation of these quantities results in
\begin{align}
    \label{app:empirical_loss_objective}
    f^{\lambda,t+1}_{\z} =\argmin_{f\in\mathcal{H}} \frac{1}{n} \sum_{i=1}^n (f(x'_i)^2 + f(x'_i) f^{\lambda,t}(x'_i)) - \frac{1}{m} \sum_{l=1}^m f(x_l) + \frac{\lambda}{2} \norm{f}^2
\end{align}
By choosing
\begin{align*}
    f^{\lambda,t+1}_{\z} = \sum_{i=1}^n \alpha_i^{t+1} k( \cdot, x'_i) + \sum_{l=1}^m \beta_l^{t+1} k( \cdot, x_l)
\end{align*}
as in~\citet{Kanamori:12} (following the representer Theorem), plugging this into Eq.~\eqref{app:empirical_loss_objective} and taking the derivatives w.r.t. $\alpha_j^{t+1}$ and $\beta_l^{t+1}$ and setting them to $0$, we get as analytic expressions for the coefficients: 
\begin{align*}
    \beta^{t+1} &= \frac{t+1}{m \lambda} 1_m \\
    \alpha^{t+1} &= \left( \lambda + \frac{1}{n} K_{11} \right)^{-1} \left( \lambda \alpha^{t} - K_{12} \frac{1}{\lambda m n} 1_m \right)
\end{align*}
where $1_m$ is the one vector,  $K_{11}$ is the gram sub-matrix for data $\x'$ and $K_{12}$ with respect to $\x'$ and $\x$, and $\alpha^0 = 0$ and $\beta^0 = 0$.

\section{Dataset with Known Density Ratios}
\label{appendix:geometric_dataset}
To
study the accuracy of the iteratively regularized estimations
of density ratios, we follow investigations of~\citet{Kanamori:12} and~\citet{Nguyen:10} to construct high dimensional data with exact known density ratios. Thereby, we both extend and generalize existing settings. For example datasets such as Ringnorm or Twonorm~\cite{Breiman:96} are special cases of a Gaussian Mixture approach with mean and covariance parameters of the respective mixture modes set accordingly. Similarly, experiments in~\citet{Nguyen:10} can be modeled by a Gaussian Mixture distribution. We make the task more complex by randomly sampling the number, the weights, the expectations and the covariances of mixture components of the underlying distributions from a space with higher dimension ($50$) compared to existing work. More specifically, we sample the expectations from $[0, 0.5]^{50}$ and the weights from $[0, 1]$ with additional normalization. Each distribution (source, target) gets assigned to a different mixture model. In each experiment there is a maximum component count of $4$ components. The randomly sampled $n \in \{1, 2, 3 \}$ is the number of components for the source and $4 - n$ for the target distribution. 

\section{Detailed Empirical Evaluations}
\label{appendix:sec:detaile_experiments}
\subsection{Known density ratios}
For each Gaussian mixture dataset we draw $5000$ samples from the underlying distributions. The selection and evaluation of the compared density ratio estimation method is done in a typical train/val/test split approach with split ratios $64/16/20$ respectively. The regularization (hyper) parameter $\lambda$ is selected from $\{10^{-6}, 10^{-5},\ldots, 10^4\}$ and for each experiment $10$ replicates are carried out. Accordingly, the number of iteration steps is selected from $\{1, 2, \ldots, 10\}$ based on the respective loss metric. We follow~\cite{Kanamori:12} in using the Gaussian kernel with kernel width set according to the median heuristic~\cite{Schoelkopf:02} for all compared density ratio estimation methods.

\subsection{Domain Adaptation}
\label{appendix:experiments}
The computation of the results for the domain adaptation benchmark experiment is based on gradient-based training of overall 9174 models where we built parts of our results on the codebase of~\citet{dinu2022aggregation} which results in the following specifics.

AmazonReviews (text data):
$11$ methods $\times\ 14$ parameters $\times\ 12$ domain adaptation tasks $\times\ 3$ seeds $= 5544$ trained models

MiniDomainNet (image data):
$11$ methods $\times\ 8$ parameters $\times\ 5$ domain adaptation tasks $\times\ 3$ seeds $= 1320$ trained models

HHAR (sensory data):
$11$ methods $\times\ 14$ parameters $\times\ 5$ domain adaptation tasks $\times\ 3$ seeds $= 2310$ trained models

Following~\citet{dinu2022aggregation} we used $11$ domain adaptation methods from the AdaTime~\citep{ragab2023adatime} benchmark. For each of these methods and domain adaptation modalities (text, image, sensory) we evaluate all $8$ density ratio estimation on all $22$ domain adaptation scenarios. 
For our experiments we follow the same model implementation and experimental setup as in~\citet{dinu2022aggregation} which results in the usage of fully connected networks for Amazon Reviews and a pretrained ResNet-18 backbone for MiniDomainNet. For training and selecting the density ratio estimation methods within this pipeline we perform an additional train/val split of $80/20$ on the datasets that are used for training the domain adaption methods. The regularization (hyper) parameter $\lambda$ is selected from $\{10^{-6}, 10^{-5},\ldots, 10^4\}$ and due to higher computational effort the number of iterations $t$ is selected from $\{1, 5, 10 \}$ based on the respective Bregman objective \eqref{eq:regularized_Bregman_objective}. We follow~\citet{Kanamori:12} in using the Gaussian kernel with kernel width set according to the median heuristic~\citep{Schoelkopf:02} for all compared density ratio estimation methods. Each experiment is run $3$ times. To test the performance of the compared methods the classification accuracy on the respective test sets of the target distribution is evaluated. We use the ensemble approach introduced in~\citet{dinu2022aggregation} as discussed in Section~\ref{sec:agg}. 

\subsection{Further Ablation Studies} \label{appendix:sec:add_exp}
\paragraph{SOTA Domain Adaptation}
We compare our method to the state-of-the-art results reported in~\citet{dinu2022aggregation}.
Table~\ref{tab:amzn} shows the results for each domain adaptation method (rows) averaged over 12 domain adaptation tasks and three seeds for the Amazon Reviews dataset.
While KuLSIF without our proposed iteration only achieves $0.787$ and does not outperform~\cite{dinu2022aggregation}, our proposed iterated KuLSIF improves~\cite{dinu2022aggregation} from $0.788$ to $0.790$.

\begin{table}[!h]
\centering
\begin{tabular}{l c c c}
\toprule
\multicolumn{4}{c}{\textbf{Domain Adaptation: Amazon Reviews}}\\
\cmidrule{1-4}
\textbf{DA-Method} & {\citet{dinu2022aggregation}} & KuLSIF & Iterated KuLSIF (ours) \\
\midrule
HoMM & $0.788 (\pm 0.010)$ & $0.777 (\pm 0.009)$ & $0.777 (\pm 0.010)$ \\
AdvSKM & $0.780 (\pm 0.009)$  & $0.779 (\pm 0.011)$ & $0.780 (\pm 0.008)$ \\
DIRT & $0.787 (\pm 0.008)$ & $0.787 (\pm 0.011)$ & $0.794 (\pm 0.011)$ \\
DDC & $0.780 (\pm 0.010)$ & $0.780 (\pm 0.010)$ & $0.780 (\pm 0.010)$ \\
CMD & $0.794 (\pm 0.009)$ & $0.790 (\pm 0.010)$ & $0.794 (\pm 0.008)$ \\
MMDA & $0.787 (\pm 0.011)$ & $0.786 (\pm 0.011)$ & $0.789 (\pm 0.010)$ \\
CoDATS & $0.796 (\pm 0.009)$ & $0.795 (\pm 0.012)$ & $0.798 (\pm 0.011)$ \\
Deep-Coral & $0.785 (\pm 0.009)$ & $0.784 (\pm 0.009)$ & $0.784 (\pm 0.009)$ \\
CDAN & $0.788 (\pm 0.010)$ & $0.787 (\pm 0.010)$ & $0.790 (\pm 0.010)$ \\
DANN & $0.797 (\pm 0.009)$ & $0.794 (\pm 0.013)$ & $0.800 (\pm 0.009)$ \\
DSAN & $0.795 (\pm 0.009)$ & $0.794 (\pm 0.011)$ & $0.800 (\pm 0.009)$ \\
\midrule 
Avg. & $0.788 (\pm 0.009)$ & $0.787 (\pm 0.011)$ & $0.790 (\pm 0.010)$ \\
\bottomrule
\end{tabular}
\caption{Mean and standard deviation (after $\pm$) of target classification accuracy on Amazon Reviews over three different random initialization of model weights and 12 domain adaptation tasks.}
\label{tab:amzn}
\end{table}

\paragraph{Extension to Deep Learning}
We follow Section 4.1 of~\citet{rhodes2020telescoping} to analyze the sample efficiency for challenging (different) densities. More precisely, we sample from an extremely peaked Gaussian $p \sim \mathcal{N}(0, 10^{-6})$ and a broad Gaussian $q \sim \mathcal{N}(0, 1)$. The density ratio estimators are based on both standard logistic regression and multi-layer networks.
In Figure~\ref{fig:exp2_3} (left) we can see that our iteration method improves sample efficiency for logistic regression. Additionally, it can be seen that our approach can be extended to neural network based models for which sample efficiency is improved as well.

\begin{figure}[!h]
  \centering
    \includegraphics[width=1\textwidth]{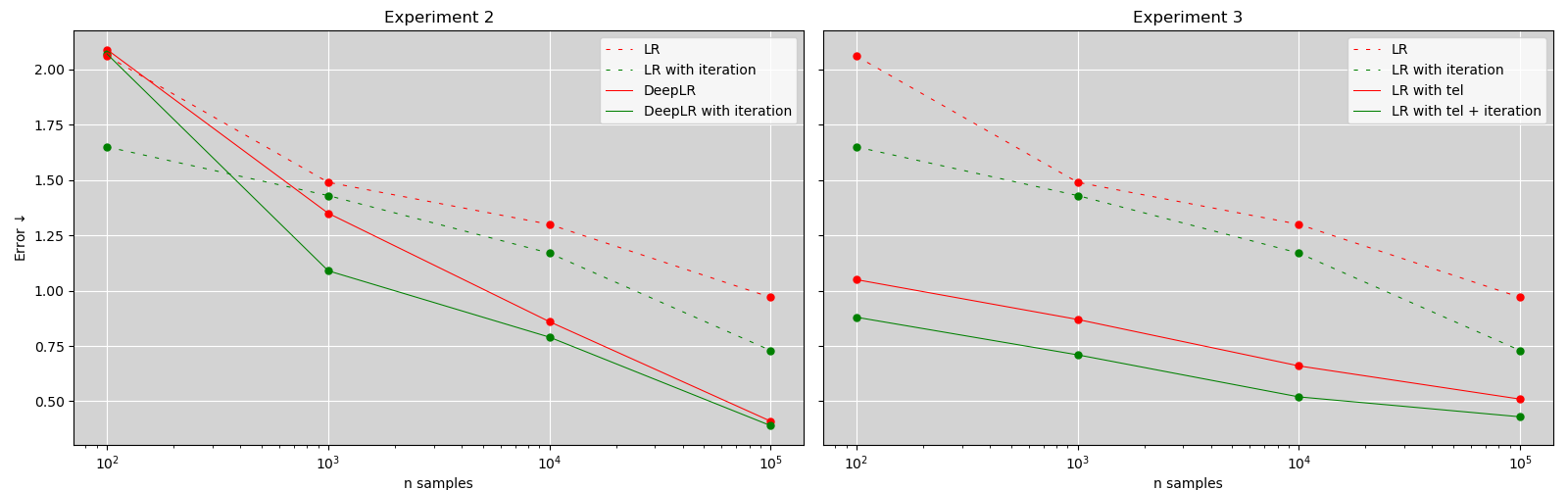}
  \caption{Sample efficiency curves for various density ratio estimators and approaches, the error is measured by L1-norm. Left: Comparison logistic regression and multi-layer network logistic regression density ratio estimators both without and with iteration. Right: Comparison of logistic regression with its iterated version, the telescoping (tel) approach from~\citet{rhodes2020telescoping} and the telescoping approach combined with our iterative method.}
  \label{fig:exp2_3}
\end{figure}

\paragraph{Combining Iteration with Telescoping}
We extend the previous experiment such that we sample from the same respective distributions and combine the telescoping framework of~\citet{rhodes2020telescoping} with our iteration method. In Figure~\ref{fig:exp2_3} (right) it can be seen that telescoping improves the non-iterated density ratio estimator as in~\citet{rhodes2020telescoping}. Additionally, the combination of telescoping and our iteration method further improves sample efficiency of both approaches.
\pagebreak

\subsection{Detailed Empirical Results}
In this section, we add all result tables for the datasets described in the main paper. Table~\ref{tab:first_table_appendix}-\ref{tab:mmda_mdn}
show all results on the MiniDomainNet image datasets, Table~\ref{tab:amazon_reviews_first}-\ref{tab:amazon_reviews_last} show all results on the Amazon Reviews text datasets, and Table~\ref{tab:HHAR_first}-\ref{tab:last_table_appendix} show all results on the HHAR sensory datasets. 
Each table shows the results for all the respective datasets given one domain adaptation method from AdaTime. Averaging the results in these tables over the scenario datasets and combining them leads to Table~\ref{tab:mdn}.

\begin{table*}[h]
\scalebox{0.725}{
\begin{tabular}{ l | c c c c | c c c c}
\toprule
  & \multicolumn{4}{c|}{\textbf{No Iteration}} & \multicolumn{4}{c}{\textbf{Iteration}}\\
\textbf{Scenario} & KuLSIF & Exp & LR & SQ & KuLSIF & Exp & LR & SQ \\
\midrule
 R $\rightarrow$ C    & $0.568 (\pm 0.008)$    & $0.561 (\pm 0.013)$    & $0.552 (\pm 0.006)$    & $0.045 (\pm 0.008)$    & $0.568 (\pm 0.009)$    & $0.561 (\pm 0.013)$    & ${\bf 0.561 (\pm 0.013})$    & ${\bf 0.046 (\pm 0.012})$ \\
 R $\rightarrow$ I    & $\bf{0.382} (\pm 0.010)$    & $0.389 (\pm 0.007)$    & $0.385 (\pm 0.026)$    & $0.050 (\pm 0.004)$    & $0.377 (\pm 0.013)$    & $0.389 (\pm 0.007)$    & ${\bf 0.389 (\pm 0.007})$    & ${\bf 0.053 (\pm 0.008})$ \\
 R $\rightarrow$ P    & $0.717 (\pm 0.005)$    & $0.712 (\pm 0.014)$    & $0.712 (\pm 0.010)$    & $0.039 (\pm 0.005)$    & ${\bf 0.721 (\pm 0.003})$    & $0.712 (\pm 0.014)$    & $0.712 (\pm 0.014)$    & $0.039 (\pm 0.003)$ \\
 R $\rightarrow$ Q    & $0.332 (\pm 0.005)$    & $0.325 (\pm 0.007)$    & $0.324 (\pm 0.011)$    & $0.146 (\pm 0.021)$    & ${\bf 0.333 (\pm 0.005})$    & $0.325 (\pm 0.007)$    & ${\bf 0.325 (\pm 0.007})$    & ${\bf 0.147 (\pm 0.021})$ \\
 R $\rightarrow$ S    & $\bf{0.584} (\pm 0.004)$    & $0.586 (\pm 0.003)$    & $\bf{0.588} (\pm 0.005)$    & $0.080 (\pm 0.011)$    & $0.582 (\pm 0.004)$    & $0.586 (\pm 0.003)$    & $0.586 (\pm 0.003)$    & $0.080 (\pm 0.013)$ \\
\midrule
 Avg.      & $0.516 (\pm 0.006)$    & $0.515 (\pm 0.009)$    & $0.512 (\pm 0.012)$    & $0.072 (\pm 0.010)$    & $0.516 (\pm 0.007)$    & $0.515 (\pm 0.009)$    & ${\bf 0.515 (\pm 0.009})$    & ${\bf 0.073 (\pm 0.012})$ \\
 \bottomrule
\end{tabular}
}
\caption{Mean and standard deviation (after $\pm$) of target classification accuracy on MiniDomainNet for the aggregation method in objective~\eqref{eq:aggregation} with models computed by AdvSKM~\cite{Liu2021kernel}.}
\label{tab:first_table_appendix}
\end{table*}

\begin{table*}
\scalebox{0.725}{
\begin{tabular}{ l | c c c c | c c c c}
\toprule
  & \multicolumn{4}{c|}{\textbf{No Iteration}} & \multicolumn{4}{c}{\textbf{Iteration}}\\
\textbf{Scenario} & KuLSIF & Exp & LR & SQ & KuLSIF & Exp & LR & SQ \\
\midrule
 R $\rightarrow$ C    & $0.592 (\pm 0.018)$    & $0.595 (\pm 0.029)$    & $0.594 (\pm 0.023)$    & $0.031 (\pm 0.029)$    & ${\bf 0.598 (\pm 0.022})$    & $0.595 (\pm 0.029)$    & ${\bf 0.595 (\pm 0.029})$    & ${\bf 0.032 (\pm 0.029})$ \\
 R $\rightarrow$ I    & $\bf{0.372} (\pm 0.013)$    & $0.381 (\pm 0.012)$    & $0.337 (\pm 0.056)$    & $0.083 (\pm 0.007)$    & $0.371 (\pm 0.016)$    & $0.381 (\pm 0.012)$    & ${\bf 0.381 (\pm 0.012})$    & $0.083 (\pm 0.007)$ \\
 R $\rightarrow$ P    & $0.729 (\pm 0.004)$    & $0.720 (\pm 0.015)$    & $\bf{0.725} (\pm 0.011)$    & $0.053 (\pm 0.014)$    & ${\bf 0.731 (\pm 0.004})$    & $0.720 (\pm 0.015)$    & $0.720 (\pm 0.015)$    & ${\bf 0.055 (\pm 0.017})$ \\
 R $\rightarrow$ Q    & $0.353 (\pm 0.018)$    & $0.352 (\pm 0.020)$    & $\bf{0.354} (\pm 0.017)$    & $0.089 (\pm 0.058)$    & ${\bf 0.356 (\pm 0.012})$    & $0.352 (\pm 0.020)$    & $0.352 (\pm 0.020)$    & $0.089 (\pm 0.058)$ \\
 R $\rightarrow$ S    & $0.609 (\pm 0.008)$    & $0.607 (\pm 0.012)$    & $\bf{0.611} (\pm 0.009)$    & $0.052 (\pm 0.015)$    & ${\bf 0.610 (\pm 0.008})$    & $0.607 (\pm 0.012)$    & $0.607 (\pm 0.012)$    & $0.052 (\pm 0.016)$ \\
\midrule
 Avg.      & $0.531 (\pm 0.012)$    & $0.531 (\pm 0.017)$    & $0.524 (\pm 0.023)$    & $0.062 (\pm 0.025)$    & ${\bf 0.533 (\pm 0.012})$    & $0.531 (\pm 0.017)$    & ${\bf 0.531 (\pm 0.017})$    & $0.062 (\pm 0.025)$ \\
 \bottomrule
\end{tabular}
}
\caption{Mean and standard deviation (after $\pm$) of target classification accuracy on MiniDomainNet for the aggregation method in objective~\eqref{eq:aggregation} with models computed by CDAN~\cite{Long2018conditional}.}
\end{table*}

\begin{table*}
\scalebox{0.725}{
\begin{tabular}{ l | c c c c | c c c c}
\toprule
  & \multicolumn{4}{c|}{\textbf{No Iteration}} & \multicolumn{4}{c}{\textbf{Iteration}}\\
\textbf{Scenario} & KuLSIF & Exp & LR & SQ & KuLSIF & Exp & LR & SQ \\
\midrule
 R $\rightarrow$ C    & $0.576 (\pm 0.015)$    & $0.560 (\pm 0.026)$    & $0.552 (\pm 0.017)$    & $0.084 (\pm 0.021)$    & ${\bf 0.593 (\pm 0.009})$    & $0.560 (\pm 0.026)$    & ${\bf 0.560 (\pm 0.026})$    & $0.084 (\pm 0.021)$ \\
 R $\rightarrow$ I    & $\bf{0.383} (\pm 0.009)$    & $0.391 (\pm 0.021)$    & $0.381 (\pm 0.033)$    & $\bf{0.083} (\pm 0.030)$    & $0.379 (\pm 0.007)$    & $0.391 (\pm 0.021)$    & ${\bf 0.391 (\pm 0.021})$    & $0.081 (\pm 0.028)$ \\
 R $\rightarrow$ P    & $0.739 (\pm 0.008)$    & $0.739 (\pm 0.016)$    & $0.742 (\pm 0.011)$    & $0.030 (\pm 0.013)$    & ${\bf 0.740 (\pm 0.004})$    & $0.739 (\pm 0.016)$    & $0.739 (\pm 0.016)$    & $0.030 (\pm 0.011)$ \\
 R $\rightarrow$ Q    & $0.354 (\pm 0.013)$    & $0.336 (\pm 0.030)$    & $0.335 (\pm 0.028)$    & $0.179 (\pm 0.039)$    & ${\bf 0.364 (\pm 0.007})$    & $0.336 (\pm 0.030)$    & ${\bf 0.336 (\pm 0.030})$    & $0.179 (\pm 0.039)$ \\
 R $\rightarrow$ S    & $\bf{0.595} (\pm 0.005)$    & $0.594 (\pm 0.010)$    & $0.597 (\pm 0.014)$    & $0.080 (\pm 0.029)$    & $0.592 (\pm 0.003)$    & $0.594 (\pm 0.010)$    & $0.594 (\pm 0.010)$    & ${\bf 0.083 (\pm 0.028})$ \\
\midrule
 Avg.      & $0.529 (\pm 0.010)$    & $0.524 (\pm 0.021)$    & $0.521 (\pm 0.021)$    & $0.091 (\pm 0.026)$    & ${\bf 0.533 (\pm 0.006})$    & $0.524 (\pm 0.021)$    & ${\bf 0.524 (\pm 0.021})$    & $0.091 (\pm 0.025)$ \\
 \bottomrule
\end{tabular}
}
\caption{Mean and standard deviation (after $\pm$) of target classification accuracy on MiniDomainNet for the aggregation method in objective~\eqref{eq:aggregation} with models computed by CMD~\citep{zellinger2017central}.}
\end{table*}

\begin{table*}
\scalebox{0.725}{
\begin{tabular}{ l | c c c c | c c c c}
\toprule
  & \multicolumn{4}{c|}{\textbf{No Iteration}} & \multicolumn{4}{c}{\textbf{Iteration}}\\
\textbf{Scenario} & KuLSIF & Exp & LR & SQ & KuLSIF & Exp & LR & SQ \\
\midrule
 R $\rightarrow$ C    & $0.597 (\pm 0.011)$    & $0.577 (\pm 0.010)$    & $\bf{0.594} (\pm 0.010)$    & $\bf{0.057} (\pm 0.020)$    & ${\bf 0.604 (\pm 0.013})$    & $0.577 (\pm 0.010)$    & $0.577 (\pm 0.010)$    & $0.056 (\pm 0.020)$ \\
 R $\rightarrow$ I    & $\bf{0.366} (\pm 0.012)$    & $0.368 (\pm 0.023)$    & $0.349 (\pm 0.038)$    & $0.095 (\pm 0.035)$    & $0.360 (\pm 0.005)$    & $0.368 (\pm 0.023)$    & ${\bf 0.368 (\pm 0.023})$    & ${\bf 0.141 (\pm 0.071})$ \\
 R $\rightarrow$ P    & $0.739 (\pm 0.015)$    & $0.733 (\pm 0.018)$    & $0.737 (\pm 0.017)$    & $\bf{0.038} (\pm 0.019)$    & ${\bf 0.743 (\pm 0.016})$    & $0.733 (\pm 0.018)$    & $0.733 (\pm 0.018)$    & $0.037 (\pm 0.017)$ \\
 R $\rightarrow$ Q    & $0.361 (\pm 0.013)$    & $0.364 (\pm 0.014)$    & $0.355 (\pm 0.027)$    & $0.065 (\pm 0.019)$    & ${\bf 0.366 (\pm 0.010})$    & $0.364 (\pm 0.014)$    & ${\bf 0.364 (\pm 0.014})$    & ${\bf 0.066 (\pm 0.019})$ \\
 R $\rightarrow$ S    & $0.619 (\pm 0.008)$    & $0.616 (\pm 0.008)$    & $0.616 (\pm 0.009)$    & $\bf{0.049} (\pm 0.007)$    & ${\bf 0.622 (\pm 0.009})$    & $0.616 (\pm 0.008)$    & $0.616 (\pm 0.008)$    & $0.048 (\pm 0.006)$ \\
\midrule
 Avg.      & $0.536 (\pm 0.012)$    & $0.532 (\pm 0.015)$    & $0.530 (\pm 0.020)$    & $0.061 (\pm 0.020)$    & ${\bf 0.539 (\pm 0.011})$    & $0.532 (\pm 0.015)$    & ${\bf 0.532 (\pm 0.015})$    & ${\bf 0.069 (\pm 0.027})$ \\
 \bottomrule
\end{tabular}
}
\caption{Mean and standard deviation (after $\pm$) of target classification accuracy on MiniDomainNet for the aggregation method in objective~\eqref{eq:aggregation} with models computed by CoDATS~\citep{Wilson2020adaptation}.}
\end{table*}

\begin{table*}
\scalebox{0.725}{
\begin{tabular}{ l | c c c c | c c c c}
\toprule
  & \multicolumn{4}{c|}{\textbf{No Iteration}} & \multicolumn{4}{c}{\textbf{Iteration}}\\
\textbf{Scenario} & KuLSIF & Exp & LR & SQ & KuLSIF & Exp & LR & SQ \\
\midrule
 R $\rightarrow$ C    & $0.591 (\pm 0.015)$    & $0.576 (\pm 0.013)$    & $\bf{0.579} (\pm 0.012)$    & $\bf{0.052} (\pm 0.026)$    & ${\bf 0.607 (\pm 0.008})$    & $0.576 (\pm 0.013)$    & $0.576 (\pm 0.013)$    & $0.050 (\pm 0.025)$ \\
 R $\rightarrow$ I    & $0.374 (\pm 0.009)$    & $0.372 (\pm 0.011)$    & $0.369 (\pm 0.026)$    & $\bf{0.058} (\pm 0.010)$    & $0.374 (\pm 0.003)$    & $0.372 (\pm 0.011)$    & ${\bf 0.372 (\pm 0.011})$    & $0.056 (\pm 0.011)$ \\
 R $\rightarrow$ P    & $0.720 (\pm 0.007)$    & $0.716 (\pm 0.008)$    & $0.699 (\pm 0.032)$    & $\bf{0.044} (\pm 0.012)$    & ${\bf 0.721 (\pm 0.008})$    & $0.716 (\pm 0.008)$    & ${\bf 0.716 (\pm 0.008})$    & $0.044 (\pm 0.008)$ \\
 R $\rightarrow$ Q    & $0.357 (\pm 0.009)$    & $0.340 (\pm 0.016)$    & $0.337 (\pm 0.016)$    & $\bf{0.078} (\pm 0.006)$    & ${\bf 0.374 (\pm 0.003})$    & $0.340 (\pm 0.016)$    & ${\bf 0.340 (\pm 0.016})$    & $0.078 (\pm 0.006)$ \\
 R $\rightarrow$ S    & $0.613 (\pm 0.007)$    & $0.607 (\pm 0.010)$    & $\bf{0.614} (\pm 0.007)$    & $\bf{0.069} (\pm 0.022)$    & ${\bf 0.617 (\pm 0.005})$    & $0.607 (\pm 0.010)$    & $0.607 (\pm 0.010)$    & $0.069 (\pm 0.022)$ \\
\midrule
 Avg.      & $0.531 (\pm 0.009)$    & $0.522 (\pm 0.012)$    & $0.520 (\pm 0.019)$    & $\bf{0.060} (\pm 0.015)$    & ${\bf 0.539 (\pm 0.005})$    & $0.522 (\pm 0.012)$    & ${\bf 0.522 (\pm 0.012})$    & $0.059 (\pm 0.015)$ \\
 \bottomrule
\end{tabular}
}
\caption{Mean and standard deviation (after $\pm$) of target classification accuracy on MiniDomainNet for the aggregation method in objective~\eqref{eq:aggregation} with models computed by DANN~\citep{ganin2016domain}.}
\end{table*}

\begin{table*}
\scalebox{0.725}{
\begin{tabular}{ l | c c c c | c c c c}
\toprule
  & \multicolumn{4}{c|}{\textbf{No Iteration}} & \multicolumn{4}{c}{\textbf{Iteration}}\\
\textbf{Scenario} & KuLSIF & Exp & LR & SQ & KuLSIF & Exp & LR & SQ \\
\midrule
 R $\rightarrow$ C    & $\bf{0.569} (\pm 0.027)$    & $0.571 (\pm 0.029)$    & $0.563 (\pm 0.028)$    & $0.042 (\pm 0.025)$    & $0.568 (\pm 0.025)$    & $0.571 (\pm 0.029)$    & ${\bf 0.571 (\pm 0.029})$    & ${\bf 0.044 (\pm 0.024})$ \\
 R $\rightarrow$ I    & $\bf{0.389} (\pm 0.011)$    & $0.391 (\pm 0.016)$    & $0.390 (\pm 0.011)$    & $0.047 (\pm 0.012)$    & $0.385 (\pm 0.008)$    & $0.391 (\pm 0.016)$    & ${\bf 0.391 (\pm 0.016})$    & ${\bf 0.049 (\pm 0.014})$ \\
 R $\rightarrow$ P    & $\bf{0.717} (\pm 0.002)$    & $0.721 (\pm 0.003)$    & $0.715 (\pm 0.002)$    & $0.034 (\pm 0.008)$    & $0.715 (\pm 0.002)$    & $0.721 (\pm 0.003)$    & ${\bf 0.721 (\pm 0.003})$    & $0.034 (\pm 0.005)$ \\
 R $\rightarrow$ Q    & $0.333 (\pm 0.004)$    & $0.323 (\pm 0.013)$    & $0.324 (\pm 0.013)$    & $0.156 (\pm 0.017)$    & ${\bf 0.335 (\pm 0.003})$    & $0.323 (\pm 0.013)$    & $0.323 (\pm 0.013)$    & $0.156 (\pm 0.018)$ \\
 R $\rightarrow$ S    & $\bf{0.580} (\pm 0.005)$    & $0.578 (\pm 0.005)$    & $0.579 (\pm 0.004)$    & $0.089 (\pm 0.012)$    & $0.579 (\pm 0.007)$    & $0.578 (\pm 0.005)$    & $0.578 (\pm 0.005)$    & $0.089 (\pm 0.011)$ \\
\midrule
 Avg.      & $\bf{0.517} (\pm 0.010)$    & $0.517 (\pm 0.013)$    & $0.514 (\pm 0.012)$    & $0.074 (\pm 0.015)$    & $0.516 (\pm 0.009)$    & $0.517 (\pm 0.013)$    & ${\bf 0.517 (\pm 0.013})$    & $0.074 (\pm 0.014)$ \\
 \bottomrule
\end{tabular}
}
\caption{Mean and standard deviation (after $\pm$) of target classification accuracy on MiniDomainNet for the aggregation method in objective~\eqref{eq:aggregation} with models computed by DDC~\citep{tzeng2014deep}.}
\end{table*}

\begin{table*}
\scalebox{0.725}{
\begin{tabular}{ l | c c c c | c c c c}
\toprule
  & \multicolumn{4}{c|}{\textbf{No Iteration}} & \multicolumn{4}{c}{\textbf{Iteration}}\\
\textbf{Scenario} & KuLSIF & Exp & LR & SQ & KuLSIF & Exp & LR & SQ \\
\midrule
 R $\rightarrow$ C    & $0.588 (\pm 0.021)$    & $0.569 (\pm 0.016)$    & $0.567 (\pm 0.012)$    & $0.051 (\pm 0.041)$    & ${\bf 0.590 (\pm 0.022})$    & $0.569 (\pm 0.016)$    & ${\bf 0.569 (\pm 0.016})$    & $0.051 (\pm 0.042)$ \\
 R $\rightarrow$ I    & $\bf{0.369} (\pm 0.012)$    & $0.372 (\pm 0.008)$    & $\bf{0.381} (\pm 0.018)$    & $0.087 (\pm 0.005)$    & $0.367 (\pm 0.012)$    & $0.372 (\pm 0.008)$    & $0.372 (\pm 0.008)$    & ${\bf 0.088 (\pm 0.006})$ \\
 R $\rightarrow$ P    & $0.735 (\pm 0.011)$    & $0.737 (\pm 0.010)$    & $\bf{0.734} (\pm 0.005)$    & $0.034 (\pm 0.009)$    & ${\bf 0.738 (\pm 0.012})$    & $0.737 (\pm 0.010)$    & ${\bf 0.737 (\pm 0.010})$    & $0.033 (\pm 0.010)$ \\
 R $\rightarrow$ Q    & $0.363 (\pm 0.011)$    & $0.348 (\pm 0.021)$    & $\bf{0.349} (\pm 0.016)$    & $0.134 (\pm 0.018)$    & ${\bf 0.380 (\pm 0.002})$    & $0.348 (\pm 0.021)$    & $0.348 (\pm 0.021)$    & ${\bf 0.135 (\pm 0.019})$ \\
 R $\rightarrow$ S    & $0.621 (\pm 0.006)$    & $0.616 (\pm 0.009)$    & $0.617 (\pm 0.009)$    & $0.041 (\pm 0.011)$    & ${\bf 0.622 (\pm 0.001})$    & $0.616 (\pm 0.009)$    & $0.616 (\pm 0.009)$    & ${\bf 0.042 (\pm 0.013})$ \\
\midrule
 Avg.      & $0.535 (\pm 0.012)$    & $0.528 (\pm 0.013)$    & $0.530 (\pm 0.012)$    & $0.070 (\pm 0.017)$    & ${\bf 0.539 (\pm 0.010})$    & $0.528 (\pm 0.013)$    & $0.528 (\pm 0.013)$    & $0.070 (\pm 0.018)$ \\
 \bottomrule
\end{tabular}
}
\caption{Mean and standard deviation (after $\pm$) of target classification accuracy on MiniDomainNet for the aggregation method in objective~\eqref{eq:aggregation} with models computed by DeepCoral~\citep{Sun2017correlation}.}
\end{table*}

\begin{table*}
\scalebox{0.725}{
\begin{tabular}{ l | c c c c | c c c c}
\toprule
  & \multicolumn{4}{c|}{\textbf{No Iteration}} & \multicolumn{4}{c}{\textbf{Iteration}}\\
\textbf{Scenario} & KuLSIF & Exp & LR & SQ & KuLSIF & Exp & LR & SQ \\
\midrule
 R $\rightarrow$ C    & $0.572 (\pm 0.021)$    & $0.474 (\pm 0.176)$    & $0.566 (\pm 0.022)$    & $\bf{0.081} (\pm 0.032)$    & ${\bf 0.577 (\pm 0.025})$    & $0.474 (\pm 0.176)$    & ${\bf 0.567 (\pm 0.018})$    & $0.079 (\pm 0.034)$ \\
 R $\rightarrow$ I    & $0.376 (\pm 0.047)$    & $0.236 (\pm 0.135)$    & $\bf{0.406} (\pm 0.025)$    & $\bf{0.083} (\pm 0.022)$    & ${\bf 0.377 (\pm 0.034})$    & $0.236 (\pm 0.135)$    & $0.403 (\pm 0.027)$    & $0.076 (\pm 0.018)$ \\
 R $\rightarrow$ P    & $\bf{0.714} (\pm 0.026)$    & $0.542 (\pm 0.271)$    & $\bf{0.713} (\pm 0.025)$    & $0.023 (\pm 0.013)$    & $0.713 (\pm 0.030)$    & $0.542 (\pm 0.271)$    & $0.712 (\pm 0.028)$    & $0.023 (\pm 0.010)$ \\
 R $\rightarrow$ Q    & $\bf{0.336} (\pm 0.018)$    & $0.248 (\pm 0.055)$    & $\bf{0.329} (\pm 0.021)$    & $0.133 (\pm 0.030)$    & $0.325 (\pm 0.036)$    & $0.248 (\pm 0.055)$    & $0.328 (\pm 0.026)$    & ${\bf 0.135 (\pm 0.028})$ \\
 R $\rightarrow$ S    & $\bf{0.586} (\pm 0.020)$    & $0.431 (\pm 0.245)$    & $\bf{0.587} (\pm 0.021)$    & $0.092 (\pm 0.029)$    & $0.583 (\pm 0.016)$    & $0.431 (\pm 0.245)$    & $0.584 (\pm 0.022)$    & $0.092 (\pm 0.028)$ \\
\midrule
 Avg.      & $\bf{0.517} (\pm 0.026)$    & $0.386 (\pm 0.177)$    & $\bf{0.520} (\pm 0.023)$    & $0\bf{.082} (\pm 0.025)$    & $0.515 (\pm 0.028)$    & $0.386 (\pm 0.177)$    & $0.519 (\pm 0.024)$    & $0.081 (\pm 0.024)$ \\
 \bottomrule
\end{tabular}
}
\caption{Mean and standard deviation (after $\pm$) of target classification accuracy on MiniDomainNet for the aggregation method in objective~\eqref{eq:aggregation} with models computed by DIRT~\citep{Shu2018dirt}.}
\end{table*}

\begin{table*}
\scalebox{0.725}{
\begin{tabular}{ l | c c c c | c c c c}
\toprule
  & \multicolumn{4}{c|}{\textbf{No Iteration}} & \multicolumn{4}{c}{\textbf{Iteration}}\\
\textbf{Scenario} & KuLSIF & Exp & LR & SQ & KuLSIF & Exp & LR & SQ \\
\midrule
 R $\rightarrow$ C    & $0.623 (\pm 0.015)$    & $0.603 (\pm 0.026)$    & $0.596 (\pm 0.019)$    & $0.034 (\pm 0.028)$    & ${\bf 0.639 (\pm 0.008})$    & $0.603 (\pm 0.026)$    & ${\bf 0.603 (\pm 0.026})$    & ${\bf 0.036 (\pm 0.027})$ \\
 R $\rightarrow$ I    & $0.374 (\pm 0.010)$    & $0.367 (\pm 0.008)$    & $0.350 (\pm 0.014)$    & $0.102 (\pm 0.029)$    & ${\bf 0.376 (\pm 0.012})$    & $0.367 (\pm 0.008)$    & ${\bf 0.367 (\pm 0.008})$    & ${\bf 0.104 (\pm 0.032})$ \\
 R $\rightarrow$ P    & $\bf{0.720} (\pm 0.007)$    & $0.724 (\pm 0.008)$    & $0.720 (\pm 0.008)$    & $0.042 (\pm 0.008)$    & $0.719 (\pm 0.007)$    & $0.724 (\pm 0.008)$    & ${\bf 0.724 (\pm 0.008})$    & ${\bf 0.044 (\pm 0.008})$ \\
 R $\rightarrow$ Q    & $0.364 (\pm 0.012)$    & $0.355 (\pm 0.018)$    & $0.354 (\pm 0.018)$    & $0.102 (\pm 0.034)$    & ${\bf 0.372 (\pm 0.011})$    & $0.355 (\pm 0.018)$    & ${\bf 0.355 (\pm 0.018})$    & $0.102 (\pm 0.033)$ \\
 R $\rightarrow$ S    & $0.616 (\pm 0.012)$    & $0.609 (\pm 0.015)$    & $0.615 (\pm 0.016)$    & $0.059 (\pm 0.025)$    & ${\bf 0.622 (\pm 0.013})$    & $0.609 (\pm 0.015)$    & $0.609 (\pm 0.015)$    & ${\bf 0.061 (\pm 0.023})$ \\
\midrule
 Avg.      & $0.539 (\pm 0.011)$    & $0.532 (\pm 0.015)$    & $0.527 (\pm 0.015)$    & $0.068 (\pm 0.025)$    & ${\bf 0.546 (\pm 0.010})$    & $0.532 (\pm 0.015)$    & ${\bf 0.532 (\pm 0.015})$    & ${\bf 0.069 (\pm 0.025})$ \\
 \bottomrule
\end{tabular}
}
\caption{Mean and standard deviation (after $\pm$) of target classification accuracy on MiniDomainNet for the aggregation method in objective~\eqref{eq:aggregation} with models computed by DSAN~\citep{Zhu2021subdomain}.}
\end{table*}

\begin{table*}
\scalebox{0.725}{
\begin{tabular}{ l | c c c c | c c c c}
\toprule
  & \multicolumn{4}{c|}{\textbf{No Iteration}} & \multicolumn{4}{c}{\textbf{Iteration}}\\
\textbf{Scenario} & KuLSIF & Exp & LR & SQ & KuLSIF & Exp & LR & SQ \\
\midrule
 R $\rightarrow$ C    & $0.590 (\pm 0.020)$    & $0.588 (\pm 0.020)$    & $0.577 (\pm 0.020)$    & $0.041 (\pm 0.014)$    & ${\bf 0.596 (\pm 0.009})$    & $0.588 (\pm 0.020)$    & ${\bf 0.588 (\pm 0.020})$    & $0.041 (\pm 0.013)$ \\
 R $\rightarrow$ I    & $\bf{0.399} (\pm 0.006)$    & $0.407 (\pm 0.009)$    & $0.378 (\pm 0.025)$    & $0.066 (\pm 0.015)$    & $0.392 (\pm 0.007)$    & $0.407 (\pm 0.009)$    & ${\bf 0.407 (\pm 0.009})$    & ${\bf 0.068 (\pm 0.013})$ \\
 R $\rightarrow$ P    & $0.729 (\pm 0.002)$    & $0.726 (\pm 0.009)$    & $0.726 (\pm 0.010)$    & $0.041 (\pm 0.010)$    & ${\bf 0.730 (\pm 0.003})$    & $0.726 (\pm 0.009)$    & $0.726 (\pm 0.009)$    & $0.041 (\pm 0.011)$ \\
 R $\rightarrow$ Q    & $0.351 (\pm 0.005)$    & $0.333 (\pm 0.013)$    & $0.333 (\pm 0.013)$    & $0.173 (\pm 0.014)$    & ${\bf 0.353 (\pm 0.003})$    & $0.333 (\pm 0.013)$    & $0.333 (\pm 0.013)$    & $0.173 (\pm 0.014)$ \\
 R $\rightarrow$ S    & $\bf{0.588} (\pm 0.008)$    & $0.593 (\pm 0.011)$    & $0.592 (\pm 0.008)$    & $0.067 (\pm 0.006)$    & $0.586 (\pm 0.008)$    & $0.593 (\pm 0.011)$    & ${\bf 0.593 (\pm 0.011})$    & $0.066 (\pm 0.007)$ \\
\midrule
 Avg.      & $0.531 (\pm 0.008)$    & $0.529 (\pm 0.012)$    & $0.521 (\pm 0.015)$    & $0.078 (\pm 0.012)$    & $0.531 (\pm 0.006)$    & $0.529 (\pm 0.012)$    & ${\bf 0.529 (\pm 0.012})$    & $0.078 (\pm 0.011)$ \\
 \bottomrule
\end{tabular}
}
\caption{Mean and standard deviation (after $\pm$) of target classification accuracy on MiniDomainNet for the aggregation method in objective~\eqref{eq:aggregation} with models computed by HoMM~\citep{Chen2020moment}.}
\end{table*}

\begin{table*}
\scalebox{0.725}{
\begin{tabular}{ l | c c c c | c c c c}
\toprule
  & \multicolumn{4}{c|}{\textbf{No Iteration}} & \multicolumn{4}{c}{\textbf{Iteration}}\\
\textbf{Scenario} & KuLSIF & Exp & LR & SQ & KuLSIF & Exp & LR & SQ \\
\midrule
 R $\rightarrow$ C    & $0.585 (\pm 0.013)$    & $0.590 (\pm 0.016)$    & $0.589 (\pm 0.024)$    & $0.035 (\pm 0.020)$    & $0.585 (\pm 0.009)$    & $0.590 (\pm 0.016)$    & ${\bf 0.590 (\pm 0.016})$    & $0.035 (\pm 0.020)$ \\
 R $\rightarrow$ I    & $\bf{0.377} (\pm 0.009)$    & $0.372 (\pm 0.008)$    & $0.378 (\pm 0.007)$    & $0.066 (\pm 0.012)$    & $0.373 (\pm 0.003)$    & $0.372 (\pm 0.008)$    & $0.372 (\pm 0.008)$    & ${\bf 0.068 (\pm 0.012})$ \\
 R $\rightarrow$ P    & $0.722 (\pm 0.004)$    & $0.729 (\pm 0.006)$    & $0.725 (\pm 0.006)$    & $0.031 (\pm 0.004)$    & ${\bf 0.723 (\pm 0.004})$    & $0.729 (\pm 0.006)$    & ${\bf 0.729 (\pm 0.006})$    & ${\bf 0.033 (\pm 0.003})$ \\
 R $\rightarrow$ Q    & $0.354 (\pm 0.008)$    & $0.348 (\pm 0.016)$    & $0.349 (\pm 0.014)$    & $0.109 (\pm 0.019)$    & ${\bf 0.362 (\pm 0.009})$    & $0.348 (\pm 0.016)$    & $0.348 (\pm 0.016)$    & $0.109 (\pm 0.020)$ \\
 R $\rightarrow$ S    & $\bf{0.596} (\pm 0.009)$    & $0.599 (\pm 0.010)$    & $0.600 (\pm 0.009)$    & $\bf{0.067} (\pm 0.010)$    & $0.593 (\pm 0.009)$    & $0.599 (\pm 0.010)$    & $0.599 (\pm 0.010)$    & $0.066 (\pm 0.010)$ \\
\midrule
 Avg.      & $0.527 (\pm 0.009)$    & $0.528 (\pm 0.011)$    & $0.528 (\pm 0.012)$    & $0.062 (\pm 0.013)$    & $0.527 (\pm 0.007)$    & $0.528 (\pm 0.011)$    & $0.528 (\pm 0.011)$    & $0.062 (\pm 0.013)$ \\
 \bottomrule
\end{tabular}
}
\caption{Mean and standard deviation (after $\pm$) of target classification accuracy on MiniDomainNet for the aggregation method in objective~\eqref{eq:aggregation} with models computed by MMDA~\citep{Rahman2020}.} \label{tab:mmda_mdn}
\end{table*}

\begin{table*}
\scalebox{0.725}{
\begin{tabular}{ l | c c c c | c c c c}
\toprule
  & \multicolumn{4}{c|}{\textbf{No Iteration}} & \multicolumn{4}{c}{\textbf{Iteration}}\\
\textbf{Scenario} & KuLSIF & Exp & LR & SQ & KuLSIF & Exp & LR & SQ \\
\midrule
 B $\rightarrow$ D    & $0.794 (\pm 0.005)$    & $0.604 (\pm 0.156)$    & $0.786 (\pm 0.006)$    & $\bf{0.216} (\pm 0.010)$    & $0.794 (\pm 0.005)$    & ${\bf 0.605 (\pm 0.159})$    & ${\bf 0.790 (\pm 0.007})$    & $0.208 (\pm 0.007)$ \\
 B $\rightarrow$ E    & $0.760 (\pm 0.012)$    & $0.504 (\pm 0.006)$    & $0.756 (\pm 0.011)$    & $0.239 (\pm 0.012)$    & ${\bf 0.761 (\pm 0.012})$    & $0.504 (\pm 0.006)$    & ${\bf 0.761 (\pm 0.012})$    & ${\bf 0.244 (\pm 0.011})$ \\
 B $\rightarrow$ K    & $0.782 (\pm 0.018)$    & $0.491 (\pm 0.004)$    & $0.781 (\pm 0.017)$    & $0.217 (\pm 0.016)$    & $0.782 (\pm 0.018)$    & $0.491 (\pm 0.004)$    & $0.781 (\pm 0.016)$    & ${\bf 0.219 (\pm 0.016})$ \\
 D $\rightarrow$ B    & $0.794 (\pm 0.010)$    & $\bf{0.694} (\pm 0.177)$    & $0.794 (\pm 0.005)$    & $0.204 (\pm 0.006)$    & ${\bf 0.797 (\pm 0.009})$    & $0.693 (\pm 0.176)$    & ${\bf 0.795 (\pm 0.005})$    & ${\bf 0.208 (\pm 0.006})$ \\
 D $\rightarrow$ E    & $0.783 (\pm 0.006)$    & $0.498 (\pm 0.008)$    & $0.776 (\pm 0.005)$    & $0.222 (\pm 0.009)$    & ${\bf 0.785 (\pm 0.005})$    & $0.498 (\pm 0.008)$    & ${\bf 0.782 (\pm 0.005})$    & $0.222 (\pm 0.009)$ \\
 D $\rightarrow$ K    & $0.789 (\pm 0.013)$    & $0.499 (\pm 0.005)$    & $0.789 (\pm 0.011)$    & $0.210 (\pm 0.011)$    & $0.789 (\pm 0.013)$    & $0.499 (\pm 0.005)$    & ${\bf 0.791 (\pm 0.012})$    & ${\bf 0.212 (\pm 0.013})$ \\
\bottomrule
\end{tabular}
}
\caption{Mean and standard deviation (after $\pm$) of target classification accuracy on Amazon Reviews for the aggregation method in objective~\eqref{eq:aggregation} with models computed by AdvSKM~\citep{Liu2021kernel}.}
\label{tab:amazon_reviews_first}
\end{table*}

\begin{table*}
\scalebox{0.725}{
\begin{tabular}{ l | c c c c | c c c c}
\toprule
  & \multicolumn{4}{c|}{\textbf{No Iteration}} & \multicolumn{4}{c}{\textbf{Iteration}}\\
\textbf{Scenario} & KuLSIF & Exp & LR & SQ & KuLSIF & Exp & LR & SQ \\
\midrule
 B $\rightarrow$ D    & $0.802 (\pm 0.009)$    & ${\bf 0.606 (\pm 0.159})$    & $0.800 (\pm 0.011)$    & ${\bf 0.588 (\pm 0.338})$    & $0.802 (\pm 0.010)$    & $0.605 (\pm 0.158)$    & $0.800 (\pm 0.009)$    & $0.199 (\pm 0.009)$ \\
 B $\rightarrow$ E    & $0.779 (\pm 0.009)$    & ${\bf 0.596 (\pm 0.164})$    & $0.777 (\pm 0.007)$    & ${\bf 0.225 (\pm 0.018})$    & ${\bf 0.782 (\pm 0.009})$    & $0.595 (\pm 0.163)$    & ${\bf 0.780 (\pm 0.009})$    & $0.224 (\pm 0.011)$ \\
 B $\rightarrow$ K    & $0.797 (\pm 0.007)$    & $0.491 (\pm 0.004)$    & $0.789 (\pm 0.018)$    & $0.206 (\pm 0.012)$    & ${\bf 0.798 (\pm 0.011})$    & $0.491 (\pm 0.004)$    & ${\bf 0.795 (\pm 0.012})$    & $0.206 (\pm 0.011)$ \\
 D $\rightarrow$ B    & $0.797 (\pm 0.008)$    & ${\bf 0.693 (\pm 0.176})$    & ${\bf 0.796 (\pm 0.007})$    & ${\bf 0.210 (\pm 0.009})$    & ${\bf 0.798 (\pm 0.006})$    & $0.692 (\pm 0.176)$    & $0.794 (\pm 0.005)$    & $0.206 (\pm 0.005)$ \\
 D $\rightarrow$ E    & $0.798 (\pm 0.008)$    & $0.498 (\pm 0.008)$    & $0.800 (\pm 0.005)$    & $0.202 (\pm 0.011)$    & ${\bf 0.805 (\pm 0.009})$    & $0.498 (\pm 0.008)$    & ${\bf 0.802 (\pm 0.005})$    & ${\bf 0.203 (\pm 0.007})$ \\
 D $\rightarrow$ K    & $0.804 (\pm 0.015)$    & $0.499 (\pm 0.005)$    & $0.802 (\pm 0.016)$    & $0.207 (\pm 0.023)$    & ${\bf 0.805 (\pm 0.014})$    & $0.499 (\pm 0.005)$    & ${\bf 0.804 (\pm 0.013})$    & ${\bf 0.209 (\pm 0.017})$ \\
\bottomrule
\end{tabular}
}
\caption{Mean and standard deviation (after $\pm$) of target classification accuracy on Amazon Reviews for the aggregation method in objective~\eqref{eq:aggregation} with models computed by CDAN~\citep{Long2018conditional}.}
\end{table*}

\begin{table*}
\scalebox{0.725}{
\begin{tabular}{ l | c c c c | c c c c}
\toprule
  & \multicolumn{4}{c|}{\textbf{No Iteration}} & \multicolumn{4}{c}{\textbf{Iteration}}\\
\textbf{Scenario} & KuLSIF & Exp & LR & SQ & KuLSIF & Exp & LR & SQ \\
\midrule
  B $\rightarrow$ D    & ${\bf 0.799 (\pm 0.005})$    & $0.603 (\pm 0.171)$    & $0.794 (\pm 0.009)$    & ${\bf 0.205 (\pm 0.005})$    & $0.796 (\pm 0.009)$    & ${\bf 0.606 (\pm 0.177})$    & ${\bf 0.796 (\pm 0.005})$    & $0.203 (\pm 0.005)$ \\
 B $\rightarrow$ E    & $0.779 (\pm 0.014)$    & ${\bf 0.584 (\pm 0.134})$    & ${\bf 0.761 (\pm 0.021})$    & $0.218 (\pm 0.010)$    & ${\bf 0.788 (\pm 0.005})$    & $0.575 (\pm 0.117)$    & $0.758 (\pm 0.050)$    & ${\bf 0.222 (\pm 0.013})$ \\
 B $\rightarrow$ K    & $0.781 (\pm 0.029)$    & $0.589 (\pm 0.171)$    & $0.793 (\pm 0.011)$    & ${\bf 0.205 (\pm 0.013})$    & ${\bf 0.800 (\pm 0.007})$    & $0.589 (\pm 0.170)$    & ${\bf 0.798 (\pm 0.009})$    & $0.203 (\pm 0.008)$ \\
 D $\rightarrow$ B    & $0.798 (\pm 0.012)$    & ${\bf 0.701 (\pm 0.175})$    & $0.802 (\pm 0.007)$    & ${\bf 0.197 (\pm 0.007})$    & ${\bf 0.803 (\pm 0.004})$    & $0.696 (\pm 0.171)$    & ${\bf 0.804 (\pm 0.009})$    & $0.196 (\pm 0.007)$ \\
 D $\rightarrow$ E    & $0.795 (\pm 0.009)$    & $0.498 (\pm 0.008)$    & $0.802 (\pm 0.010)$    & $0.198 (\pm 0.006)$    & ${\bf 0.811 (\pm 0.007})$    & $0.498 (\pm 0.008)$    & ${\bf 0.805 (\pm 0.007})$    & ${\bf 0.288 (\pm 0.156})$ \\
 D $\rightarrow$ K    & $0.799 (\pm 0.018)$    & $0.499 (\pm 0.005)$    & ${\bf 0.804 (\pm 0.015})$    & $0.196 (\pm 0.010)$    & ${\bf 0.810 (\pm 0.012})$    & $0.499 (\pm 0.005)$    & $0.803 (\pm 0.010)$    & ${\bf 0.202 (\pm 0.015})$ \\
\bottomrule
\end{tabular}
}
\caption{Mean and standard deviation (after $\pm$) of target classification accuracy on Amazon Reviews for the aggregation method in objective~\eqref{eq:aggregation} with models computed by CMD~\citep{zellinger2017central}.}
\end{table*}

\begin{table*}
\scalebox{0.725}{
\begin{tabular}{ l | c c c c | c c c c}
\toprule
  & \multicolumn{4}{c|}{\textbf{No Iteration}} & \multicolumn{4}{c}{\textbf{Iteration}}\\
\textbf{Scenario} & KuLSIF & Exp & LR & SQ & KuLSIF & Exp & LR & SQ \\
\midrule
 B $\rightarrow$ D    & $0.804 (\pm 0.008)$    & $0.609 (\pm 0.165)$    & $0.806 (\pm 0.007)$    & ${\bf 0.205 (\pm 0.016})$    & $0.804 (\pm 0.003)$    & $0.609 (\pm 0.165)$    & $0.806 (\pm 0.006)$    & $0.195 (\pm 0.007)$ \\
 B $\rightarrow$ E    & $0.785 (\pm 0.012)$    & $0.504 (\pm 0.006)$    & $0.782 (\pm 0.010)$    & ${\bf 0.239 (\pm 0.032})$    & ${\bf 0.791 (\pm 0.005})$    & $0.504 (\pm 0.006)$    & ${\bf 0.787 (\pm 0.006})$    & $0.213 (\pm 0.008)$ \\
 B $\rightarrow$ K    & ${\bf 0.807 (\pm 0.009})$    & $0.491 (\pm 0.004)$    & ${\bf 0.807 (\pm 0.007})$    & $0.196 (\pm 0.008)$    & $0.806 (\pm 0.012)$    & $0.491 (\pm 0.004)$    & $0.806 (\pm 0.008)$    & ${\bf 0.197 (\pm 0.008})$ \\
 D $\rightarrow$ B    & $0.801 (\pm 0.005)$    & $0.594 (\pm 0.175)$    & $0.802 (\pm 0.005)$    & $0.197 (\pm 0.007)$    & ${\bf 0.805 (\pm 0.005})$    & ${\bf 0.595 (\pm 0.177})$    & ${\bf 0.805 (\pm 0.006})$    & ${\bf 0.198 (\pm 0.006})$ \\
 D $\rightarrow$ E    & $0.806 (\pm 0.011)$    & $0.498 (\pm 0.008)$    & $0.804 (\pm 0.015)$    & $0.191 (\pm 0.009)$    & ${\bf 0.811 (\pm 0.008})$    & $0.498 (\pm 0.008)$    & ${\bf 0.809 (\pm 0.009})$    & ${\bf 0.206 (\pm 0.029})$ \\
 D $\rightarrow$ K    & $0.819 (\pm 0.012)$    & $0.499 (\pm 0.005)$    & $0.819 (\pm 0.016)$    & $0.179 (\pm 0.011)$    & ${\bf 0.822 (\pm 0.009})$    & $0.499 (\pm 0.005)$    & ${\bf 0.822 (\pm 0.011})$    & $0.179 (\pm 0.012)$ \\
\bottomrule
\end{tabular}
}
\caption{Mean and standard deviation (after $\pm$) of target classification accuracy on Amazon Reviews for the aggregation method in objective~\eqref{eq:aggregation} with models computed by CoDATS~\citep{Wilson2020adaptation}.}
\end{table*}

\begin{table*}
\scalebox{0.725}{
\begin{tabular}{ l | c c c c | c c c c}
\toprule
  & \multicolumn{4}{c|}{\textbf{No Iteration}} & \multicolumn{4}{c}{\textbf{Iteration}}\\
\textbf{Scenario} & KuLSIF & Exp & LR & SQ & KuLSIF & Exp & LR & SQ \\
\midrule
  B $\rightarrow$ D    & $0.802 (\pm 0.006)$    & ${\bf 0.604 (\pm 0.173})$    & $0.802 (\pm 0.010)$    & ${\bf 0.198 (\pm 0.009})$    & ${\bf 0.804 (\pm 0.011})$    & $0.603 (\pm 0.172)$    & $0.802 (\pm 0.009)$    & $0.197 (\pm 0.005)$ \\
 B $\rightarrow$ E    & $0.784 (\pm 0.012)$    & ${\bf 0.594 (\pm 0.160})$    & $0.763 (\pm 0.013)$    & $0.214 (\pm 0.010)$    & ${\bf 0.791 (\pm 0.005})$    & $0.586 (\pm 0.147)$    & ${\bf 0.782 (\pm 0.017})$    & ${\bf 0.216 (\pm 0.014})$ \\
 B $\rightarrow$ K    & $0.801 (\pm 0.009)$    & $0.491 (\pm 0.004)$    & $0.803 (\pm 0.008)$    & $0.203 (\pm 0.016)$    & ${\bf 0.807 (\pm 0.008})$    & $0.491 (\pm 0.004)$    & ${\bf 0.805 (\pm 0.008})$    & ${\bf 0.378 (\pm 0.313})$ \\
 D $\rightarrow$ B    & $0.804 (\pm 0.003)$    & ${\bf 0.596 (\pm 0.178})$    & $0.799 (\pm 0.003)$    & ${\bf 0.202 (\pm 0.007})$    & ${\bf 0.806 (\pm 0.002})$    & $0.592 (\pm 0.171)$    & ${\bf 0.805 (\pm 0.006})$    & $0.198 (\pm 0.004)$ \\
 D $\rightarrow$ E    & $0.808 (\pm 0.007)$    & $0.498 (\pm 0.008)$    & $0.811 (\pm 0.009)$    & $0.187 (\pm 0.006)$    & ${\bf 0.814 (\pm 0.006})$    & $0.498 (\pm 0.008)$    & ${\bf 0.813 (\pm 0.006})$    & $0.187 (\pm 0.006)$ \\
 D $\rightarrow$ K    & $0.814 (\pm 0.013)$    & $0.499 (\pm 0.005)$    & $0.814 (\pm 0.018)$    & $0.189 (\pm 0.007)$    & ${\bf 0.818 (\pm 0.010})$    & $0.499 (\pm 0.005)$    & ${\bf 0.820 (\pm 0.009})$    & ${\bf 0.362 (\pm 0.319})$ \\
\bottomrule
\end{tabular}
}
\caption{Mean and standard deviation (after $\pm$) of target classification accuracy on Amazon Reviews for the aggregation method in objective~\eqref{eq:aggregation} with models computed by DANN~\citep{ganin2016domain}.}
\end{table*}

\begin{table*}
\scalebox{0.725}{
\begin{tabular}{ l | c c c c | c c c c}
\toprule
  & \multicolumn{4}{c|}{\textbf{No Iteration}} & \multicolumn{4}{c}{\textbf{Iteration}}\\
\textbf{Scenario} & KuLSIF & Exp & LR & SQ & KuLSIF & Exp & LR & SQ \\
\midrule
 B $\rightarrow$ E    & $0.763 (\pm 0.014)$    & $0.584 (\pm 0.133)$    & $0.760 (\pm 0.016)$    & $0.237 (\pm 0.015)$    & $0.763 (\pm 0.014)$    & ${\bf 0.585 (\pm 0.134})$    & $0.760 (\pm 0.021)$    & ${\bf 0.239 (\pm 0.018})$ \\
 B $\rightarrow$ K    & $0.783 (\pm 0.014)$    & ${\bf 0.579 (\pm 0.153})$    & ${\bf 0.779 (\pm 0.022})$    & $0.217 (\pm 0.016)$    & ${\bf 0.784 (\pm 0.014})$    & $0.578 (\pm 0.152)$    & $0.778 (\pm 0.023)$    & $0.217 (\pm 0.015)$ \\
 D $\rightarrow$ B    & $0.795 (\pm 0.006)$    & ${\bf 0.596 (\pm 0.176})$    & $0.791 (\pm 0.008)$    & ${\bf 0.226 (\pm 0.022})$    & $0.795 (\pm 0.005)$    & $0.595 (\pm 0.173)$    & ${\bf 0.792 (\pm 0.008})$    & $0.208 (\pm 0.008)$ \\
 D $\rightarrow$ E    & $0.784 (\pm 0.004)$    & $0.498 (\pm 0.008)$    & $0.778 (\pm 0.008)$    & ${\bf 0.223 (\pm 0.016})$    & ${\bf 0.785 (\pm 0.002})$    & $0.498 (\pm 0.008)$    & ${\bf 0.784 (\pm 0.002})$    & $0.219 (\pm 0.006)$ \\
 D $\rightarrow$ K    & $0.789 (\pm 0.021)$    & $0.499 (\pm 0.005)$    & $0.790 (\pm 0.015)$    & $0.207 (\pm 0.016)$    & ${\bf 0.793 (\pm 0.016})$    & $0.499 (\pm 0.005)$    & ${\bf 0.793 (\pm 0.017})$    & ${\bf 0.208 (\pm 0.015})$ \\
 E $\rightarrow$ B    & $0.705 (\pm 0.024)$    & $0.502 (\pm 0.014)$    & $0.699 (\pm 0.024)$    & $0.299 (\pm 0.020)$    & $0.705 (\pm 0.026)$    & ${\bf 0.568 (\pm 0.100})$    & ${\bf 0.703 (\pm 0.024})$    & $0.299 (\pm 0.023)$ \\
\bottomrule
\end{tabular}
}
\caption{Mean and standard deviation (after $\pm$) of target classification accuracy on Amazon Reviews for the aggregation method in objective~\eqref{eq:aggregation} with models computed by DDC~\citep{tzeng2014deep}.}
\end{table*}

\begin{table*}
\scalebox{0.725}{
\begin{tabular}{ l | c c c c | c c c c}
\toprule
  & \multicolumn{4}{c|}{\textbf{No Iteration}} & \multicolumn{4}{c}{\textbf{Iteration}}\\
\textbf{Scenario} & KuLSIF & Exp & LR & SQ & KuLSIF & Exp & LR & SQ \\
\midrule
 B $\rightarrow$ D    & $0.801 (\pm 0.010)$    & ${\bf 0.702 (\pm 0.175})$    & $0.799 (\pm 0.011)$    & ${\bf 0.203 (\pm 0.008})$    & $0.801 (\pm 0.010)$    & $0.701 (\pm 0.174)$    & $0.799 (\pm 0.009)$    & $0.202 (\pm 0.009)$ \\
 B $\rightarrow$ E    & $0.773 (\pm 0.008)$    & ${\bf 0.587 (\pm 0.137})$    & $0.763 (\pm 0.012)$    & $0.234 (\pm 0.013)$    & $0.773 (\pm 0.008)$    & $0.586 (\pm 0.136)$    & $0.763 (\pm 0.019)$    & ${\bf 0.236 (\pm 0.017})$ \\
 B $\rightarrow$ K    & $0.785 (\pm 0.017)$    & ${\bf 0.580 (\pm 0.155})$    & $0.781 (\pm 0.018)$    & ${\bf 0.218 (\pm 0.016})$    & $0.785 (\pm 0.016)$    & $0.579 (\pm 0.153)$    & $0.781 (\pm 0.020)$    & $0.216 (\pm 0.016)$ \\
 D $\rightarrow$ B    & $0.797 (\pm 0.004)$    & ${\bf 0.594 (\pm 0.172})$    & $0.794 (\pm 0.003)$    & ${\bf 0.226 (\pm 0.024})$    & ${\bf 0.798 (\pm 0.003})$    & $0.591 (\pm 0.167)$    & ${\bf 0.796 (\pm 0.003})$    & $0.206 (\pm 0.003)$ \\
 D $\rightarrow$ E    & $0.785 (\pm 0.003)$    & $0.498 (\pm 0.008)$    & $0.779 (\pm 0.011)$    & ${\bf 0.227 (\pm 0.024})$    & ${\bf 0.786 (\pm 0.005})$    & $0.498 (\pm 0.008)$    & ${\bf 0.784 (\pm 0.004})$    & $0.217 (\pm 0.004)$ \\
 D $\rightarrow$ K    & $0.792 (\pm 0.014)$    & $0.499 (\pm 0.005)$    & $0.791 (\pm 0.016)$    & $0.206 (\pm 0.014)$    & ${\bf 0.796 (\pm 0.014})$    & $0.499 (\pm 0.005)$    & ${\bf 0.796 (\pm 0.013})$    & $0.206 (\pm 0.013)$ \\
\bottomrule
\end{tabular}
}
\caption{Mean and standard deviation (after $\pm$) of target classification accuracy on Amazon Reviews for the aggregation method in objective~\eqref{eq:aggregation} with models computed by Deep Coral~\citep{Sun2017correlation}.}
\end{table*}

\begin{table*}
\scalebox{0.725}{
\begin{tabular}{ l | c c c c | c c c c}
\toprule
  & \multicolumn{4}{c|}{\textbf{No Iteration}} & \multicolumn{4}{c}{\textbf{Iteration}}\\
\textbf{Scenario} & KuLSIF & Exp & LR & SQ & KuLSIF & Exp & LR & SQ \\
\midrule
 B $\rightarrow$ D    & ${\bf 0.812 (\pm 0.004})$    & $0.509 (\pm 0.010)$    & ${\bf 0.811 (\pm 0.005})$    & ${\bf 0.195 (\pm 0.004})$    & $0.810 (\pm 0.004)$    & $0.509 (\pm 0.010)$    & $0.808 (\pm 0.006)$    & $0.193 (\pm 0.006)$ \\
 B $\rightarrow$ E    & ${\bf 0.786 (\pm 0.007})$    & ${\bf 0.775 (\pm 0.009})$    & $0.779 (\pm 0.008)$    & $0.224 (\pm 0.004)$    & $0.779 (\pm 0.006)$    & $0.768 (\pm 0.009)$    & $0.779 (\pm 0.007)$    & ${\bf 0.396 (\pm 0.274})$ \\
 B $\rightarrow$ K    & ${\bf 0.809 (\pm 0.006})$    & $0.491 (\pm 0.004)$    & $0.797 (\pm 0.007)$    & $0.198 (\pm 0.008)$    & $0.800 (\pm 0.006)$    & $0.491 (\pm 0.004)$    & ${\bf 0.805 (\pm 0.006})$    & ${\bf 0.205 (\pm 0.021})$ \\
 D $\rightarrow$ B    & $0.806 (\pm 0.009)$    & ${\bf 0.707 (\pm 0.183})$    & $0.811 (\pm 0.009)$    & ${\bf 0.196 (\pm 0.012})$    & ${\bf 0.817 (\pm 0.011})$    & $0.696 (\pm 0.173)$    & ${\bf 0.813 (\pm 0.012})$    & $0.191 (\pm 0.011)$ \\
 D $\rightarrow$ E    & ${\bf 0.804 (\pm 0.012})$    & $0.498 (\pm 0.008)$    & $0.802 (\pm 0.022)$    & $0.199 (\pm 0.011)$    & $0.800 (\pm 0.005)$    & $0.498 (\pm 0.008)$    & ${\bf 0.805 (\pm 0.008})$    & ${\bf 0.209 (\pm 0.008})$ \\
 D $\rightarrow$ K    & $0.830 (\pm 0.012)$    & $0.499 (\pm 0.005)$    & $0.823 (\pm 0.018)$    & $0.173 (\pm 0.007)$    & $0.830 (\pm 0.010)$    & $0.499 (\pm 0.005)$    & ${\bf 0.830 (\pm 0.008})$    & ${\bf 0.174 (\pm 0.011})$ \\
\bottomrule
\end{tabular}
}
\caption{Mean and standard deviation (after $\pm$) of target classification accuracy on Amazon Reviews for the aggregation method in objective~\eqref{eq:aggregation} with models computed by DIRT~\citep{Shu2018dirt}.}
\end{table*}

\begin{table*}
\scalebox{0.725}{
\begin{tabular}{ l | c c c c | c c c c}
\toprule
  & \multicolumn{4}{c|}{\textbf{No Iteration}} & \multicolumn{4}{c}{\textbf{Iteration}}\\
\textbf{Scenario} & KuLSIF & Exp & LR & SQ & KuLSIF & Exp & LR & SQ \\
\midrule
 B $\rightarrow$ D    & $0.803 (\pm 0.015)$    & ${\bf 0.707 (\pm 0.179})$    & $0.803 (\pm 0.011)$    & ${\bf 0.206 (\pm 0.024})$    & ${\bf 0.808 (\pm 0.011})$    & $0.704 (\pm 0.176)$    & ${\bf 0.806 (\pm 0.009})$    & $0.193 (\pm 0.010)$ \\
 B $\rightarrow$ E    & $0.788 (\pm 0.004)$    & $0.589 (\pm 0.142)$    & ${\bf 0.786 (\pm 0.016})$    & $0.206 (\pm 0.006)$    & ${\bf 0.799 (\pm 0.004})$    & ${\bf 0.592 (\pm 0.147})$    & $0.779 (\pm 0.030)$    & ${\bf 0.207 (\pm 0.006})$ \\
 B $\rightarrow$ K    & $0.804 (\pm 0.012)$    & ${\bf 0.593 (\pm 0.178})$    & $0.807 (\pm 0.015)$    & $0.193 (\pm 0.009)$    & ${\bf 0.809 (\pm 0.010})$    & $0.588 (\pm 0.169)$    & $0.807 (\pm 0.010)$    & $0.193 (\pm 0.009)$ \\
 D $\rightarrow$ B    & $0.801 (\pm 0.005)$    & ${\bf 0.596 (\pm 0.176})$    & ${\bf 0.802 (\pm 0.005})$    & ${\bf 0.202 (\pm 0.013})$    & ${\bf 0.802 (\pm 0.004})$    & $0.595 (\pm 0.175)$    & $0.799 (\pm 0.006)$    & $0.199 (\pm 0.007)$ \\
 D $\rightarrow$ E    & $0.803 (\pm 0.008)$    & $0.498 (\pm 0.008)$    & $0.801 (\pm 0.009)$    & $0.192 (\pm 0.006)$    & ${\bf 0.813 (\pm 0.009})$    & $0.498 (\pm 0.008)$    & ${\bf 0.808 (\pm 0.004})$    & ${\bf 0.196 (\pm 0.015})$ \\
 D $\rightarrow$ K    & $0.803 (\pm 0.029)$    & $0.499 (\pm 0.005)$    & $0.805 (\pm 0.017)$    & $0.184 (\pm 0.012)$    & ${\bf 0.819 (\pm 0.010})$    & $0.499 (\pm 0.005)$    & ${\bf 0.816 (\pm 0.013})$    & ${\bf 0.185 (\pm 0.012})$ \\
\bottomrule
\end{tabular}
}
\caption{Mean and standard deviation (after $\pm$) of target classification accuracy on Amazon Reviews for the aggregation method in objective~\eqref{eq:aggregation} with models computed by DSAN~\citep{Zhu2021subdomain}.}
\end{table*}

\begin{table*}
\scalebox{0.725}{
\begin{tabular}{ l | c c c c | c c c c}
\toprule
  & \multicolumn{4}{c|}{\textbf{No Iteration}} & \multicolumn{4}{c}{\textbf{Iteration}}\\
\textbf{Scenario} & KuLSIF & Exp & LR & SQ & KuLSIF & Exp & LR & SQ \\
\midrule
  B $\rightarrow$ D    & $0.794 (\pm 0.007)$    & ${\bf 0.700 (\pm 0.172})$    & $0.792 (\pm 0.010)$    & ${\bf 0.217 (\pm 0.016})$    & $0.794 (\pm 0.007)$    & $0.698 (\pm 0.171)$    & ${\bf 0.793 (\pm 0.007})$    & $0.208 (\pm 0.006)$ \\
 B $\rightarrow$ E    & $0.761 (\pm 0.012)$    & $0.583 (\pm 0.131)$    & $0.756 (\pm 0.018)$    & $0.239 (\pm 0.011)$    & $0.761 (\pm 0.012)$    & $0.583 (\pm 0.131)$    & $0.756 (\pm 0.020)$    & ${\bf 0.243 (\pm 0.018})$ \\
 B $\rightarrow$ K    & ${\bf 0.781 (\pm 0.017})$    & ${\bf 0.580 (\pm 0.155})$    & $0.776 (\pm 0.022)$    & ${\bf 0.222 (\pm 0.015})$    & $0.780 (\pm 0.019)$    & $0.576 (\pm 0.149)$    & $0.776 (\pm 0.022)$    & $0.220 (\pm 0.013)$ \\
 D $\rightarrow$ B    & $0.790 (\pm 0.002)$    & ${\bf 0.594 (\pm 0.172})$    & $0.789 (\pm 0.003)$    & ${\bf 0.217 (\pm 0.015})$    & ${\bf 0.791 (\pm 0.004})$    & $0.593 (\pm 0.171)$    & $0.789 (\pm 0.004)$    & $0.210 (\pm 0.004)$ \\
 D $\rightarrow$ E    & $0.781 (\pm 0.004)$    & $0.498 (\pm 0.008)$    & $0.772 (\pm 0.004)$    & $0.221 (\pm 0.003)$    & $0.781 (\pm 0.003)$    & $0.498 (\pm 0.008)$    & ${\bf 0.779 (\pm 0.003})$    & ${\bf 0.224 (\pm 0.006})$ \\
 D $\rightarrow$ K    & $0.789 (\pm 0.015)$    & $0.499 (\pm 0.005)$    & $0.788 (\pm 0.015)$    & ${\bf 0.212 (\pm 0.014})$    & ${\bf 0.790 (\pm 0.016})$    & $0.499 (\pm 0.005)$    & $0.788 (\pm 0.017)$    & $0.211 (\pm 0.015)$ \\
\bottomrule
\end{tabular}
}
\caption{Mean and standard deviation (after $\pm$) of target classification accuracy on Amazon Reviews for the aggregation method in objective~\eqref{eq:aggregation} with models computed by HoMM~\citep{Chen2020moment}.}
\end{table*}

\begin{table*}
\scalebox{0.725}{
\begin{tabular}{ l | c c c c | c c c c}
\toprule
  & \multicolumn{4}{c|}{\textbf{No Iteration}} & \multicolumn{4}{c}{\textbf{Iteration}}\\
\textbf{Scenario} & KuLSIF & Exp & LR & SQ & KuLSIF & Exp & LR & SQ \\
\midrule
  B $\rightarrow$ D    & $0.795 (\pm 0.003)$    & ${\bf 0.699 (\pm 0.172})$    & $0.796 (\pm 0.003)$    & ${\bf 0.206 (\pm 0.007})$    & ${\bf 0.797 (\pm 0.003})$    & $0.695 (\pm 0.168)$    & ${\bf 0.798 (\pm 0.003})$    & $0.203 (\pm 0.004)$ \\
 B $\rightarrow$ E    & ${\bf 0.776 (\pm 0.018})$    & ${\bf 0.590 (\pm 0.143})$    & $0.766 (\pm 0.016)$    & ${\bf 0.228 (\pm 0.020})$    & $0.773 (\pm 0.022)$    & $0.580 (\pm 0.126)$    & ${\bf 0.772 (\pm 0.020})$    & $0.227 (\pm 0.015)$ \\
 B $\rightarrow$ K    & $0.793 (\pm 0.019)$    & $0.587 (\pm 0.167)$    & $0.794 (\pm 0.021)$    & ${\bf 0.209 (\pm 0.024})$    & ${\bf 0.800 (\pm 0.013})$    & $0.587 (\pm 0.168)$    & ${\bf 0.798 (\pm 0.013})$    & $0.203 (\pm 0.012)$ \\
 D $\rightarrow$ B    & ${\bf 0.799 (\pm 0.007})$    & ${\bf 0.696 (\pm 0.170})$    & $0.793 (\pm 0.006)$    & ${\bf 0.205 (\pm 0.013})$    & $0.798 (\pm 0.004)$    & $0.694 (\pm 0.168)$    & ${\bf 0.796 (\pm 0.006})$    & $0.203 (\pm 0.007)$ \\
 D $\rightarrow$ E    & $0.791 (\pm 0.008)$    & $0.498 (\pm 0.008)$    & $0.790 (\pm 0.011)$    & $0.208 (\pm 0.006)$    & ${\bf 0.793 (\pm 0.005})$    & $0.498 (\pm 0.008)$    & ${\bf 0.792 (\pm 0.005})$    & ${\bf 0.224 (\pm 0.029})$ \\
 D $\rightarrow$ K    & $0.802 (\pm 0.010)$    & $0.499 (\pm 0.005)$    & $0.801 (\pm 0.009)$    & $0.194 (\pm 0.009)$    & ${\bf 0.808 (\pm 0.002})$    & $0.499 (\pm 0.005)$    & ${\bf 0.806 (\pm 0.008})$    & ${\bf 0.199 (\pm 0.010})$ \\
\bottomrule
\end{tabular}
}
\caption{Mean and standard deviation (after $\pm$) of target classification accuracy on Amazon Reviews for the aggregation method in objective~\eqref{eq:aggregation} with models computed by MMDA~\citep{Rahman2020}.}
\label{tab:amazon_reviews_last}
\end{table*}

\begin{table*}
\scalebox{0.75}{
\begin{tabular}{ l | c c c c | c c c c}
\toprule
  & \multicolumn{4}{c|}{\textbf{No Iteration}} & \multicolumn{4}{c}{\textbf{Iteration}}\\
\textbf{Scenario} & KuLSIF & Exp & LR & SQ & KuLSIF & Exp & LR & SQ \\
\midrule
  0 $\rightarrow$ 6    & $0.718 (\pm 0.009)$    & ${\bf 0.714 (\pm 0.010})$    & $0.705 (\pm 0.020)$    & $0.012 (\pm 0.019)$    & $0.718 (\pm 0.009)$    & $0.705 (\pm 0.020)$    & $0.705 (\pm 0.020)$    & $0.012 (\pm 0.019)$ \\
 1 $\rightarrow$ 6    & $0.857 (\pm 0.006)$    & $0.835 (\pm 0.026)$    & $0.836 (\pm 0.025)$    & $0.000 (\pm 0.001)$    & $0.857 (\pm 0.006)$    & ${\bf 0.836 (\pm 0.025})$    & $0.836 (\pm 0.025)$    & $0.000 (\pm 0.001)$ \\
 2 $\rightarrow$ 7    & $0.490 (\pm 0.014)$    & $0.497 (\pm 0.010)$    & $0.502 (\pm 0.024)$    & $0.004 (\pm 0.004)$    & $0.490 (\pm 0.014)$    & ${\bf 0.502 (\pm 0.024})$    & $0.502 (\pm 0.024)$    & $0.004 (\pm 0.004)$ \\
 3 $\rightarrow$ 8    & $0.810 (\pm 0.002)$    & $0.806 (\pm 0.005)$    & $0.807 (\pm 0.005)$    & $0.003 (\pm 0.004)$    & $0.810 (\pm 0.002)$    & ${\bf 0.807 (\pm 0.005})$    & $0.807 (\pm 0.005)$    & $0.003 (\pm 0.004)$ \\
 4 $\rightarrow$ 5    & $0.884 (\pm 0.005)$    & $0.875 (\pm 0.014)$    & $0.878 (\pm 0.017)$    & $0.000 (\pm 0.001)$    & $0.884 (\pm 0.005)$    & ${\bf 0.878 (\pm 0.017})$    & $0.878 (\pm 0.017)$    & $0.000 (\pm 0.001)$ \\
\midrule
 Avg.      & $0.752 (\pm 0.007)$    & $0.746 (\pm 0.013)$    & $0.746 (\pm 0.018)$    & $0.004 (\pm 0.006)$    & $0.752 (\pm 0.007)$    & $0.746 (\pm 0.018)$    & $0.746 (\pm 0.018)$    & $0.004 (\pm 0.006)$ \\
 \bottomrule
\end{tabular}
}
\caption{Mean and standard deviation (after $\pm$) of target classification accuracy on HHAR for the aggregation method in objective~\eqref{eq:aggregation} with models computed by AdvSKM~\citep{Liu2021kernel}.}
\label{tab:HHAR_first}
\end{table*}

\begin{table*}
\scalebox{0.75}{
\begin{tabular}{ l | c c c c | c c c c}
\toprule
  & \multicolumn{4}{c|}{\textbf{No Iteration}} & \multicolumn{4}{c}{\textbf{Iteration}}\\
\textbf{Scenario} & KuLSIF & Exp & LR & SQ & KuLSIF & Exp & LR & SQ \\
\midrule
 0 $\rightarrow$ 6    & $0.622 (\pm 0.009)$    & $0.507 (\pm 0.281)$    & $0.666 (\pm 0.014)$    & $0.014 (\pm 0.004)$    & $0.622 (\pm 0.009)$    & ${\bf 0.666 (\pm 0.014})$    & $0.666 (\pm 0.014)$    & $0.014 (\pm 0.004)$ \\
 1 $\rightarrow$ 6    & $0.933 (\pm 0.000)$    & $0.906 (\pm 0.008)$    & $0.908 (\pm 0.006)$    & $0.014 (\pm 0.004)$    & $0.933 (\pm 0.000)$    & ${\bf 0.908 (\pm 0.006})$    & $0.908 (\pm 0.006)$    & $0.014 (\pm 0.004)$ \\
 2 $\rightarrow$ 7    & ${\bf 0.550 (\pm 0.061})$    & $0.547 (\pm 0.061)$    & $0.552 (\pm 0.066)$    & $0.029 (\pm 0.024)$    & $0.548 (\pm 0.063)$    & ${\bf 0.552 (\pm 0.066})$    & $0.552 (\pm 0.066)$    & $0.029 (\pm 0.024)$ \\
 3 $\rightarrow$ 8    & $0.874 (\pm 0.076)$    & ${\bf 0.863 (\pm 0.059})$    & $0.862 (\pm 0.056)$    & $0.000 (\pm 0.000)$    & $0.874 (\pm 0.076)$    & $0.862 (\pm 0.055)$    & $0.862 (\pm 0.056)$    & $0.000 (\pm 0.000)$ \\
 4 $\rightarrow$ 5    & $0.980 (\pm 0.000)$    & $0.707 (\pm 0.423)$    & $0.954 (\pm 0.006)$    & $0.000 (\pm 0.000)$    & $0.980 (\pm 0.000)$    & ${\bf 0.708 (\pm 0.424})$    & $0.954 (\pm 0.006)$    & $0.000 (\pm 0.000)$ \\
\midrule
 Avg.      & ${\bf 0.792 (\pm 0.029})$    & $0.706 (\pm 0.166)$    & $0.788 (\pm 0.029)$    & $0.012 (\pm 0.007)$    & $0.791 (\pm 0.030)$    & ${\bf 0.739 (\pm 0.113})$    & $0.788 (\pm 0.029)$    & $0.012 (\pm 0.007)$ \\
 \bottomrule
\end{tabular}
}
\caption{Mean and standard deviation (after $\pm$) of target classification accuracy on HHAR for the aggregation method in objective~\eqref{eq:aggregation} with models computed by CDAN~\citep{Long2018conditional}.}
\end{table*}

\begin{table*}
\scalebox{0.75}{
\begin{tabular}{ l | c c c c | c c c c}
\toprule
  & \multicolumn{4}{c|}{\textbf{No Iteration}} & \multicolumn{4}{c}{\textbf{Iteration}}\\
\textbf{Scenario} & KuLSIF & Exp & LR & SQ & KuLSIF & Exp & LR & SQ \\
\midrule
 0 $\rightarrow$ 6    & $0.693 (\pm 0.012)$    & ${\bf 0.713 (\pm 0.012})$    & $0.693 (\pm 0.010)$    & $0.068 (\pm 0.018)$    & $0.693 (\pm 0.012)$    & $0.693 (\pm 0.010)$    & $0.693 (\pm 0.010)$    & $0.068 (\pm 0.018)$ \\
 1 $\rightarrow$ 6    & $0.907 (\pm 0.006)$    & $0.885 (\pm 0.018)$    & $0.896 (\pm 0.008)$    & $0.005 (\pm 0.001)$    & $0.907 (\pm 0.006)$    & ${\bf 0.896 (\pm 0.008})$    & $0.896 (\pm 0.008)$    & $0.005 (\pm 0.001)$ \\
 2 $\rightarrow$ 7    & $0.320 (\pm 0.277)$    & $0.503 (\pm 0.027)$    & $0.513 (\pm 0.040)$    & $0.001 (\pm 0.002)$    & ${\bf 0.482 (\pm 0.008})$    & ${\bf 0.513 (\pm 0.040})$    & $0.513 (\pm 0.040)$    & $0.001 (\pm 0.002)$ \\
 3 $\rightarrow$ 8    & $0.811 (\pm 0.011)$    & $0.813 (\pm 0.009)$    & $0.813 (\pm 0.011)$    & $0.000 (\pm 0.000)$    & $0.811 (\pm 0.011)$    & ${\bf 0.814 (\pm 0.011})$    & ${\bf 0.814 (\pm 0.011})$    & $0.000 (\pm 0.000)$ \\
 4 $\rightarrow$ 5    & $0.625 (\pm 0.541)$    & ${\bf 0.937 (\pm 0.011})$    & $0.934 (\pm 0.013)$    & $0.000 (\pm 0.001)$    & ${\bf 0.939 (\pm 0.002})$    & $0.934 (\pm 0.013)$    & $0.934 (\pm 0.013)$    & $0.000 (\pm 0.001)$ \\
\midrule
 Avg.      & $0.671 (\pm 0.170)$    & $0.770 (\pm 0.015)$    & $0.770 (\pm 0.016)$    & $0.015 (\pm 0.004)$    & ${\bf 0.766 (\pm 0.008})$    & $0.770 (\pm 0.016)$    & $0.770 (\pm 0.016)$    & $0.015 (\pm 0.004)$ \\

 \bottomrule
\end{tabular}
}
\caption{Mean and standard deviation (after $\pm$) of target classification accuracy on HHAR for the aggregation method in objective~\eqref{eq:aggregation} with models computed by CMD~\citep{zellinger2017central}.}
\end{table*}

\begin{table*}
\scalebox{0.75}{
\begin{tabular}{ l | c c c c | c c c c}
\toprule
  & \multicolumn{4}{c|}{\textbf{No Iteration}} & \multicolumn{4}{c}{\textbf{Iteration}}\\
\textbf{Scenario} & KuLSIF & Exp & LR & SQ & KuLSIF & Exp & LR & SQ \\
\midrule
0 $\rightarrow$ 6    & $0.596 (\pm 0.048)$    & ${\bf 0.634 (\pm 0.045})$    & $0.631 (\pm 0.043)$    & $0.003 (\pm 0.005)$    & ${\bf 0.601 (\pm 0.029})$    & $0.631 (\pm 0.043)$    & $0.631 (\pm 0.043)$    & $0.003 (\pm 0.005)$ \\
 1 $\rightarrow$ 6    & $0.939 (\pm 0.006)$    & ${\bf 0.908 (\pm 0.011})$    & $0.906 (\pm 0.013)$    & $0.001 (\pm 0.001)$    & $0.939 (\pm 0.006)$    & $0.906 (\pm 0.013)$    & $0.906 (\pm 0.013)$    & $0.001 (\pm 0.001)$ \\
 2 $\rightarrow$ 7    & $0.472 (\pm 0.005)$    & $0.470 (\pm 0.008)$    & $0.476 (\pm 0.008)$    & $0.005 (\pm 0.008)$    & $0.472 (\pm 0.005)$    & ${\bf 0.476 (\pm 0.008})$    & $0.476 (\pm 0.008)$    & $0.005 (\pm 0.008)$ \\
 3 $\rightarrow$ 8    & $0.960 (\pm 0.030)$    & $0.921 (\pm 0.028)$    & $0.934 (\pm 0.030)$    & $0.000 (\pm 0.000)$    & $0.960 (\pm 0.030)$    & ${\bf 0.936 (\pm 0.032})$    & $0.934 (\pm 0.030)$    & $0.000 (\pm 0.000)$ \\
 4 $\rightarrow$ 5    & $0.648 (\pm 0.562)$    & $0.947 (\pm 0.013)$    & $0.947 (\pm 0.012)$    & $0.000 (\pm 0.000)$    & ${\bf 0.973 (\pm 0.004})$    & $0.947 (\pm 0.012)$    & $0.947 (\pm 0.012)$    & $0.000 (\pm 0.000)$ \\
\midrule
 Avg.      & $0.723 (\pm 0.130)$    & $0.776 (\pm 0.021)$    & $0.779 (\pm 0.021)$    & $0.002 (\pm 0.003)$    & ${\bf 0.789 (\pm 0.015})$    & ${\bf 0.779 (\pm 0.022})$    & $0.779 (\pm 0.021)$    & $0.002 (\pm 0.003)$ \\
 \bottomrule
\end{tabular}
}
\caption{Mean and standard deviation (after $\pm$) of target classification accuracy on HHAR for the aggregation method in objective~\eqref{eq:aggregation} with models computed by CoDATS~\citep{Wilson2020adaptation}.}
\end{table*}

\begin{table*}
\scalebox{0.75}{
\begin{tabular}{ l | c c c c | c c c c}
\toprule
  & \multicolumn{4}{c|}{\textbf{No Iteration}} & \multicolumn{4}{c}{\textbf{Iteration}}\\
\textbf{Scenario} & KuLSIF & Exp & LR & SQ & KuLSIF & Exp & LR & SQ \\
\midrule
 0 $\rightarrow$ 6    & $0.604 (\pm 0.036)$    & $0.652 (\pm 0.043)$    & $0.657 (\pm 0.027)$    & $0.000 (\pm 0.000)$    & ${\bf 0.606 (\pm 0.039})$    & ${\bf 0.657 (\pm 0.027})$    & $0.657 (\pm 0.027)$    & $0.000 (\pm 0.000)$ \\
 1 $\rightarrow$ 6    & $0.936 (\pm 0.002)$    & $0.890 (\pm 0.038)$    & $0.906 (\pm 0.010)$    & $0.003 (\pm 0.005)$    & $0.936 (\pm 0.002)$    & ${\bf 0.906 (\pm 0.010})$    & $0.906 (\pm 0.010)$    & $0.003 (\pm 0.005)$ \\
 2 $\rightarrow$ 7    & $0.327 (\pm 0.284)$    & $0.511 (\pm 0.026)$    & $0.533 (\pm 0.045)$    & $0.000 (\pm 0.000)$    & ${\bf 0.493 (\pm 0.004})$    & ${\bf 0.533 (\pm 0.045})$    & $0.533 (\pm 0.045)$    & $0.000 (\pm 0.000)$ \\
 3 $\rightarrow$ 8    & $0.964 (\pm 0.006)$    & $0.919 (\pm 0.007)$    & $0.923 (\pm 0.015)$    & $0.000 (\pm 0.000)$    & $0.964 (\pm 0.006)$    & ${\bf 0.923 (\pm 0.015})$    & $0.923 (\pm 0.015)$    & $0.000 (\pm 0.000)$ \\
 4 $\rightarrow$ 5    & $0.654 (\pm 0.566)$    & $0.951 (\pm 0.005)$    & $0.957 (\pm 0.010)$    & $0.002 (\pm 0.004)$    & ${\bf 0.980 (\pm 0.000})$    & ${\bf 0.957 (\pm 0.010})$    & $0.957 (\pm 0.010)$    & $0.002 (\pm 0.004)$ \\
\midrule
 Avg.      & $0.697 (\pm 0.179)$    & $0.785 (\pm 0.024)$    & $0.795 (\pm 0.021)$    & $0.001 (\pm 0.002)$    & ${\bf 0.796 (\pm 0.010})$    & ${\bf 0.795 (\pm 0.021})$    & $0.795 (\pm 0.021)$    & $0.001 (\pm 0.002)$ \\
 \bottomrule
\end{tabular}
}
\caption{Mean and standard deviation (after $\pm$) of target classification accuracy on HHAR for the aggregation method in objective~\eqref{eq:aggregation} with models computed by DANN~\citep{ganin2016domain}}
\end{table*}

\begin{table*}
\scalebox{0.75}{
\begin{tabular}{ l | c c c c | c c c c}
\toprule
  & \multicolumn{4}{c|}{\textbf{No Iteration}} & \multicolumn{4}{c}{\textbf{Iteration}}\\
\textbf{Scenario} & KuLSIF & Exp & LR & SQ & KuLSIF & Exp & LR & SQ \\
\midrule
 0 $\rightarrow$ 6    & $0.697 (\pm 0.013)$    & $0.349 (\pm 0.286)$    & $0.690 (\pm 0.011)$    & $0.072 (\pm 0.046)$    & $0.697 (\pm 0.013)$    & $0.349 (\pm 0.287)$    & $0.690 (\pm 0.011)$    & $0.072 (\pm 0.046)$ \\
 1 $\rightarrow$ 6    & $0.882 (\pm 0.017)$    & $0.639 (\pm 0.413)$    & $0.873 (\pm 0.014)$    & $0.001 (\pm 0.002)$    & $0.882 (\pm 0.017)$    & $0.639 (\pm 0.413)$    & $0.873 (\pm 0.014)$    & $0.001 (\pm 0.002)$ \\
 2 $\rightarrow$ 7    & $0.439 (\pm 0.041)$    & $0.440 (\pm 0.036)$    & $0.439 (\pm 0.037)$    & $0.005 (\pm 0.007)$    & $0.439 (\pm 0.041)$    & $0.440 (\pm 0.038)$    & $0.439 (\pm 0.037)$    & $0.005 (\pm 0.007)$ \\
 3 $\rightarrow$ 8    & $0.818 (\pm 0.008)$    & $0.621 (\pm 0.342)$    & $0.823 (\pm 0.006)$    & $0.001 (\pm 0.002)$    & $0.818 (\pm 0.008)$    & ${\bf 0.622 (\pm 0.343})$    & $0.823 (\pm 0.006)$    & $0.001 (\pm 0.002)$ \\
 4 $\rightarrow$ 5    & $0.905 (\pm 0.008)$    & $0.219 (\pm 0.000)$    & $0.895 (\pm 0.025)$    & $0.001 (\pm 0.002)$    & $0.905 (\pm 0.008)$    & $0.219 (\pm 0.000)$    & $0.895 (\pm 0.025)$    & $0.001 (\pm 0.002)$ \\
\midrule
 Avg.      & $0.748 (\pm 0.017)$    & $0.454 (\pm 0.215)$    & $0.744 (\pm 0.019)$    & $0.016 (\pm 0.012)$    & $0.748 (\pm 0.017)$    & $0.454 (\pm 0.216)$    & $0.744 (\pm 0.019)$    & $0.016 (\pm 0.012)$ \\
 \bottomrule
\end{tabular}
}
\caption{Mean and standard deviation (after $\pm$) of target classification accuracy on HHAR for the aggregation method in objective~\eqref{eq:aggregation} with models computed by DDC~\citep{tzeng2014deep}.}
\end{table*}

\begin{table*}
\scalebox{0.75}{
\begin{tabular}{ l | c c c c | c c c c}
\toprule
  & \multicolumn{4}{c|}{\textbf{No Iteration}} & \multicolumn{4}{c}{\textbf{Iteration}}\\
\textbf{Scenario} & KuLSIF & Exp & LR & SQ & KuLSIF & Exp & LR & SQ \\
\midrule
 0 $\rightarrow$ 6    & $0.728 (\pm 0.019)$    & $0.700 (\pm 0.032)$    & $0.717 (\pm 0.025)$    & $0.006 (\pm 0.009)$    & $0.728 (\pm 0.019)$    & ${\bf 0.717 (\pm 0.025})$    & $0.717 (\pm 0.025)$    & $0.006 (\pm 0.009)$ \\
 1 $\rightarrow$ 6    & $0.886 (\pm 0.009)$    & $0.845 (\pm 0.039)$    & $0.861 (\pm 0.014)$    & $0.008 (\pm 0.003)$    & $0.886 (\pm 0.009)$    & ${\bf 0.861 (\pm 0.014})$    & $0.861 (\pm 0.014)$    & $0.008 (\pm 0.003)$ \\
 2 $\rightarrow$ 7    & $0.460 (\pm 0.012)$    & $0.470 (\pm 0.014)$    & $0.480 (\pm 0.020)$    & $0.004 (\pm 0.003)$    & $0.460 (\pm 0.012)$    & ${\bf 0.480 (\pm 0.021})$    & $0.480 (\pm 0.020)$    & $0.004 (\pm 0.003)$ \\
 3 $\rightarrow$ 8    & $0.812 (\pm 0.000)$    & $0.810 (\pm 0.003)$    & $0.811 (\pm 0.001)$    & $0.008 (\pm 0.007)$    & $0.812 (\pm 0.000)$    & ${\bf 0.811 (\pm 0.001})$    & $0.811 (\pm 0.001)$    & $0.008 (\pm 0.007)$ \\
 4 $\rightarrow$ 5    & $0.618 (\pm 0.536)$    & $0.919 (\pm 0.020)$    & $0.923 (\pm 0.016)$    & $0.003 (\pm 0.003)$    & ${\bf 0.936 (\pm 0.016})$    & ${\bf 0.923 (\pm 0.016})$    & $0.923 (\pm 0.016)$    & $0.003 (\pm 0.003)$ \\
\midrule
 Avg.      & $0.701 (\pm 0.115)$    & $0.749 (\pm 0.021)$    & $0.758 (\pm 0.015)$    & $0.006 (\pm 0.005)$    & ${\bf 0.764 (\pm 0.011})$    & ${\bf 0.758 (\pm 0.016})$    & $0.758 (\pm 0.015)$    & $0.006 (\pm 0.005)$ \\
 \bottomrule
\end{tabular}
}
\caption{Mean and standard deviation (after $\pm$) of target classification accuracy on HHAR for the aggregation method in objective~\eqref{eq:aggregation} with models computed by DeepCoral~\citep{Sun2017correlation}.}
\end{table*}

\begin{table*}
\scalebox{0.75}{
\begin{tabular}{ l | c c c c | c c c c}
\toprule
  & \multicolumn{4}{c|}{\textbf{No Iteration}} & \multicolumn{4}{c}{\textbf{Iteration}}\\
\textbf{Scenario} & KuLSIF & Exp & LR & SQ & KuLSIF & Exp & LR & SQ \\
\midrule
0 $\rightarrow$ 6    & ${\bf 0.693 (\pm 0.072})$    & $0.694 (\pm 0.050)$    & $0.709 (\pm 0.043)$    & $0.006 (\pm 0.010)$    & $0.688 (\pm 0.052)$    & ${\bf 0.709 (\pm 0.043})$    & $0.709 (\pm 0.043)$    & $0.006 (\pm 0.010)$ \\
 1 $\rightarrow$ 6    & $0.938 (\pm 0.000)$    & $0.884 (\pm 0.029)$    & $0.904 (\pm 0.010)$    & $0.002 (\pm 0.003)$    & $0.938 (\pm 0.000)$    & ${\bf 0.904 (\pm 0.010})$    & $0.904 (\pm 0.010)$    & $0.002 (\pm 0.003)$ \\
 2 $\rightarrow$ 7    & $0.194 (\pm 0.336)$    & ${\bf 0.582 (\pm 0.090})$    & $0.545 (\pm 0.041)$    & $0.002 (\pm 0.003)$    & ${\bf 0.574 (\pm 0.057})$    & $0.557 (\pm 0.062)$    & ${\bf 0.546 (\pm 0.043})$    & $0.002 (\pm 0.003)$ \\
 3 $\rightarrow$ 8    & $0.848 (\pm 0.007)$    & $0.839 (\pm 0.022)$    & $0.842 (\pm 0.011)$    & $0.007 (\pm 0.012)$    & $0.848 (\pm 0.007)$    & ${\bf 0.841 (\pm 0.010})$    & $0.842 (\pm 0.011)$    & $0.007 (\pm 0.012)$ \\
 4 $\rightarrow$ 5    & $0.984 (\pm 0.004)$    & $0.949 (\pm 0.008)$    & $0.951 (\pm 0.009)$    & $0.002 (\pm 0.002)$    & $0.984 (\pm 0.004)$    & ${\bf 0.951 (\pm 0.009})$    & $0.951 (\pm 0.009)$    & $0.002 (\pm 0.002)$ \\
\midrule
 Avg.      & $0.731 (\pm 0.084)$    & $0.790 (\pm 0.040)$    & $0.790 (\pm 0.023)$    & $0.004 (\pm 0.006)$    & ${\bf 0.806 (\pm 0.024})$    & ${\bf 0.793 (\pm 0.027})$    & $0.790 (\pm 0.023)$    & $0.004 (\pm 0.006)$ \\
 \bottomrule
\end{tabular}
}
\caption{Mean and standard deviation (after $\pm$) of target classification accuracy on HHAR for the aggregation method in objective~\eqref{eq:aggregation} with models computed by DIRT~\citep{Shu2018dirt}.}
\end{table*}

\begin{table*}
\scalebox{0.75}{
%
}
\caption{Mean and standard deviation (after $\pm$) of target classification accuracy on HHAR for the aggregation method in objective~\eqref{eq:aggregation} with models computed by DSAN~\citep{Zhu2021subdomain}.}
\end{table*}

\begin{table*}
\scalebox{0.75}{
%
}
\caption{Mean and standard deviation (after $\pm$) of target classification accuracy on HHAR for the aggregation method in objective~\eqref{eq:aggregation} with models computed by HoMM~\citep{Chen2020moment}.}
\end{table*}

\begin{table*}
\scalebox{0.75}{
%
}
\caption{Mean and standard deviation (after $\pm$) of target classification accuracy on HHAR for the aggregation method in objective~\eqref{eq:aggregation} with models computed by MMDA~\citep{Rahman2020}.}
\label{tab:last_table_appendix}
\end{table*}

\end{document}